\documentclass[11pt,a4paper,oneside]{article}

\usepackage{geometry}
\usepackage[OT1]{fontenc}

\usepackage{titlesec}

\titleformat{\subsubsubsection}
{\normalfont\normalsize\bfseries}{\theparagraph}{1em}{}
\titlespacing*{\subsubsubsection}
{0pt}{3.25ex plus 1ex minus .2ex}{1.5ex plus .2ex}

\usepackage[english]{babel}

\usepackage[utf8]{inputenc}

\usepackage{amsmath,amssymb,amsfonts,mathrsfs}
\allowdisplaybreaks

\usepackage[amsmath,thmmarks]{ntheorem}

\usepackage{graphicx}

\usepackage{soul}

\usepackage{pdfpages}

\usepackage[preprint]{jmlr2e}

\usepackage{varioref}

\usepackage{datetime}

\usepackage{mathtools}
\mathtoolsset{showonlyrefs}  

\usepackage{url}

\usepackage[h]{esvect}

\usepackage{array}

\usepackage{listings}
\usepackage{microtype}

\usepackage{booktabs}

\usepackage{algorithm, algorithmic}

\usepackage{natbib}

\usepackage{mdframed}

\usepackage[font=small, labelfont=bf]{caption}
\usepackage[labelfont=rm]{subcaption}

\usepackage{tikz}
\usetikzlibrary{bayesnet}
\usepackage{subcaption}
\usepackage{footnote}
\usepackage{bbm}
\usepackage{accents}
\usepackage{ulem}

\newtheorem{condition}{Condition}
\renewtheorem{remark}{Remark}

\theoremstyle{nonumberplain}
\theorembodyfont{\normalfont}
\theoremsymbol{\ensuremath{\square}}

\newcommand{\bszero}{\boldsymbol{0}}
\newcommand{\bsone}{\boldsymbol{1}}

\newcommand{\bse}{\boldsymbol{e}}

\newcommand{\bsg}{\boldsymbol{g}}
\newcommand{\bsp}{\boldsymbol{p}}
\newcommand{\bsw}{\boldsymbol{w}}

\newcommand{\bsbeta}{\boldsymbol{\beta}}

\newcommand{\bsphi}{\boldsymbol{\phi}}
\newcommand{\bspsi}{\boldsymbol{\psi}}

\newcommand{\EE}{\mathbb{E}}

\newcommand{\PP}{\mathbb{P}}

\newcommand{\RR}{\mathbb{R}}
\newcommand{\VV}{\mathbb{V}}

\newcommand{\cB}{\mathcal{B}}

\newcommand{\cE}{\mathcal{E}}
\newcommand{\cF}{\mathcal{F}}
\newcommand{\cG}{\mathcal{G}}
\newcommand{\cH}{\mathcal{H}}

\newcommand{\cN}{\mathcal{N}}

\newcommand{\cT}{\mathcal{T}}
\newcommand{\cU}{\mathcal{U}}
\newcommand{\cW}{\mathcal{W}}
\newcommand{\cX}{\mathcal{X}}
\newcommand{\cY}{\mathcal{Y}}
\newcommand{\cZ}{\mathcal{Z}}

\newcommand{\obs}{\mathrm{obs}}
\newcommand{\base}{\mathrm{base}}
\newcommand{\rbox}{\mathrm{box}}

\newcommand{\textinf}{\mathrm{inf}}
\newcommand{\textsup}{\mathrm{sup}}
\newcommand{\ZSB}{\mathrm{ZSB}}
\newcommand{\CMC}{\mathrm{CMC}}
\newcommand{\KCMC}{\mathrm{KCMC}}
\newcommand{\KPCA}{\mathrm{KPCA}}
\newcommand{\IPW}{\mathrm{IPW}}

\newcommand{\rd}{\, \mathrm{d}}
\newcommand{\pto}{{\overset{p.}{\to}}}
\newcommand{\dto}{{\overset{d.}{\to}}}
\DeclareMathOperator*{\plim}{plim}

\newcommand{\ubar}[1]{\underaccent{\bar}{#1}}
\newcommand{\bbmone}{\mathbbm{1}}
\newcommand{\indep}{\perp \!\!\! \perp}

\renewcommand{\epsilon}{\ensuremath\varepsilon}

\renewcommand{\phi}{\ensuremath{\varphi}}

\DeclareMathOperator*{\sumint}{
\mathchoice
  {\ooalign{$\displaystyle\sum$\cr\hidewidth$\displaystyle\int$\hidewidth\cr}}
  {\ooalign{\raisebox{.14\height}{\scalebox{.7}{$\textstyle\sum$}}\cr\hidewidth$\textstyle\int$\hidewidth\cr}}
  {\ooalign{\raisebox{.2\height}{\scalebox{.6}{$\scriptstyle\sum$}}\cr$\scriptstyle\int$\cr}}
  {\ooalign{\raisebox{.2\height}{\scalebox{.6}{$\scriptstyle\sum$}}\cr$\scriptstyle\int$\cr}}
}

\usepackage[colorinlistoftodos,bordercolor=orange,backgroundcolor=orange!20,linecolor=orange,textsize=scriptsize]{todonotes}

\newcommand{\ki}[1]{\todo[inline]{\textbf{KI: }#1}}
\usepackage{lastpage}
\jmlrheading{24}{2024}{1-\pageref{LastPage}}{11/23; Revised 1/26}{0/0}{00-0000}{Kei Ishikawa, Niao He and Takafumi Kanamori}

\ShortHeadings{Convex Framework for Confounding Robust Inference}{Ishikawa, He and Kanamori}
\firstpageno{1}

\title{A Convex Framework for Confounding Robust Inference}
\author{
    \name \hspace{-0.65em} Kei Ishikawa \email k.stoneriv@gmail.com\\
    \addr Department of Mathematics \\
    ETH Zurich
    \thanks{Part of this work was done while the author was affiliated with Tokyo Institute of Technology (currently Science Tokyo).}
    \\
    Rämistrasse 101, 8092 Zürich, Switzerland
    \AND
    \name Niao He \email niao.he@inf.ethz.ch\\
    \addr Department of Computer Science\\
    ETH Zurich\\
    Universitätstrasse 6, 8092 Zürich, Switzerland
    \AND
    \name Takafumi Kanamori \email kanamori@comp.isct.ac.jp\\
    \addr School of Computing\\
    Institute of Science Tokyo\\
    2-12-1 Okayama, Meguro-ku, Tokyo, Japan\\
    RIKEN AIP\\
    1-4-1 Nihonbashi, Chuo-ku, Tokyo, Japan
}

\editor{TBD}

\begin{document}

\maketitle

\begin{abstract}
We study policy evaluation of offline contextual bandits subject to unobserved confounders. 
Sensitivity analysis methods are commonly used to estimate the policy value under the worst-case confounding scenario within a given uncertainty set. 
However, existing work often resorts to some coarse relaxation of the uncertainty set for the sake of tractability, leading to overly conservative estimation of the policy value. 
In this paper, we propose a general estimator that provides a sharp lower bound of the policy value using convex programming.
The generality of our estimator enables various extensions such as sensitivity analysis using f-divergence, model selection with cross validation and information criterion, and robust policy learning with the sharp lower bound.
Furthermore, our estimation method can be reformulated as an empirical risk minimization problem thanks to the strong duality, which enables us to provide strong theoretical guarantees of the proposed estimator using M-estimation techniques.
\end{abstract}


\section{Introduction}\label{chap:intro}


The offline contextual bandit, a simple but powerful model for decision-making, finds applications across a spectrum of domains from personalized medical treatment to recommendations and advertisements on online platforms. 
For the policy evaluation in this setting, the inverse probability weighting (IPW) method \citep{hirano2001estimation, hirano2003efficient} and its variants are commonly used.
These methods rely on the so-called unconfoundedness assumption, presupposing full observability of pertinent variables to exclude any influence from unobserved factors on the action selection and the resulting reward \citep{rubin1974estimating}. 
However, in practice, this unconfoundedness assumption can easily be violated, since often times unobserved confounders are not recorded in the logged data. 
For instance, consider personalized medicine, where private information such as patients' financial status is usually not recorded in medical records yet may impact both the choice of treatments and their efficacy. 
An expensive treatment option may be only available for wealthy patients, and higher financial status may also have a correlation with their health status.
In such situations, if the treatment effect of the expensive option depends on the health status of the patients, we can potentially overestimate (or underestimate) the true treatment effect due to this hidden underlying correlation.

To address this, sensitivity analysis offers a solution by establishing a worst-case lower bound on the policy value.  This involves defining an uncertainty set encompassing potential confounding situations and evaluating the policy value under the most adverse conditions within this set. Such a worst-case lower bound allows for robust decision-making in the presence of confounding.  A wide range of sensitivity models has been  studied over the years  \citep{rosenbaum2002overt, tan2006distributional, rosenbaum2010design, liu2013introduction}.

The main focus of our paper is the marginal sensitivity model by \citet{tan2006distributional} and its extensions, which assumes similarity between the observational probability and the true probability of the action given the context.
Recently, \citet{zhao2019sensitivity} introduced an elegant algorithm for Tan's marginal sensitivity analysis using the linear fractional programming, which revitalized the study of this model. They also used the bootstrap method to provide an interpretable estimate of the lower bound.
This approach was extended to policy learning in \citet{kallus2018confounding, kallus2021minimax}, where they established a statistical guarantee for confounding robust policy learning through lower bound maximization. 
Notably, their theoretical analysis included a uniform bound for the estimation error of a class of policies, unlike previous works that only focused on the evaluation of a single policy.

\subsection{Motivation}

These sensitivity analysis methods rely on algorithms using linear programming with relaxed constraints for their tractability, leading to overly conservative lower bound estimates of the policy value.  Such estimates are not necessarily guaranteed to be a consistent estimator of the true lower bound.
To obtain a sharp lower bound,  conditional moment constraints must be leveraged, which are infinite dimensional linear constraints. Recently, \citet{dorn2022sharp} analyzed these constraints and characterized the sharp lower bound of Tan's marginal sensitivity model with a conditional quantile function of the reward distribution. 
Using this characterization, they proposed the first tractable algorithm to estimate the sharp lower bound that converges to the true lower bound of the policy value. 

In this paper, we revisit the sharp estimation problem from a novel perspective. Instead of relying on conditional quantile functions, 
we propose a tractable approximation to the infinite dimensional conditional moment constraints, enabling the sharp lower bound estimation to be expressed directly as a convex optimization problem. Unlike the previous two-stage optimization method introduced by \citet{dorn2022sharp}, our formulation is more general and naturally includes their sharp estimator as a special case.

\subsection{Contributions}

The generality of our estimator expands the horizon of the sensitivity analysis, both in terms of theory and applications.

Our major theoretical contributions stem from the reformulation of the lower bound estimation as an empirical risk minimization using the strong duality.
This enables us to leverage the standard proof techniques for M-estimation to derive consistency and asymptotics of our estimator, simplifying our proofs compared to the bespoke proofs in previous works \citep{kallus2018confounding, kallus2021minimax}. In particular, we are the first to prove a consistent policy learning guarantee for the sharp lower bound, which parallels the seminal result in \citet{kallus2018confounding, kallus2021minimax} for the classic unsharp bound.

In terms of applications, our convex formulation is less restrictive compared to the two-stage approach of the previous sharp estimators \citep{dorn2022sharp, dorn2021doubly} and offers possibilities for various extensions, such as generalization of the classic marginal sensitivity model \citep{tan2006distributional} using f-divergence, policy learning with the sharp lower bound, and model selection using the cross validation or the generalized information criterion \citep{konishi2008information}.
Furthermore, our estimator can naturally handle both discrete and continuous treatment, setting it apart from conventional unsharp estimators \citep{tan2006distributional, zhao2019sensitivity, kallus2018confounding}, which are unsuitable for cases involving continuous treatment.

Lastly, we highlight the difference between this paper from an earlier version by the authors \citep{ishikawa2023kernel}.
All the analyses on asymptotic normality and its applications including hypothesis testing and model selection are the novel results of this paper. This paper also provides a more comprehensive study of the consistency and specification bias, by providing the consistent policy learning guarantee for a general Vapnik-Chervonenkis policy class with the sharp lower bound for the first time, and by characterizing the magnitude of the specification bias when the kernel cannot be chosen optimally to achieve zero bias.

\subsection{Related works}

The sensitivity analysis of average treatment effect using quantile regression by \citet{dorn2022sharp,dorn2021doubly} was the first approach that consistently recover the true lower bound under Tan’s marginal sensitivity model, significantly sharpening bounds compared to traditional methods.
Subsequently, this approach was also applied to the conditional average treatment effects \citep{oprescu2023b} and individual treatment effect \citet{jin2022sensitivity}.
In parallel, confounding-robust policy evaluation techniques have evolved from evaluation to robust policy learning \citep{kallus2018confounding,kallus2021minimax} and to sequential decision-making settings \citep{kallus2020confounding}.
Recently, the sharp bound introduced by \citet{dorn2022sharp} has also been effectively adapted for sequential decision-making by \citet{bruns2023robust,bennett2024efficient}. However, extending this sharp bound approach specifically to policy learning has remained unexplored, and our paper is the first to address this gap.

In parallel, a growing literature has studied causal inference using modern machine learning methods, including instrumental variable estimation \citep{hartford2017deep}, heterogeneous and individual treatment effect estimation \citep{wager2018estimation,shalit2017estimating,louizos2017causal}, and debiased learning frameworks that combine machine learning with classical causal inference \citep{chernozhukov2018double}. Kernel-based methods have also been widely used in causal inference, particularly for testing conditional independence and representing nonparametric conditional relationships \citep{gretton2005measuring,zhang2012kernel,peters2017elements}.

Our convex reformulation of the infinite dimensional constraints is inspired by kernel methods \citep{scholkopf1997kernel}.
The most relevant to our work is by \citet{kremer2022functional}, who used kernel methods to represent conditional moment restrictions and estimate model parameters.
While their work focused on the dual formulation, we solve the primal but still exploit the dual in our theoretical analysis.
\citet{muandet2020kernel} also studied conditional moment restrictions, introducing a kernel-based test statistic structurally similar to ours but for hypothesis testing rather than uncertainty quantification.
More broadly, the idea of using the kernel methods for constraints has been explored in areas such as fairness \citep{perez2017fair}, distributional robustness \citep{staib2019distributionally, zhu2020worst}, and shape-constrained regression \citep{aubin2020hard}.
The kernel methods have also seen growing use in causal inference, including treatment effect estimation \citep{singh2024kernel}, instrumental variables \citep{singh2019kernel, muandet2020dual, zhang2023instrumental}, negative controls \citep{singh2020kernel, kallus2021causal, mastouri2021proximal}, and policy evaluation \citep{kallus2018balanced}.

\subsection{Organization of the paper}
The remainder of the paper is organized as follows:
In Section \ref{chap:background}, we formally introduce the problem we study and propose a new sharp lower bound estimator of policy value.
Section \ref{chap:theoretical_analysis} studies the theoretical properties of the proposed estimator through reformulation of our estimation method via strong duality.
In Section \ref{chap:experiments}, we present our numerical experiments, and Section \ref{chap:conclusion} concludes the paper.

\section{Problem Settings and Proposed Method}\label{chap:background}


In this section, we start by introducing the formal definitions for confounded offline contextual bandits. Then, we discuss a sensitivity analysis model for its policy evaluation. Since the original formulation of this model requires an intractable optimization, we make a series of relaxations to motivate our proposed estimator.

Given the abundance of notations introduced herein, we offer a list of common notations used throughout the paper in Appendix \ref{app:notations}.

\subsection{Confounded offline contextual bandits}

Here, we formally define the confounded offline contextual bandits  
using the graphical models \citep{koller2009probabilistic, pearl2009causality}.
Though a potential outcome framework \citep{rubin2005causal} is a popular alternative, we employed the graphical models as we find it easier to formulate our general modeling assumptions that accommodate both discrete or continuous action space.

\begin{figure}[t]
     \centering
     \begin{subfigure}[b]{0.45\textwidth}
     \centering
         \begin{tikzpicture}
\node[obs]  (x) {$X$};
\node[obs, below=0.3 of x, xshift=1.3cm]  (t) {$T$};
\node[obs, above=0.3 of t, xshift=1.3cm]  (y) {$Y$};
\edge {x} {t, y} ; %
\edge {t} {y} ; %
\end{tikzpicture}
         \caption{unconfounded case}
         \label{fig:unconfounded_bayes_net}
     \end{subfigure}
     \begin{subfigure}[b]{0.45\textwidth}
         \centering
         \begin{tikzpicture}
\node[obs]  (x) {$X$};
\node[latent, above=0.3 of x, xshift=1.3cm]  (u) {$U$};
\node[obs, below=0.3 of x, xshift=1.3cm]  (t) {$T$};
\node[obs, below=0.3 of u, xshift=1.3cm]  (y) {$Y$};
\edge[-] {x} {u}
\edge {x} {t, y} ; %
\edge {u} {t, y} ; %
\edge {t} {y} ; %
\end{tikzpicture}
         \caption{confounded case}
         \label{fig:confounded_bayes_net}
     \end{subfigure}
     \caption{Graphical model of unconfounded and confounded contextual bandit}
\end{figure}
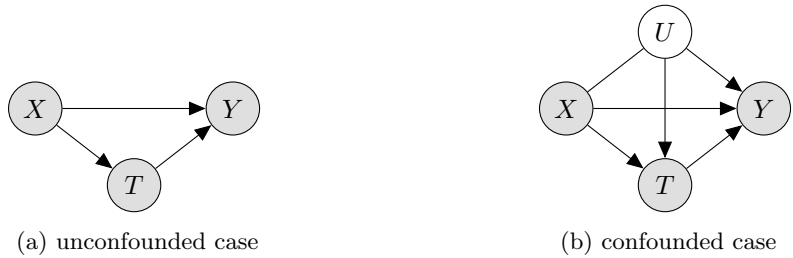

The confounded offline contextual bandits are an extension of the standard offline contextual bandits that have an additional unobserved confounding variable.
Let $\cX$ and $\cT$ be the context space and the action space.
The model can be described as the probabilistic model of the four variables: context $X\in\cX$, confounder $U\in\cU$, action (or treatment) $T\in\cT$, and reward (or outcome) $Y\in\RR$.
\footnote{Though the outcome and the reward are not always the same in general, we assume them to be the same for simplicity.}
For some unknown base policy $\pi_\base(t|x, u)$, we obtain an offline dataset of variables $Y$, $T$, and $X$ from the following data generating process: 
Unlike the other variables confounder $U$ is unobservable, and that introduces a confounding effect in the observable data.
In confounded offline contextual bandits, the data are generated according to the following model:
\begin{align}
\begin{array}{ll}
   X, U &\sim p(x, u), \\
   T|X, U &\sim \pi_\base(t|X, U),\\
   Y|T, X, U &\sim p(y|T, X, U),
\end{array}\label{eq:data_gen_base_confounded}
\end{align}
but only $Y$, $T$, and $X$ are observable and policy $\pi_\base(t|x, u)$ is unknown.
This corresponds to the graphical model in Figure \ref{fig:confounded_bayes_net}, where the existence of $U$ is the clear difference from the unconfounded graphical model in Figure \ref{fig:unconfounded_bayes_net}.
\footnote{Alternatively, we could consider the unconfounded case as the situation where confounder $U$ is constant.}

To better understand variables, let us remember the earlier example of personalized medical treatment.
The observational data $\{(Y_i, T_i, X_i)\}_{i=1}^n$ corresponds to medical records. They include basic information about patients such as age, sex, and weight, which corresponds to context $X$.
Additionally, $T$ and $Y$ are the medical treatment taken by the patients and its outcome.
However, some information such as the financial status of the patients is not included in them.
Yet, such information can influence the choice of treatment and the effectiveness of the treatment in either direct or indirect manners. Such a variable corresponds to confounder $U$ in our model.

Now, we consider the offline evaluation of the policy from the observational data.
Given new policy $\pi(t|x)$, we are interested in estimating the expectation of reward $Y$ under the following probabilistic model:
\begin{align}
\begin{array}{ll}
   X, U &\sim p(x, u), \\
   T|X, U &\sim \pi(t|X),\\
   Y|T, X, U &\sim p(y|T, X, U).
\end{array}\label{eq:data_gen_eval_confounded}
\end{align}
Here, we only consider an observable policy $\pi(t|x)$, because it is trivially impossible to evaluate a policy that depends on unobserved variable $U$ only using the offline data.

In confounded contextual bandits, offline evaluation of the above observable policy is still impossible without any further assumptions.
To explain the reason, let's consider change of expectations
\begin{align}
V(\pi)
= \EE_{T\sim\pi(\cdot|X)}\left[ Y \right]
= \EE_{T\sim \pi_\base(\cdot|X, U)}\left[ \left(\frac{\pi(T|X)}{\pi_\base(T|X, U)}\right) Y \right]
\label{eq:ipw_confounded},
\end{align}
where $\EE_{T\sim \pi_\base(\cdot|X, U)}[f(Y, T, X, U)]$ and $\EE_{T\sim\pi(\cdot|X)}[f(Y, T, X, U)]$ correspond to the expectations of $f(Y, T, X, U)$ under the generative processes \eqref{eq:data_gen_base_confounded} and \eqref{eq:data_gen_eval_confounded}.
Compared to the unconfounded case, the inverse probability in \eqref{eq:ipw_confounded} has a dependence on unobserved variable $U$.
This makes it impossible to construct a consistent estimator unlike the unconfounded case, where the inverse probability weighting estimator is consistent.
As the observable variables are only $Y$, $T$, and $X$, any valid estimator must depend only on them.
Ideally, we would like to know conditional expectation of the inverse probability $w(y, t, x):= \EE_{T\sim\pi_\base(\cdot|X, U)}\left[\frac{1}{\pi_\base(T|X, U)}|Y=y, T=t, X=x\right]$.
Then, we can construct a valid unbiased and consistent estimator $\hat V_w := \frac{1}{n}\sum_i w(Y_i, T_i, X_i) \pi(T_i|X_i) Y_i$, because 
\begin{equation}
    V(\pi) = 
    \EE_{T\sim\pi_\base(\cdot|X, U)}\left[
        \EE_{T\sim\pi_\base(\cdot|X, U)}\left[
            \frac{1}{\pi_\base(T|X, U)}| Y, T, X
        \right] \pi(T|X) Y
    \right]\label{eq:conditional_ipw}.
\end{equation}
However, in practice, it is impossible to know $w(y, t, x)$ in the confounded offline contextual bandits, and it motivates the use of lower bound of the policy value.

Hereafter, for simplicity of notation and analysis, we assume that the conditional distribution of reward $Y$ given $T$, $X$, and $U$ is continuous and that it yields a density function $p(y|t, x, u)$ when such assumptions are appropriate.
We will also assume that $\pi(T|X, U) > 0$ almost surely under data generating process \eqref{eq:data_gen_base_confounded} so that the inverse probability weights are always well-defined.
Additionally, we will use the following simplified notations. 
To indicate a part of model \eqref{eq:data_gen_base_confounded} that can be identified by the offline data, we use $p_\obs$ to indicate the observable distribution of \eqref{eq:data_gen_base_confounded} so that
\[
p_\obs(y, t, x) = \sumint p(y|t, x, u) \pi_\base(t|x, u) p(x, u)\rd u.
\footnote{We use notation $\sumint$ to emphasize that unobserved confounder $U$ could be both a summation over a discrete variable or an integral over a continuous variable. The same notation is used later for treatment $T$.}
\]
Similarly, $p_\obs(t|x)$ and $p_\obs(x)$ will denote the corresponding conditional and marginal distributions, and $\EE_\obs[f(Y, T, X)] := \EE_{T\sim\pi_\base(\cdot|X, U)}[f(Y, T, X)]$ represents the expectation of $f(Y, T, X)$ with respect to $p_\obs(y, t, x)$.
To represent the empirical average that approximates $\EE_\obs$, we use $\hat\EE_n$ so that $\hat\EE_n[f(Y, T, X)] := \frac{1}{n}\sum_{i=1}^n f(Y_i, T_i, X_i)$ for any $f(y, t, x)$.
Finally, throughout the paper, we assume that we know the true value of $p_\obs(t|x)$ in the estimation of the policy value and its lower (or upper) bounds for the sake of simplicity. 
\footnote{
In the literature such as  \cite{kallus2018confounding, dorn2022sharp}, it is commonly assumed that the probability of action given a context in observational data can be consistently estimated from observable data. This probability is often referred to as the "nominal propensity score." However, to maintain consistent notation, we will refer to it as the observational probability and denote it as $p_\obs(t|x)$ to highlight the tractable components of our formulas.
}

\subsection{Uncertainty sets of base policies}\label{sec:policy_uncertainty_set}
A practical workaround to the above-mentioned issue is partial identification of the policy value under some reasonable assumption about confounding.
More specifically, we first define uncertainty set $\cE$ of $\pi_\base(t|x, u)$ that we believe to contain all possible base policies underlying the observational data. 
Then we find infimum  policy value $V_\textinf$ (or supremum $V_\textsup$) within the uncertainty set as
\begin{equation}
V_\textinf(\pi) := \inf_{\pi_\base \in \cE} \EE_{T\sim\pi_\base(\cdot|X, U)}\left[\left(\frac{\pi(T|X)}{\pi_\base(T|X, U)} \right)Y\right].
\label{eq:v_inf_exact_policy}
\footnote{
Alternatively, one may take the infimum not only over $\pi_\base(t|x, u)$ but also over the joint distribution $p(x, u)$ and the conditional outcome model $p(y|t, x, u)$. 
In contrast, our formulation assumes oracle access to the remaining components of the data-generating process. 
As we will show later, this assumption enables a sequence of relaxations that lead to a tractable lower-bound estimator.
}
\end{equation}
In the following, we introduce the conditional f-constraint, a class of constraints that can express a few reasonable classes of uncertainty sets.

Let $f:\RR\times\cT\times\cX\to\RR$ be a function satisfying that $f_{t, x}(\cdot):=f(\cdot, t, x)$ is convex and $f_{t, x}(1)=0$ for any fixed $t\in\cT$ and $x\in\cX$.
Then, we define the uncertainty set of the conditional f-constraint \footnote{Though this class of constraint contains the f-divergence constraint discussed later in Example \ref{ex:f_divergence_const}, we do not call them "conditional f-divergence constraint". This is because it does not become f-divergence in general, for instance, in the case of the box constraints.}
as
\begin{equation}
    \cE_{f_{t, x}} := \left\{
    \pi_\base(t|x, u):
    \begin{array}{c}
    \EE_{T\sim\pi_\base(\cdot|X, U)} \left[f_{T, X}\left( \frac{p_\obs(T|X)}{\pi_\base(T|X, U)} \right)\right] \leq \gamma \\
    \text{and} \\
    \pi_\base(t|x, u) \text{ is proper distribution}
    \end{array}
    \right\}.
    \label{eq:conditional_f_policy_uncertainty_set}
\end{equation}

Note that the lower bound for this uncertainty set is still intractable, as it involves the expectation with respect to unobserved variable $U$.
This issue is addressed in the next subsection, where we introduce relaxation to these uncertainty sets to obtain a tractable formulation of the lower bound.

Though this class of uncertainty sets has not been studied in prior work, we focus on this type of uncertainty set because it is designed to include several practically relevant examples, such as box constraints and f-divergence constraints.


\begin{example}[Box constraints]\label{ex:box_const}

The box constraints are the type of uncertainty set that has traditionally been adopted in the sensitivity analysis.
They can be expressed as uncertainty set
\begin{equation}
    \cE_\rbox := \left\{
    \pi_\base(t|x, u):
    \begin{array}{c}
    a_\pi(t, x) \leq \pi_\base(t|x, u) \leq b_\pi(t, x) \\
    \text{and} \\
    \pi_\base(\cdot|x, u) \text{ is proper distribution}
    \end{array}
    \right\}
    \label{eq:box_policy_uncertainty_set}
\end{equation}
for some $a_\pi(t, x)$ and $b_\pi(t, x)$.
This model includes the well-known marginal sensitivity model by \citet{tan2006distributional} and most of its extensions \citep{zhao2019sensitivity, kallus2018confounding, dorn2022sharp}. Tan considered a binary action space and assumed that the odds ratio of observational probability of action $p_\obs(t|x)$ and true base policy $\pi_\base(t|x, u)$ are not too far from $1$ so that
\begin{equation*}
    \Gamma^{-1} \leq \frac{p_\obs(t|x) (1 - \pi_\base(t|x, u)}{(1 - p_\obs(t|x)) \pi_\base(t|x, u)} \leq \Gamma.
\end{equation*}
As $p_\obs(t|x)$ can be identified from the observational data, we can enforce such constraints by choosing $a_\pi$ and $b_\pi$ in \eqref{eq:box_policy_uncertainty_set} as
\begin{equation}
a_\pi(t, x) = \frac{1}{1 + \Gamma (1/p_\obs(t|x) - 1)} 
\text{ and }
b_\pi(t, x) = \frac{1}{1 + \Gamma^{-1} (1/p_\obs(t|x) - 1)}.
\label{eq:tan_box_constraints}
\end{equation}

This model can be expressed in terms of the conditional f-constraint \eqref{eq:conditional_f_policy_uncertainty_set} by choosing 
\begin{equation}
    f_{t, x}(\tilde w) = I_{[a_{\tilde w}(t, x), b_{\tilde w}(t, x)]}(\tilde w) 
    \label{eq:f_box_constraints}
\end{equation}
for $a_{\tilde w}(t, x) = p_\obs(t|x) / b_\pi(t, x)$
and $b_{\tilde w}(t, x) = p_\obs(t|x) / a_\pi(t, x)$, 
$\cE_{f_{t, x}}$,
where $I_{[a, b]}(\tilde w)$ is a characteristic function 
\footnote{Please be aware that we use $\bbmone_{[a, b]}$ to represent indicator function that takes value in $\{0, 1\}$.}
of interval $[a, b]$ defined as
\[
I_{[a, b]}(\tilde w) :=
\begin{cases}
    0 &\text{ if } \tilde w \in [a, b], \\
    \infty & \text{otherwise.}
\end{cases}
\]
    
\end{example}

\begin{example}[f-divergence constraint]\label{ex:f_divergence_const}
Using the f-divergence, we can also define another class of uncertainty sets that belong to the family of conditional f-constraints.
We call such a constraint and its corresponding sensitivity model the f-divergence constraint and the f-sensitivity model.
To our knowledge, the majority of the existing literature only studies the box constraints, and we are the first to discuss the use of f-divergence constraint. 
\footnote{
Indeed, \citet{jin2022sensitivity} considered a similar use of f-divergence for sensitivity analysis. However, their model makes a different type of confounding assumption, as discussed in Appendix \ref{app:comp_jin_f_sensitivity}.
}

Before introducing the new uncertainty set, we first recall the formal definition of f-divergence.
Let $p(t)$ and $q(t)$ be the probability mass function over countable set $\cT$ (or the probability density function over $\cT\subseteq\RR^d$ for some $d$ with respect to the Lebesgue measure), and let  $f:\RR\to\RR$ be a convex function satisfying $f(1) = 0$.
Then, the f-divergence between the distribution of $p(t)$ and $q(t)$ is defined as
    $D_f[p||q] := \sumint_\cT f \left( \frac{p(t)}{q(t)} \right)q(t)\rd t.$
The f-divergence is a rich class of divergence between probability distributions that includes many divergences such as the Kullback-Leibler (KL), squared Hellinger, and Pearson $\chi^2$ divergences. 


Using the f-divergence, we define a new class of uncertainty set as
\begin{equation}
    \cE_f := \left\{
    \pi_\base(t|x, u):
    \begin{array}{c}
        \EE_{X, U} \left[D_f[p_\obs(t|X) || \pi_\base(t|X, U)]\right] \leq \gamma \\
        \text{and} \\
        \pi_\base(t|x, u) \text{ is proper distribution}
    \end{array}
    \right\},
    \label{eq:f_policy_uncertainty_set}
\end{equation}
where the expectation $\EE_{X, U}$ can be taken with respect to $p(x, u)$ in both \eqref{eq:data_gen_base_confounded} and \eqref{eq:data_gen_eval_confounded}, regardless of the policy.
This formulation naturally fits into the conditional f-constraints \eqref{eq:conditional_f_policy_uncertainty_set} as the divergence term becomes a simple expectation as
    $\EE_{X, U} \left[D_f[p_\obs(t|X) || \pi_\base(t|X, U)]\right] = 
    \EE_{T\sim\pi_\base(\cdot|X, U)}\left[ f\left(\frac{p_\obs(T|X)}{\pi_\base(T|X, U)}\right) \right]$.

The key distinction between f-divergence constraints and box constraints lies in their non-uniform nature. While box constraints impose uniform bounds on the odds ratio  $\frac{p_\obs(t|x)}{\pi_\base(t|x,u)}$ with respect to context $X$, f-divergence constraints only limit the overall deviation of the policy in aggregate. As a result, they allow for base policies that can vary significantly for a limited range of context values $X$.
Practically, box constraints are appropriate when one wants a worst-case, uniform robustness guarantee across all contexts. In contrast, $f$-divergence constraints are useful when one expects confounding to be concentrated on a subset of the population and prefers to control the average deviation globally.


\end{example}

\subsection{Relaxation of the uncertainty sets by reparametrization}\label{sec:weight_undertainty_set}

Having defined the uncertainty set of our interest, we now consider the computation method of lower bound $V_\textinf$.
Unfortunately, directly solving \eqref{eq:v_inf_exact_policy} is often too difficult in practice, so we instead solve a relaxed problem that is more tractable. 
First, we introduce reparametrization $w(y, t, x)= \EE_{T\sim\pi_\base(\cdot|X, U)}\left[\frac{1}{\pi_\base(T|X, U)}|Y=y, T=t, X=x\right]$.
By relaxing the uncertainty set of unknown base policy $\pi_\base(t|x, u)$, we can construct an uncertainty set of $w(y, t, x)$ and obtain a convex relaxation of original problem \eqref{eq:v_inf_exact_policy}.

Without the relaxation, the uncertainty set of base policy $\cE_{f_{t, x}}$ can be translated to the uncertainty set of the reparametrized weights as
\begin{equation}
    \cW_{f_{t, x}} := \left\{
    w(y, t, x):
    \begin{array}{c}
        w(y, t, x) = \EE_{T\sim\pi_\base(\cdot|X, U)}\left[\frac{1}{\pi_\base(T|X, U)}|Y=y, T=t, X=x\right] \\
        \text{and} \\
        \EE_{T\sim\pi_\base(\cdot|X, U)} \left[f_{T, X}\left( \frac{p_\obs(T|X)}{\pi_\base(T|X, U)} \right)\right] \leq \gamma \\
        \text{and} \\
        \pi_\base(t|x, u) \text{ is proper distribution}
    \end{array}
    \right\}.
    \label{eq:sharp_uncertainty_set}
\end{equation}
Using \eqref{eq:conditional_ipw}, we can obtain the lower bound as
\begin{equation}
    V_\textinf(\pi) = \inf_{w\in\cW_{f_{t, x}}}\EE_\obs\left[ w(Y, T, X) \pi(T|X) Y \right]\label{eq:v_inf_exact}.
\end{equation}
This uncertainty set requires that the conditional f-constraint is satisfied and that the base policy $\pi_\base(t|x, u)$ is proper distribution.
However, in practice, both conditions are intractable. 
Therefore, we will consider the relaxation of these conditions.

\subsubsection*{Relaxation of the conditional f-constraint}
Let us consider the first condition of \eqref{eq:sharp_uncertainty_set} on conditional f-divergence.
With Jensen's inequality, we have 
\begin{align}
    \EE_{T\sim\pi_\base(\cdot|X, U)}\left[ f_{T, X}\left(\frac{p_\obs(T|X)}{\pi_\base(T|X, U)}\right) \right]
    \geq \EE_\obs\left[ f_{T, X}\left(
        \EE_{T\sim\pi_\base(\cdot|X, U)} \left[\frac{p_\obs(T|X)}{\pi_\base(T|X, U)}|Y, T, X\right]
    \right) \right],
    \label{eq:f_divergence_lower_bound}
\end{align}
where the equality holds when there's $\pi_\base(t|x, u) = p_\obs(t| x)$ (i.e., no confounding) or when $f_{T, X}(\cdot)$ is linear almost surely.\footnote{Indeed, the box constraints satisfy this linear condition and no relaxation is made. On the other hand, the f-divergence constraint relies on strictly convex functions, and thus the equality of Jensen's inequality does not hold.}
Thus, we relax condition 
$\EE_{T\sim\pi_\base(\cdot|X, U)}\left[ f_{T, X}\left(\frac{p_\obs(T|X)}{\pi_\base(T|X, U)}\right) \right] \leq \gamma$
in \eqref{eq:sharp_uncertainty_set}
to
\begin{equation}
    \EE_\obs\left[ f_{T, X}\left( p_\obs(T|X)w(Y, T, X) \right) \right] \leq \gamma.
    \label{eq:relaxed_f_constraints}
\end{equation}

Here, we must emphasize that the above constraint is convex with respect to $w(y, t, x)$.
This convexity translates to the empirical version of the above constraint, and enables us to estimate the lower bound by solving convex programming.

Additionally, when $f_{t, x}$ does not depend on $t$ and $x$, there is a simple trick to ensure that the left hand side of \eqref{eq:relaxed_f_constraints} is non-negative. 
By imposing $\EE_\obs\left[ p_\obs(T|X)w(Y, T, X) \right] = 1$,
we can guarantee
$0 = f(1) = f\left( \EE_\obs\left[ p_\obs(T|X)w(Y, T, X) \right] \right) \leq \EE_\obs\left[ f\left( p_\obs(T|X)w(Y, T, X) \right) \right]$
due to Jensen's inequality.
In the numerical experiments, we observed that this helps to obtain a more stable implementation of the f-sensitivity model.

\subsubsection*{Relaxation of the distributional constraints}
Now we consider the relaxation of the second condition in \eqref{eq:sharp_uncertainty_set} that $\pi_\base(t|x, u)$ is proper distribution, which is equivalent to 
\begin{equation}
   \sumint_\cT \pi_\base(t'|x, u) \rd t' = 1 \text{ and } \pi_\base(t|x, u) \geq 0
   \text{ for any }
   t\in\cT, x\in\cX, \text{ and }u\in\cU.
   \label{eq:proper_base_policy}
\end{equation}

We relax constraint \eqref{eq:proper_base_policy} as conditional moment constraints
\begin{equation}
    \EE_\obs[w(Y, T, X)|T=t, X=x] \cdot p_\obs(t|x) = 1 \text{ for any } t\in\cT \text{ and }x\in\cX
    \label{eq:conditional_moment_constraints}
\end{equation}
and non-negativity condition
\begin{equation}
    w(y, t, x) \geq 0 \text{ for any } y\in\cY, t\in\cT, \text{ and }x\in\cX.
    \footnote{When $\cT$ is discrete, it is possible to further tighten this condition as $w\geq 1$ by using the fact that $\pi_\base(t|x, u)\leq 1$. However, we leave this condition as it is so that it can handle cases where $\cT$ is not discrete.}
    \label{eq:nonnegativity_constraint}
\end{equation}

To simplify the treatment of the non-negativity condition, we hereafter assume that the non-negativity condition is already imposed by conditional f-constraint \eqref{eq:relaxed_f_constraints} so that $f_{t,x}(u) = \infty$ for any $u < 0$.
This can be assumed without loss of generality because for any conditional-f constraint $\EE_\obs[f_{t, x}(w(Y, T, X)p_\obs(T|X))] \leq \gamma$, we can consider a new f-constraint $\EE_\obs[\tilde f_{t, x}(w(Y, T, X)p_\obs(T|X))] \leq \gamma$ for $\tilde f_{t, x}(v) = f_{t, x}(v) + I_{[0, \infty)}(v)$.

The validity of this relaxation can be checked as
\begin{align*}
\EE_\obs &[w(Y, T, X)|T=t, X=x]\cdot p_\obs(t, x) \\
&= \EE_\obs \left[
    \EE_{T\sim\pi_\base(\cdot|X, U)}\left[\frac{1}{\pi_\base(T|X, U)}|Y, T, X\right]
| T=t, X=x \right] \cdot p_\obs(t, x) \\
&= \EE_{T\sim\pi_\base(\cdot|X, U)}\left[ \frac{1}{\pi_\base(T|X, U)} | T=t, X=x \right] \cdot p_\obs(t, x) \\
&= \sumint_\cU \frac{1}{\pi_\base(t|x, u)} p_{\pi_\base}(u| t, x) \rd u \cdot p_\obs(t, x) \\
&= \sumint_\cU \frac{1}{\pi_\base(t|x, u)} p_{\pi_\base}(t, x, u) \rd u \\
&= \sumint_\cU p(x, u)\rd u 
= p_\obs(x)
\end{align*}
and the fact that $p_\obs(t|x) = p_\obs(t, x) / p_\obs(x)$.
Here, $p_{\pi_\base}(t, x, u)$ and $p_{\pi_\base}(u|t, x)$ denote the joint and conditional distribution of $T$, $X$, and $U$ under confounded contextual bandits model \eqref{eq:data_gen_base_confounded}.
To summarize, we combine these conditional moment constraints (CMC) with relaxed conditional f-constraint \eqref{eq:relaxed_f_constraints} to define relaxed uncertainty set
\begin{equation}
    \cW_{f_{t, x}}^\CMC:= \left\{
    w(y, t, x):
    \begin{array}{c}
        \EE_\obs \left[f_{T, X}\left(p_\obs(T|X)w(Y, T, X)\right)\right] \leq \gamma \\
        \text{and} \\
        \EE_\obs[w(Y, T, X)|T=t, X=x] \cdot p_\obs(t|x) = 1 \text{ for any } t\in\cT \text{ and } x\in\cX
    \end{array}
    \right\}
    \label{eq:conditional_moment_uncertainty_set}
\end{equation}
and its corresponding lower bound
\begin{equation}
V_\textinf^\CMC(\pi):= \inf_{w\in\cW^\CMC_{f_{t, x}}}\EE_\obs\left[w(Y, T, X)\pi(T|X)Y \right]. \label{eq:v_inf_conditional_moment_constraints}
\end{equation}

\begin{remark}[A comparison to ZSB sensitivity model]
Until the recent work by \citet{dorn2022sharp}, unsharp estimators of the lower bound had been the de facto standard in sensitivity analysis. Rather than adopting the conditional moment constraint \eqref{eq:conditional_moment_constraints} as relaxation, earlier works --- including \citet{tan2006distributional, zhao2019sensitivity, kallus2018confounding, kallus2021minimax} --- relied on the following looser relaxations, leading to unsharp estimates.
\begin{equation}
    \EE_\obs[\bbmone_{T=t}w(Y, T, X)] = 1 \text{ for any } t\in\cT,
    \label{eq:zsb_constraint}
\end{equation}
where $\bbmone$ denotes the indicator function for event $A$.
This condition can be obtained from conditional moment constraints \eqref{eq:conditional_moment_constraints} as $
\EE_\obs[\bbmone_{T=t}w(Y, T, X)] =
\EE_\obs[\bbmone_{T=t}\EE_\obs[w(Y, T, X)|T, X]] =
\EE_\obs[\bbmone_{T=t} / p_\obs(T | X)] = 1
\text{ for any } t\in\cT.$
Following the naming convention in \citet{dorn2022sharp}, we call these constraints the ZSB constraints after the authors of \citet{zhao2019sensitivity}.

Combining the above with the relaxation of the conditional f-constraint as in \eqref{eq:relaxed_f_constraints}, the following uncertainty set with the ZSB constraints becomes 
\begin{equation}
    \cW_{f_{t, x}}^\ZSB := \left\{
    w(y, t, x):
    \begin{array}{c}
        \EE_\obs \left[f_{T, X}\left(p_\obs(T|X)w(Y, T, X)\right)\right] \leq \gamma \\
        \text{and} \\
        \EE_\obs[\bbmone_{T=t}w(Y, T, X)] = 1 \text{ for any } t\in\cT
    \end{array}
    \right\},
    \label{eq:zsb_uncertainty_set}
\end{equation}
with associated lower bound
\begin{equation}
V_\textinf^\ZSB(\pi) := \inf_{w\in\cW^\ZSB_{f_{t, x}}}\EE_\obs\left[w(Y, T, X)\pi(T|X)Y \right] \label{eq:v_inf_zsb}.
\end{equation}

In practice, for the ZSB constraint to be tractable, we need to assume that action space $\cT$ is discrete and finite. This is in contrast to our approach, which directly approximates the conditional moment constraints and can handle both discrete and continuous action spaces.

Naturally, for the uncertainty sets discussed, one can show inclusion relations $\cW_{f_{t, x}} \subseteq \cW^\CMC_{f_{t, x}}\subseteq \cW^\ZSB_{f_{t, x}}$.
The former inclusion follows from the definition of $\cW^\CMC_{f_{t, x}}$, and the latter follows from the fact that the ZSB constraint follows from the conditional moment constraints as discussed above.
Under the assumption that the original sensitivity model is correctly specified --- meaning the true base policy $\pi_\base(t|x, u)$ lies within the original uncertainty set $\cE_{f_{t, x}}$ --- we can ensure that the tightness of the resulting lower bounds aligns with the order of relaxation, i.e., $V \geq V_\textinf \geq V_\textinf^\CMC \geq V_\textinf^\ZSB$.

\footnote{
Here, it is worth considering under what kind of setting the second equality $V_\textinf = V_\textinf^\CMC$ holds so that the lower bound of the conditional moment constraints is tight.
Surprisingly, a recent work by \citet[Section 3.2]{dorn2022sharp} showed that this equality holds for average treatment effect estimation with Tan's marginal sensitivity model \citep{tan2006distributional}.
They showed that there exists minimizer $w^*(y, t, x)$ of $\mathrm{ATE}(w):=\EE_\obs[w(Y, T, X)\bbmone_{T=1}Y] - \EE_\obs[w(Y, T, X)\bbmone_{T=0}Y]$, which is realizable so that there exists $p(y|t, x, u)$, $\pi_\base(t|x, u)$, and $p(x, u)$ that is compatible with any $p_\obs(y, t, x)$ and minimizer $w^*(y, t, x)$.
However, they did not provide any realizability results for more general settings discussed in this paper, such as the evaluation of general policy with non-binary action space and conditional f-constraint. This is an open question we could not address in our paper.
}

\end{remark}

\subsection{Low-rank approximation of the conditional moment constraints}


Now, we introduce an empirical approximation of conditional moment constraints \eqref{eq:conditional_moment_constraints} using low-rank approximation.

Let us introduce another re-parametrization $e(y, t, x) := p_\obs(t|x)w(y, t, x) - 1$ for further tractability so that conditional moment constraints \eqref{eq:conditional_moment_constraints} can be written as 
\begin{equation}
    \EE_\obs[e(Y, T, X)|T=t, X=x] = 0 \text{ \ \ for any \ \ }t\in\cT \text{ and }x\in\cX.
    \label{eq:conditional_moment_constraints_error_version}
\end{equation}
This is equivalent to orthogonality condition $\EE_\obs[e(Y, T, X) \psi(T, X)]=0$ for any function $\psi(t, x)$. 
However, imposing this orthogonality condition for all possible $\psi$ is infeasible if we impose the empirical version of the orthogonality constraints, as there are infinite choices of $\psi$. To cope with this issue, we apply low-rank approximation to the orthogonality conditions above. With a collection of basis functions $\{\psi_d(t, x)\}_{d=1}^D$, we define the kernel conditional moment constraints (KCMC) for the population as
\begin{equation}
    \KCMC
    \stackrel{\text{def}}{\Leftrightarrow} \EE_\obs[e(Y, T, X)\psi_d(T, X)] = 0  \text{ \ \ for any \ \ }d = 1, \ldots, D.
    \label{eq:kernel_conditional_moment_constraints_low_rank_orthogonality}
\end{equation}

We call this approximation the kernel conditional moment constraints because it can be motivated from the existing kernel methods.
In the preliminary version of our work \citep{ishikawa2023kernel}, we show that the low-rank kernel ridge regression can also be used to motivate the KCMC.
For example, it cound be considered a version of the kernel formulation of conditional moment constraints by \citet{kremer2022functional} and \citet{muandet2020kernel} using low-rank kernel $k\left((t, x), (t', x')\right) := \sum_{d=1, ..., D}\psi_d(t_1, x_1) \psi_d(t_2, x_2)$. \footnote{
Instead of the low-rank approximation, we could use the maximum moment restriction (MMR) as in \citet{kremer2022functional} and \citet{muandet2020kernel}. The MMR is a quadratic form that quantifies the violation of the conditional moment and upper bounding it leads to an alternative convex uncertainty set. 
However, in terms of the theoretical analysis, their approach is more technically involved, and deriving asymptotic results comparable to ours becomes substantially more complicated than in our finite-dimensional, low-rank setting, where our theoretical results can be derived as special cases of classic asymptotics in a finite-dimensional parameter space.
}
Additionally, as we discuss later in Example \ref{lemma:kpca_bias}, the principal components of the kernel PCA \citep{scholkopf1997kernel} can be used as a reasonable collection of basis functions. When we take first $D$ principal components found by kernel PCA as the basis functions for KCMC, we can prove that KCMC can approximate the original conditional moment constraints at an arbitrary precision as $D$ and $n$ go to infinity.

Analogously, the empirical version of the KCMC is also defined by replacing $\EE_\obs$ with $\hat\EE_n$.
Finally, combining the KCMC approximation \eqref{eq:kernel_conditional_moment_constraints_low_rank_orthogonality} and relaxed conditional f-constraint \eqref{eq:relaxed_f_constraints}, we can define the uncertainty set 
\begin{equation}
    \cW^\KCMC_{f_{t, x}} = \left\{
    w(y, t, x):
    \begin{array}{c}
         \EE_\obs[f_{T, X}(w(Y, T, X)p_\obs(T|X))] \leq \gamma \\
         \text{and} \\
         \EE_\obs[\left(w(Y, T, X)p_\obs(T|X) - 1\right)\psi_d] = 0 \text{ \ for any }d = 1, \ldots, D.
    \end{array}
    \right\}
\end{equation}
and the associated lower bound with the KCMC as 
\begin{equation}\label{eq:v_inf_kernel_conditional_moment_constraints}
    V_\textinf^\KCMC(\pi) = \inf_{w\in\cW^\KCMC_{f_{t, x}}}\EE_\obs[w(Y, T, X)\pi(T|X)Y].
\end{equation}
Likewise, we use its empirical version and define the KCMC estimator as
\begin{equation}
    \hat V_\textinf^\KCMC(\pi) = \inf_{\bsw\in\hat\cW^\KCMC_{f_{t, x}}}\hat\EE_n[w(Y, T, X)\pi(T|X)Y],
    \label{eq:empirical_v_inf_kernel_conditional_moment_constraints}
\end{equation}
where
\begin{equation}
    \hat\cW^\KCMC_{f_{t, x}} = \left\{
    w(y, t, x):
    \begin{array}{c}
         \hat\EE_n[f_{T, X}(w(Y, T, X)p_\obs(T|X))] \leq \gamma \\
         \text{and} \\
         \hat\EE_n[\left(w(Y, T, X)p_\obs(T|X) - 1\right)\psi_d] = 0 \text{ \ for any }d = 1, \ldots, D
    \end{array}
    \right\}.
\end{equation}
As can be seen, the KCMC estimator is the solution of a convex programming with a tractable feasibility set. Therefore, our estimator can easily implemented using any convex optimization solvers such as CVXPY \citep{diamond2016cvxpy}.

\section{Theoretical Analysis}\label{chap:theoretical_analysis}

In this section, we study the property of lower bound estimate of the policy value with the KCMC estimator.
We begin the analysis of the KCMC estimator by deriving the convex dual of the above problems and characterizing their minimizers.
Then, we study the specification error of the KCMC lower bound $\left|V^\CMC_\textinf - V^\KCMC_\textinf\right|$. 
We provide a condition on the orthogonal function set $\{\psi_d(t, x)\}_{d=1}^D$ under which the specification error can be bounded.
Next, we show the consistency of the KCMC estimator in policy evaluation by showing the equivalence of the KCMC estimator with a standard M-estimator. 
We further prove consistency guarantees for policy learning with the KCMC estimator, in the cases of a concave policy class and the Vapnik-Chevonenkis policy class.
Lastly, we derive the asymptotic normality of the KCMC estimator and discuss its application to the construction of a confidence interval and model selection.
For the sake of conciseness, we defer some proofs and detailed regularity conditions to Appendix \ref{app:omitted_proofs}.

Before further discussion, we introduce several simplifications of the notations.
We omit subscripts of $\cW_{f_{t, x}}^\CMC$, $\cW_{f_{t, x}}^\KCMC$, $\hat\cW_{f_{t, x}}^\KCMC$ and $\EE_\obs$, unless they are unclear from the context.
We also introduce $r(y, t, x) = \left(\frac{\pi(t|x)}{p_\obs(t|x)}\right)\cdot y$ and reparametrization $\tilde w(y, t, x) = p_\obs(t|x)w(y, t, x)$.

Furthermore, we briefly review the subgradient and the Fenchel conjugate here, as we will make heavy use of them in this section.
The subgradient of convex function $f$ is represented by $\partial f$.
When we apply the addition operator to the subgradient, it represents the Minkowski sum.
Other operations to the subgradient such as multiplication are similarly defined.
The Fenchel conjugate of $f:\RR\to\RR$ is defined as $f^*(v):=\sup_{u}\{uv - f(u)\}$, and it has a few important properties.
The Fenchel conjugate is always convex because the supremum of a family of convex functions is convex.
\footnote{In this case, $v\mapsto vu - f(u)$ is linear, and therefore, is convex.}
Additionally, there exists maximizer $u^*$ that solves $f^*(v) = \max_{u}\left\{vu - f(u)\right\}$ and it satisfies 
\begin{equation}
    u^* \in \partial f^*(v)
    \label{eq:fenchel_conjugate_solution}
\end{equation}
if $f$ is closed and convex.
Function $f$ is closed and convex if its epigraph $\mathrm{epi}(f):=\{(u, t): u\in\mathrm{dom}(f), \ f(u)\leq t\}$ is closed and convex, and these conditions are satisfied in our problem.
A more thorough treatment of the subgradient and the Fenchel conjugate can be found in  \citep{boyd2004convex}.

\subsection{Characterization of the solution}
In this subsection, we derive explicit formulae for the minimizers that give three lower bounds $V_\textinf^\CMC$, $V_\textinf^\KCMC$, and $\hat V_\textinf^\KCMC$, which are
\begin{align}
\begin{split}
    w^*_\CMC &= \arg\min_{w\in\cW^\CMC}\EE[w(Y, T, X)\pi(T|X)Y], \\
    w^*_\KCMC &= \arg\min_{w\in\cW^\KCMC}\EE[w(Y, T, X)\pi(T|X)Y], \\
    \hat w_\KCMC &= \arg\min_{w\in\hat\cW^\KCMC}\hat\EE_n[w(Y, T, X)\pi(T|X)Y].
\end{split}
\label{eq:w_solutions}
\end{align}
Here, we know that these problems have minimizers because the above problems are minimizations of linear objectives under convex constraints.
Furthermore, we know that the strong duality holds for the above convex optimizations, as their feasibility sets have a non-empty relative interior, satisfying Slater's constraint qualification.
Using these properties, we obtain the following lemma:

\begin{lemma}[Characterization of solutions]\label{lemma:w_characterization}
Let $w^*_\CMC$, $w^*_\KCMC$, and $\hat w_\KCMC$ be defined as in \eqref{eq:w_solutions}. 
Then, there exist function $\eta_\CMC:\cT\times\cX\to\RR$, vectors $\eta_\KCMC, \eta_\KCMC'\in\RR^D$, and constants $\eta_f, \eta_f', \eta_f'' > 0$ such that 
\begin{align}
w^*_\CMC(y, t, x) &\in \left( \frac{1}{p_\obs(t|x)} \right) \partial f_{t, x}^* \left( \frac{\eta_\CMC(t, x) - r(y, t, x)}{\eta_f} \right),
\label{eq:cmc_solution_characterization} \\
w^*_\KCMC(y, t, x) &\in \left( \frac{1}{p_\obs(t|x)} \right) \partial f_{t, x}^* \left( \frac{{\eta_\KCMC}^T\bspsi(t, x) - r(y, t, x)}{\eta_f'} \right),
\label{eq:kcmc_solution_characterization} \\
\hat w_\KCMC(y, t, x) &\in \left( \frac{1}{p_\obs(t|x)} \right) \partial f_{t, x}^* \left( \frac{{\eta_\KCMC'}^T\bspsi(t, x) - r(y, t, x)}{\eta_f''} \right).
\label{eq:empirical_kcmc_solution_characterization}
\end{align}  
\end{lemma}

\begin{proof}
    See below.
\end{proof}

\subsubsection*{Characterization of $w^*_\CMC$}

By using the strong duality, we can transform the original problem for $V_\textinf^\CMC$ as
\begin{align}
V_\textinf^\CMC
&= \inf_{w\in\cW^\CMC}\EE[\tilde w(Y, T, X)r(Y, T, X)] \\
&= \inf_{\tilde w}\sup_{\substack{\eta_\CMC:\cT\times\cX\to\RR,\\ \eta_f\geq 0}}
    \EE[\tilde w r] 
    - \EE\left[\eta_\CMC(T, X)(\tilde w - 1)\right] 
    + \eta_f\left(\EE[f_{T, X}(\tilde w)] - \gamma\right) \\
&= \sup_{\substack{\eta_\CMC:\cT\times\cX\to\RR,\\ \eta_f\geq 0}}\inf_{\tilde w}
    \EE[\tilde w r] 
    - \EE\left[\eta_\CMC(T, X)(\tilde w - 1)\right] 
    + \eta_f\left(\EE[f_{T, X}(\tilde w)] - \gamma\right) \\
&= \sup_{\substack{\eta_\CMC:\cT\times\cX\to\RR,\\ \eta_f > 0}}
    - \eta_f \gamma 
    + \EE[\eta_\CMC]
    - \eta_f \EE\left[\sup_{\tilde w}\left\{\left(\frac{\eta_\CMC - r}{\eta_f}\right)\tilde w - f_{T, X}(\tilde w)\right\}\right]  \\
&= \sup_{\substack{\eta_\CMC:\cT\times\cX\to\RR,\\ \eta_f > 0}}
    - \eta_f \gamma 
    + \EE[\eta_\CMC]
    - \eta_f \EE\left[f_{T, X}^*\left(\frac{\eta_\CMC - r}{\eta_f}\right)\right],
    \label{eq:v_inf_conditional_moment_constraints_dual}
\end{align}
assuming that optimal $\eta_f$ is positive. 
This assumption does not hold only when conditional f-constraint \eqref{eq:relaxed_f_constraints} is not tight.
However, in such cases,
\begin{align}
V_\textinf^\CMC
&= \sup_{\eta_\CMC:\cT\times\cX\to\RR} \inf_{\tilde w}
    \EE[\tilde w r] 
    - \EE\left[\eta_\CMC(T, X)(\tilde w - 1)\right] \\
&= \sup_{\eta_\CMC:\cT\times\cX\to\RR} \inf_{\tilde w}
    \EE\left[\tilde w(Y, T, X) (r(Y, T, X) - \eta_\CMC(T, X))\right]
    + \EE\left[\eta_\CMC(T, X)\right]
\label{eq:v_inf_conditional_moment_constraints_dual_zero_eta_f}
\end{align}
and in order for the supremum to be reached for $\eta_f=0$, we need
$r(Y, T, X) - \eta_\CMC(T, X) = 0$ almost surely, which is not satisfied in practice.
Therefore, for simplicity in our proofs, we will focus on cases where $\eta_f > 0$.

Now, as primal solution $w^*_\CMC$ must satisfy the Karush-Kuhn-Tucker (KKT) conditions, we can take the maximizers of \eqref{eq:v_inf_conditional_moment_constraints_dual} as $\eta_f^*$ and $\eta_\CMC^*(t, x)$.
Using property \eqref{eq:fenchel_conjugate_solution} of the Fenchel conjugate, we can obtain solution form of $w^*_\CMC$
\begin{equation}
    w^*_\CMC(y, t, x) \in  \left( \frac{1}{p_\obs(t|x)} \right) \partial f_{t, x}^* \left( \frac{\eta_\CMC^*(t, x) - r(y, t, x)}{\eta_f^*} \right),
\end{equation}
which proves \eqref{eq:cmc_solution_characterization}. Here, factor $1/p_\obs(t|x)$ in front of the subgradient appears because of reparametrization $\tilde w(y, t, x) = p_\obs(t|x)w(y, t, x)$.

\subsubsection*{Characterization of $w^*_\KCMC$}
Now, we derive the characterization of $w^*_\KCMC$.
Let $\bspsi(t, x) = \left(\psi_1(T, X), \ldots, \psi_D(T, X)\right)^T$.
We can use exactly the same technique as the above and reach a similar characterization of the solution as
\begin{align}
V_\textinf^\KCMC
&= \inf_{w\in\cW^\KCMC}\EE[\tilde w(Y, T, X)r(Y, T, X)] \\
&= \inf_{\tilde w}\sup_{\substack{\eta_\KCMC\in\RR^D,\\ \eta_f\geq 0}}
    \EE[\tilde w r] 
    - \EE\left[(\tilde w - 1) {\eta_\KCMC}^T\bspsi\right] 
    + \eta_f\left(\EE[f_{T, X}(\tilde w)] - \gamma\right) \\
&= \sup_{\substack{\eta_\KCMC\in\RR^D,\\ \eta_f\geq 0}} \inf_{\tilde w}
    \EE[\tilde w r] 
    - \EE\left[(\tilde w - 1) {\eta_\KCMC}^T\bspsi\right] 
    + \eta_f\left(\EE[f_{T, X}(\tilde w)] - \gamma\right) \\
&= \sup_{\substack{\eta_\KCMC\in\RR^D,\\ \eta_f > 0}}
    - \eta_f \gamma 
    + {\eta_\KCMC}^T\EE[\bspsi]
    - \eta_f \EE\left[\sup_{\tilde w}\left\{\left(\frac{{\eta_\KCMC}^T\bspsi - r}{\eta_f}\right)\tilde w - f_{T, X}(\tilde w)\right\}\right]  \\
&= \sup_{\substack{\eta_\KCMC\in\RR^D,\\ \eta_f > 0}}
    - \eta_f \gamma 
    + {\eta_\KCMC}^T\EE[\bspsi]
    - \eta_f \EE\left[f_{T, X}^*\left( \frac{{\eta_\KCMC}^T\bspsi - r}{\eta_f}\right)\right].
    \label{eq:v_inf_kernel_conditional_moment_constraints_dual}
\end{align}
Now, using the maximizers of dual problem \eqref{eq:v_inf_kernel_conditional_moment_constraints_dual} $\eta_f^*$ and $\eta_\KCMC^*$, we can obtain a characterization of $w^*_\KCMC$ as
\begin{equation}
    w^*_\KCMC(y, t, x) \in  \left( \frac{1}{p_\obs(t|x)} \right) \partial f^* \left( \frac{{\eta_\KCMC^*}^T\bspsi(t, x) - r(y, t, x)}{\eta_f^*} \right),
    \label{eq:w_solution_kernel_conditional_moment_constraints}
\end{equation}
which proves \eqref{eq:kcmc_solution_characterization}.


\subsubsection*{Characterization of $\hat w_\KCMC$}
Again, using the same techniques, we can derive the characterization of $\hat w_\KCMC$.
By exchanging $\EE$ with $\hat \EE_n$ in the proof for $w^*_\KCMC$ above and writing the maximizers of dual problem
\begin{equation}
    \sup_{\substack{\eta_\KCMC\in\RR^D,\\ \eta_f > 0}}
    - \eta_f \gamma 
    + {\eta_\KCMC}^T\hat\EE_n[\bspsi]
    - \eta_f \hat\EE_n\left[f_{T, X}^*\left( \frac{{\eta_\KCMC}^T\bspsi - r}{\eta_f}\right)\right].
    \label{eq:empirical_v_inf_kernel_conditional_moment_constraints_dual}
\end{equation}
as $\hat\eta_f$ and $\hat\eta_\KCMC$, we get
\begin{equation}
    \hat w_\KCMC(y, t, x) \in  \left( \frac{1}{p_\obs(t|x)} \right)
    \partial f_{t, x}^* \left(
        \frac{{{}\hat\eta_\KCMC}^T\bspsi(t, x) - r(y, t, x)}{{\hat\eta}_f}
    \right),
    \label{eq:empirical_w_solution_kernel_conditional_moment_constraints}
\end{equation}
which proves \eqref{eq:empirical_kcmc_solution_characterization}.

For general choice of $f_{t, x}$, it is difficult to derive analytical expressions for solutions $\eta_f^*$ and $\eta^*_\CMC$, as well as $w^*_\CMC$.
However, we can actually obtain their explicit expressions in the case of the box constraints.

\begin{example}[Solutions for box constraints]\label{ex:box_constraint_analytical_solution}
Let us consider the box constraints corresponding to \eqref{eq:f_box_constraints} of the conditional f-constraint.
Here, for notational simplicity, we omit the subscript of $a_{\tilde w}$ and $b_{\tilde w}$.
For example, we will simply write $f_{t, x}(\tilde w) = I_{[a(t, x), b(t, x)]}(\tilde w)$.
Then, we can show that \begin{equation}
    \eta^*_\CMC(t, x) = \left( \frac{\pi(t|x)}{p_\obs(t|x)} \right)Q(t, x),
\end{equation}
where $Q(t, x)$ is defined as the $\tau(t, x)$-quantile of the conditional distribution of $Y$ given $T=t$ and $X=x$ 
for $ \tau(t, x) := \frac{1 - a(t, x)}{b(t, x) - a(t, x)}$.
From this dual solution, the primal solution can also be recovered using \eqref{eq:cmc_solution_characterization} as
\begin{equation}
    w^*_\CMC(y, t, x) =
    \begin{cases}
        b(t, x) & \text{ if \ \ }y \leq Q(t, x), \\
        a(t, x) & \text{ otherwise}.
    \end{cases}
\end{equation}
    See Appendix \ref{app:proof_box_constraint_analytical_solution} for details.
\end{example}

\subsection{Specification error}

Now we study the specification error of estimator $\hat V_\textinf^\KCMC(\pi)$, defined as $\left| V_\textinf^\KCMC(\pi) - V_\textinf^\CMC(\pi) \right|$.
It turns out that it is possible to upper bound the specification error if the dual objective satisfies a Lipschitz condition.

\begin{theorem}\label{th:specification_error_liptchitz_bound}
Let $\|f\|$ denote the $L^2(\cT\times\cX, p_\obs)$ norm for any $f\in L^2(\cT\times\cX, p_\obs)$.
Let $\Pi_{\bspsi}$  
be the projection operator onto the subspace spanned by $\{\psi_1, \ldots, \psi_D\}$ in $L^2(\cT\times\cX, p_\obs)$,
and let $\left(\eta^*_f, \eta^*_\CMC\right)$ be the solution of dual problem \eqref{eq:v_inf_conditional_moment_constraints_dual} for $V_\textinf^\CMC$.
Additionally, define convex functional $J:L^2(\cT\times\cX, p_\obs)\to\RR$ corresponding to the negative of dual objective \eqref{eq:v_inf_conditional_moment_constraints_dual}
for  $V_\textinf^\CMC$ with fixed $\eta^*_f$ as
\begin{equation}
J[h] := \eta^*_f \gamma - \EE[h] 
    + \eta^*_f \EE\left[
        f_{t, x}^*  \left( \frac{h(T, X) - r(Y, T, X)}{\eta^*_f} \right)
    \right].
\end{equation}
Let us choose $\{\psi_1, \ldots, \psi_D\}$ such that $\|\Pi_{\bspsi} \eta_\CMC^* - \eta_\CMC^*\| \leq \varepsilon_\text{spec}$ for some $\varepsilon_\text{spec}>0$.
Then, if $J[h]$ is $L$-Lipschitz in neighborhood $\{h:\ \| h - \eta^*_\CMC \| \leq \varepsilon_\text{spec}\}$, or equivalently, if functional subgradient $\partial J$ satisfies $\| \delta J \| \leq L$ for any $\delta J\in \bigcup_{h:\ \| h - \eta^*_\CMC \| \leq \varepsilon_\text{spec}}\partial J[h]$, we have 
\begin{equation}
    \left|V_\textinf^\CMC - V_\textinf^\KCMC \right| \leq L \|\Pi_{\bspsi} \eta_\CMC^* - \eta_\CMC^*\|.
\end{equation}
Moreover, even if $J[h]$ does not satisfy the $L$-Lipschitz condition, if 
\begin{equation}
    \eta_\CMC^* \in \mathrm{span}\left(\{\psi_1, \ldots, \psi_D\}\right),
    \label{eq:eta_cmc_in_kernel_subspace}
\end{equation}
we have no specification error so that
\begin{equation}
V_\textinf^\CMC(\pi) = V_\textinf^\KCMC(\pi).
\end{equation}
\end{theorem}

\begin{proof}
See Appendix \ref{app:proof_specification_error_lipschitz_bound}. 
\end{proof}

In Theorem \ref{th:specification_error_liptchitz_bound}, we provided the upper bound of the specification error in terms of the Lipschitz constant. 
A natural follow-up question would be whether it is possible to know such a constant.
Indeed, it is possible to calculate the Lipschitz constant in some cases.

\begin{example}[Lipschitz constant for box constraints]\label{ex:lipschitz_box}
Consider the same settings as Example \ref{ex:box_constraint_analytical_solution}.
Then, if the upper bound of the box constraint satisfies $b \in L^2(\cT\times\cX, p_\obs)$, the Lipschitz condition for Theorem \ref{th:specification_error_liptchitz_bound} is satisfied.
See Appendix \ref{app:proof_lipschitz_box} for details.
\end{example}

\begin{example}[Lipschitz constant for bounded conditional f-constraint]\label{ex:lipschitz_bounded_f}
Consider an extension of the previous example where we additionally impose a conditional f-constraint with $f^0_{t, x}$ so that $f_{t, x}(\tilde w):= I_{[a(t, x), b(t, x)]}(\tilde w) + f^0_{t, x}(\tilde w)$.
The uncertainty set under this constraint can be considered as the intersection of the uncertainty sets of the box constraints and the conditional f-constraint.
Interestingly, even for this constraint, we obtain the same Lipschitz constant as the previous example.
See Appendix \ref{app:proof_lipschitz_bounded_f} for details.
\end{example}

Note that the above examples provide a uniform Lipschitz constant for any policy $\pi$.
This is in contrast to Theorem \ref{th:specification_error_liptchitz_bound}, which only provide pointwise error bounds for fixed policy $\pi$ as it assumes that $\eta^*_\CMC$ is fixed.
Unfortunately, the uniform Lipschitz constant alone cannot provide a uniform bound on the specification error.
Bounding the specification error calls for a uniform bound $\|\Pi_{\bspsi} \eta^*_\CMC - \eta^*_\CMC\|$, but the derivation of such a bound is difficult.
Thus, in the policy learning, we have no choice but to optimize $V_\textinf^\KCMC$, which is only guaranteed to be lower than $V_\textinf^\CMC$.

However, in the policy evaluation, it is possible to provably reduce $\|\Pi_{\bspsi} \eta^*_\CMC - \eta^*_\CMC\|$  when we pick the orthogonal functions for kernel conditional moment constraints using the kernel principal component analysis (PCA) \citep{scholkopf1997kernel}.

\begin{lemma}[Convergence of $\|\Pi_{\bspsi} \eta^*_\CMC - \eta^*_\CMC\|$ with kernel PCA]\label{lemma:kpca_bias}
Let $k\left((t, x), (t', x')\right)$ be a kernel of feature $(T, X)$. 
Let $\bsphi^\KPCA(t, x):=\left(\phi^\KPCA_1(t, x), \ldots, \phi^\KPCA_D(t, x)\right)^T$ be the kernel principal components that approximates the original kernel as $k\left((t, x), (t', x')\right) \approx {\bsphi^\KPCA(t, x)}^T \bsphi^\KPCA(t', x')$ obtained by applying the kernel PCA to the observations of $\left\{(T_i, X_i)\right\}_{i=1}^n$.
Now, suppose the orthogonal functions are chosen as $\bspsi := \bsphi^\KPCA(t, x)$.
Then, if kernel $k\left((t, x), (t', x')\right)$ satisfies a set of regularity conditions \footnote{The formal version of this statement can be found in Appendix \ref{app:proof_kpca_bias}.}, we have 
\[
    \lim_{D\to\infty}\plim_{n\to\infty} \|\Pi_{\bspsi} \eta^*_\CMC - \eta^*_\CMC\| = 0,
\]
where $\plim$ denotes convergence in probability.
\end{lemma}
\begin{proof}
   See Appendix \ref{app:proof_kpca_bias}.
\end{proof}

Therefore, the kernel PCA can be a useful tool for obtaining an orthogonal basis function set with (close to) zero specification error. When using finite basis functions, manually crafting the basis functions satisfying such a condition is practically difficult. However, the use of kernel PCA allows us to gradually reduce the specification error as we increase the number of basis functions and the sample size, which enables us to achieve consistency under an appropriate choice of kernel thanks to Theorem \ref{th:specification_error_liptchitz_bound} and \ref{th:policy_evaluation_consistency}. In practice, the GIC from Example \ref{ex:gic} (or cross-validation) can be used to optimally balance the number of basis functions(i.e., kernel principal components) and the sample size.

Additionally, with the latter statement in Theorem \ref{th:specification_error_liptchitz_bound} on the no specification bias, 
we can derive the previously proposed sharp estimator by \citet{dorn2022sharp} as a special case.

\begin{example}[Derivation of quantile balancing estimator \citet{dorn2022sharp}]\label{ex:quantile_balancing_as_a_special_case}
Let us consider the box constraints as in Example \ref{ex:box_constraint_analytical_solution} and choose $a(t, x)$ and $b(t, x)$ to be the marginal sensitivity model \eqref{eq:tan_box_constraints} so that $\tau(t, x) = \tau = \frac{1}{1 + \Gamma}$.
Let $\hat Q(t, x)$ be a solution of $\tau$-quantile regression of $Y$ on $(T, X)$.
Then, by using a single basis function in the KCMC so that $D=1$ and $\psi_1(t, y)= \left(\frac{\pi(t|x)}{p_\obs(t| x)}\right)\hat Q(t, x)$, we recover the original quantile balancing (QB) estimator by \citep{dorn2022sharp}.
\footnote{
Technically, their estimator employs the linear fractional programming technique from \citet{zhao2019sensitivity}, but we consider the QB estimator without such a technique here.
}
Especially, if the estimated quantile equals the population quantile so that $\hat Q(t, x) = Q(t, x)$, both KCMC and QB estimators have no specification bias.

Furthermore, suppose quantile estimate $\hat Q(t, x)$ is estimated by a linear quantile regression with $D$-dimensional features.
Then, we can construct $D$-dimensional basis functions for the KCMC using the same features and guarantee that the KCMC estimate of the lower bound is tighter than or equal to that of the QB estimator.
\end{example}
\begin{proof}
    See Appendix \ref{app:proof_quantile_balancing_as_a_special_case}.
\end{proof}

\begin{remark}[Comparison of the KCMC estimator and the QB estimator]
As our estimator generalizes the previous work, our estimator overcomes some drawbacks of the quantile balancing estimators.
As discussed in Section \ref{chap:intro}, the quantile balancing estimator cannot handle policy learning and the f-constraint.
Policy learning is also difficult with the quantile balancing estimator because taking the derivative with respect to policy requires differentiability of the solution of the above linear programming with respect to parameter $\hat Q$.  
Moreover, the quantile balancing method is designed only for the box constraints and does not have a proper extension to the f-sensitivity model \eqref{eq:f_policy_uncertainty_set}.
In contrast, our estimator of sharper bound $V_\textinf^\CMC$ based on the kernel method can naturally handle the above-mentioned generalized cases of sensitivity analysis.
\end{remark}

\subsection{Consistency of policy evaluation}

Now, we study empirical estimator $\hat V_\textinf^\KCMC$ and provide convergence guarantees for policy evaluation and learning.
First, we prove the consistency of our estimator for fixed policy $\pi$ by reducing our problem to the M-estimation \citep{van2000empirical} using the dual formulation.

\subsubsection*{Consistency of policy evaluation}


To align our analysis with the classic framework of M-estimation, let us introduce loss function $\ell:\Theta\times\cZ\to\RR$, where $\Theta\subseteq\RR^K$ for some $K$ and $\cZ:=\cY\times\cT\times\cX$.
Then, the following consistency result holds:

\begin{theorem}[Consistency of policy evaluation (informal version)]\label{th:policy_evaluation_consistency}
Define the parameter space of $\left(\eta_f, \eta_\KCMC\right)$ as $\Theta \subseteq \RR_+\times\RR^D$.
Further, define 
$\theta^* :=(\eta^*_f, \eta^*_\KCMC)$ as the solution of dual problem \eqref{eq:v_inf_kernel_conditional_moment_constraints_dual} for $V_\textinf^\KCMC(\pi)$ and 
$\hat\theta_n :=(\hat\eta_f, \hat\eta_\KCMC)$ as the solution to dual problem \eqref{eq:empirical_v_inf_kernel_conditional_moment_constraints_dual} for $\hat V_\textinf^\KCMC(\pi)$.
Define $\ell:\Theta\times\cZ\to\RR$ as
\begin{equation}
    \ell_\theta (t, y, x) := \eta_f \gamma - {\eta_\KCMC}^T \bspsi(t, x)
    + \eta_f
        f_{t, x}^*  \left( \frac{{\eta_\KCMC}^T \bspsi(t, x) - r(y, t, x)}{\eta_f} \right)
    \label{eq:dual_loss_policy_evaluation}
\end{equation}
so that it is the negative version of dual objectives \eqref{eq:v_inf_kernel_conditional_moment_constraints_dual} and  \eqref{eq:empirical_v_inf_kernel_conditional_moment_constraints_dual} before taking the expectations.
Then, if $\ell_\theta (t, y, x)$ satisfies a set of regularity conditions\footnote{The formal version of this statement can be found in Appendix \ref{app:proof_policy_evaluation_consistency}.}, we have 
$\hat\theta_n\pto\theta^*$
and $\hat V_\textinf^\KCMC(\pi)  \pto V_\textinf^\KCMC(\pi)$.
\end{theorem}
\begin{proof}
    See Appendix \ref{app:proof_policy_evaluation_consistency}
\end{proof}


\subsection{Consistency of policy learning}

Here, we provide consistency guarantees for policy learning with the KCMC estimator for a finite dimensional concave policy class and a Vapnik-Chevonenkis (VC) policy class. 
Again, we take advantage of the reduction to the M-estimation.
This simplifies the proof compared to the one in \citet{kallus2021minimax} using the original nested max-min formulation because the max-max formulation we use is simple and unnested maximization, for which the well-studied theory of M-estimation can be immediately applied.

For both proofs, we define a new parameter space and a loss function.
Let us define parameter $\theta=\left(\beta, \eta_f, \eta_\KCMC\right)$, $\eta = \left(\eta_f, \eta_\KCMC\right)$ and its space $\Theta = \cB\times H$ and $H = \RR_+\times\RR^D$.
Define loss function $\ell:\Theta\times\cZ\to\RR$ as
\begin{equation}
    \ell_\theta (t, y, x) := \eta_f \gamma - {\eta_\KCMC}^T \bspsi(t, x)
    + \eta_f
        f_{t, x}^*  \left( \frac{{\eta_\KCMC}^T \bspsi(t, x) - \left(\frac{\pi_\beta(t|x)}{p_\obs(t|x)}\right)y}{\eta_f} \right)
    \label{eq:dual_loss_policy_learning}
\end{equation}
so that it is the negative version of dual objectives \eqref{eq:v_inf_kernel_conditional_moment_constraints_dual} and  \eqref{eq:empirical_v_inf_kernel_conditional_moment_constraints_dual} before taking the expectations.
Define also
$\theta^* :=(\eta^*_f, \eta^*_\KCMC, \beta^*)$ 
so that $\beta^* \in \arg\max_{\beta\in\cB} V_\textinf^\KCMC(\pi_\beta)$ 
and $\eta^* := \left(\eta^*_f, \eta^*_\KCMC\right)$ is the solution of dual problem \eqref{eq:v_inf_kernel_conditional_moment_constraints_dual}
for $V_\textinf^\KCMC$ at policy $\pi_{\beta^*}$. 
Similarly, define
$\hat\theta_n :=(\hat\eta_f, \hat\eta_\KCMC, \hat\beta)$ 
so that $\hat\beta \in \arg\max_{\beta\in\cB} \hat V_\textinf^\KCMC(\pi_\beta) $ 
and $\hat\eta := \left(\hat\eta_f, \hat\eta_\KCMC\right)$ is the solution of dual problem 
\eqref{eq:empirical_v_inf_kernel_conditional_moment_constraints_dual}
for $\hat V_\textinf^\KCMC$ at policy $\pi_{\hat\beta}$. 

\subsubsection*{Concave Policy Learning}

With the above definitions, we can now show the consistency of policy learning with a concave policy class.

\begin{theorem}[Consistency of concave policy learning (informal version)]\label{th:concave_policy_learning}
Assume concave policy class $\{\pi_\beta(t|x):\ \beta\in\cB\}$ with convex parameter space $\cB$ satisfying that $\beta\mapsto\pi_\beta(t|x)y$ is concave for any $y\in\cY$, $t\in\cT$ and $x\in\cX$.
Then, if loss function \eqref{eq:dual_loss_policy_learning} satisfies a set of regularity conditions\footnote{The formal version of this statement can be found in Appendix \ref{app:proof_concave_policy_learning}.},  we have
$\hat\theta_n\pto\theta^*$
and 
$\hat V_\textinf^\KCMC(\pi_{\hat\beta}) \pto V_\textinf^\KCMC(\pi_{\beta^*})$.
\end{theorem}
\begin{proof}
    See Appendix \ref{app:proof_concave_policy_learning}.
\end{proof}
An example of concave policy is mixed policy $\pi_\beta(t|x):=\sum_k\beta_k\pi_k(t|x)$ for $\sum_k\beta_k=1$, $\beta_k\geq0$. Indeed, policy learning with such a concave policy class is concave; therefore, the globally optimal policy can be found by convex optimization algorithms.

\subsubsection*{Vapnik–Chervonenkis Policy Learning}
Now, let us discuss the consistency of policy learning with a VC policy class.
While the proof for VC policy class is more involved than the concave policy class, we can prove a similar consistency guarantee for policy learning with a VC policy class:

\begin{theorem}[Consistency of VC policy learning (informal version)]\label{th:vc_policy_learning}
Let $\{\pi_\beta(t|x):\ \beta\in\cB\}$ be a class of policies that is VC.
Under a set of regularity conditions\footnote{The formal version of this statement can be found in Appendix \ref{app:proof_vc_policy_learning}.}, we have
$\hat V_\textinf^\KCMC(\pi_{\hat\beta}) \pto V_\textinf^\KCMC(\pi_{\beta^*})$.
\end{theorem}
\begin{proof}
    See Appendix \ref{app:proof_vc_policy_learning}.
\end{proof}
Here, we adopt the same definition of the VC function class as in \citet[Definition 6.4.1]{van2020empirical}, where they define it as the class of functions whose subgraphs collectively form a VC set.

\subsection{Asymptotic normality of policy evaluation}

Now, we consider the asymptotic normality in the policy evaluation settings. 

\begin{theorem}[Asymptotic normality (informal version)]\label{thm:asymptotic_normality}
Under a set of regularity conditions\footnote{The formal version of this statement can be found in Appendix \ref{app:proof_asymptotic_normality}.}, we have
\begin{equation}
    \sqrt{n}\left[\hat\EE_n \ell_{\hat\theta}(Z) - \hat\EE_n \ell_{\theta^*}(Z) \right] \pto 0,
\end{equation}
and
\begin{equation}
    \sqrt{n}\left[
        V (\hat\theta  - \theta^*) + \hat\EE_n \delta \ell_{\theta^*}(Z)
    \right]
    \pto 0.
\end{equation}
Therefore,
\begin{equation}
    \sqrt{n}[ \hat\theta  - \theta^*]
    \pto  V^{-1} \sqrt{n} \hat\EE_n [\delta\ell_{\theta^*}(Z)]
    \dto N(0, V^{-1}J V^{-1}),
\end{equation}
where $V$, $J$ are defined as $V:=\nabla_\theta \nabla_\theta^T \EE\ell_{\theta^*}(Z)$ and $J:=\EE\left[ \delta \ell_{\theta^*}(Z) \delta \ell_{\theta^*}(Z)^T \right]$. Convergence symbol $\dto$ implies the convergence in distribution.
\end{theorem}
\begin{proof}
    See Appendix \ref{app:proof_asymptotic_normality}.
\end{proof}

Now, let us consider a few applications of the above asymptotic result.

\begin{example}[Confidence interval]\label{ex:confidence_interval}
Under Condition \ref{cond:regularity_1} and Condition \ref{cond:regularity_3}, we know that empirical objective
$\sqrt{n}\hat\EE_n\ell_{\hat\theta}$ has the same asymptotic distribution as $\sqrt{n}\EE_n\ell_{\theta^*}$,
which is $N\left(\EE[\ell_{\theta^*}(Z)], \VV[\ell_{\theta^*}(Z)]\right)$.
Therefore, the confidence interval of the lower bound $V_\textinf^\KCMC(\pi) = -\EE[\ell_{\theta^*}(Z)]$ with significance level $\alpha$ is
\begin{equation}
    C_n^{V_\textinf^\KCMC} = \left[C_n^-, C_n^+ \right]
\end{equation}
for
\begin{align}
    C_n^\pm
    &= \EE[\ell_{\theta^*}(Z)] \pm \sqrt{\VV[\ell_{\theta^*}(Z)] / n} \cdot \Phi^{-1}(1 - \alpha / 2),
\end{align}
where
$\Phi(s)$ defined as the inverse of the cumulative density function of standard normal distribution $\Phi(s):= \int_{-\infty}^s \frac{1}{\sqrt{2\pi}} \exp\left(\frac{1}{2}t^2\right)dt$.
In practice, we can estimate $\EE[\ell_{\theta^*}(Z)]$ and $\VV[\ell_{\theta^*}(Z)]$
with the sample averages using the M-estimator in place of the true parameter as $\hat\EE_n[\ell_{\hat\theta}(Z)]$ and $\hat\VV_n[\ell_{\theta}(Z)]$.
\end{example}

\begin{example}[Generalized information criterion]\label{ex:gic}
In the previous example, we approximated true lower bound $\EE[\ell_{\theta^*}(Z)]$ with $\hat\EE[\ell_{\hat\theta}(Z)]$.
Though this approximation is a correct first order approximation, we can consider its second order correction as follows, as discussed in Appendix \ref{app:proof_gic}:
\begin{align}
    n\hat\EE_n[\ell_{\hat\theta}(Z) - \ell_{\theta^*}(Z)] 
    \pto - \frac{1}{2}\sqrt{n}\left(\hat \theta - \theta^*\right)^T V \sqrt{n} (\hat\theta - \theta^*).
\end{align}
Thus, the bias of $\hat\EE[\ell_{\hat\theta}(Z)]$ can be written as 
$-\EE\left[\frac{1}{2}\left(\hat \theta - \theta^*\right)^T V (\hat\theta - \theta^*)\right] = - \frac{1}{2n}\mathrm{tr}\left[V^{-1}J\right]$
since $n \cdot \EE\left[\left(\hat \theta - \theta^*\right)(\hat\theta - \theta^*)^T\right] \pto V^{-1}JV^{-1}$.
Using this second order bias correction, we can obtain the generalized information criterion (GIC) \citep{konishi2008information} 
\footnote{
Note that the original GIC corrects bias term 
$\hat\EE_n[\ell_{\hat\theta}(Z)] - \EE[\ell_{\hat\theta}(Z)] 
    \pto - \frac{1}{n} \mathrm{tr}\left[V^{-1}J\right]$,
which has multiplier $\frac{1}{n}$ instead of $\frac{1}{2n}$ of our bias adjustment term of the lower bound.
This is because the GIC corrects both underestimation bias of training risk
$\EE\left(\hat\EE_n[\ell_{\hat\theta}(Z)]\right) - \EE[\ell_{\theta^*}(Z)]$
and generalization error
$\EE[\ell_{\theta^*}(Z)] - \EE[\ell_{\hat\theta}(Z)]$
while our estimator only corrects the former.
}
of the lower bound as
\begin{equation}
    \hat V_{\textinf, \mathrm{GIC}}^\KCMC(\pi) = \hat V_\textinf^\KCMC(\pi) - \frac{1}{2n}\mathrm{tr}\left[\hat V^{-1}\hat J\right],
\end{equation}
where $\hat J:=\hat\EE_n \left[ \delta \ell_{\hat\theta}(Z) \delta \ell_{\hat\theta}(Z)^T \right]$
and $\hat V:= \hat \EE_n \left[ \nabla_\theta \nabla_\theta^T \ell_{\hat\theta}(Z)\right]$ when $\theta\mapsto\ell_\theta(z)$ is twice differentiable at $\theta=\theta^*$.
When the twice differentiability does not hold for the loss function, we need to derive the analytical form of Hessian $V$ and construct its estimator.
Indeed, the loss function is not pointwise differentiable in the case of box constraints, and the analytical form of its Hessian needs to be derived as in Appendix \ref{app:hessian_expected_loss}.
\end{example}

\begin{example}[Confidence interval with second order bias correction]\label{ex:confidence_interval_second_order_correction}
Using the second order bias correction above, we can obtain a new confidence interval of $V_\textinf^\KCMC(\pi) = -\EE[\ell_{\theta^*}(Z)]$ as
\begin{equation}
    \tilde C_n^{V_\textinf^\KCMC} = \left[\tilde C_n^-, \tilde C_n^+ \right]
\end{equation}
for
\begin{align}
    \tilde C_n^\pm
    &= \EE[\ell_{\theta^*}(Z)]
    - \frac{1}{2n}\mathrm{tr}\left[V^{-1}J\right]
    \pm \sqrt{\VV[\ell_{\theta^*}(Z)] / n} \cdot \Phi^{-1}(1 - \alpha / 2).
\end{align}
We examine the benefit of this second order bias correction in our numerical experiments in Section \ref{subsec:confidence_interval}.
\end{example}

\section{Numerical Experiments}\label{chap:experiments}

In this section, we present the results of our numerical experiments.
Our experiments demonstrate the applicability of the KCMC estimator in various settings such as sensitivity analysis with a continuous (treatment) action space, the new f-sensitivity model, policy learning, construction of confidence interval, and model selection. 

\subsection{Experimental settings}
In all the experiments except Section \ref{sec:lower_upper_bounds}, we use synthetic data with a binary action space sampled from the following data generating process:
\begin{align}
    X &\sim \mathcal{N}(\mu_x, \mathrm{I}_5), \\
    Y(0)|X &\sim \mathcal{N}(\beta_{x, 0}^T X + \beta_{\text{const}, 0}, 1), \\
    Y(1)|X &\sim \mathcal{N}(\beta_{x, 1}^T X + \beta_{\text{const}, 1}, 1), \\
    T|X &\sim \mathrm{Bern}\left(\frac{1}{1 + \exp(- \beta_t^T X)}\right), \\
    Y &= Y(T),
    \label{eq:experimental_data_binary}
\end{align}
where
\begin{align*}
    \mu_x &= (-1, 0.5, -1, 0, -1)^T, \\
    \beta_{x, 0} &= (0, 0.5, -0.5, 0, 0)^T, \\
    \beta_{x, 1} &= (-1.5, 1.5, -2, 1, 0.5)^T, \\
    \beta_{\text{const}, 0} &= 2.5,\\
    \beta_{\text{const}, 1} &= 0.5,\\
    \beta_t &= (0, 0.75, -0.5, 0, -1)^T.
\end{align*}
Though this model is not confounded at all, it has a partially analytical solution of the policy value lower bound for box constraints.
This semi-analytical lower bound can be computed using the Monte Carlo method, in a similar manner to the synthetic data in \citet[Corollary 3.]{dorn2022sharp}.
In the policy evaluation, we use logistic policy $\pi(T=1|X) = 1 / \left(1 + \exp(- \beta_\pi^T X)\right)$, where $\beta_\pi = (1, 0.5, -0.5, 0, 0)^T$, and we consider box constraints $\Gamma^{-1} \leq \frac{p_\obs(t|x) (1 - \pi_\base(t|x, u)}{(1 - p_\obs(t|x)) \pi_\base(t|x, u)} \leq \Gamma$ unless otherwise specified.

Similarly to this, we created a variation of the above synthetic data with continuous action space. In the continuous version, we replaced the fourth line and the fifth line of data generating process for the binary data \eqref{eq:experimental_data_binary} with 
\begin{align}
    T|X &\sim \cN\left(\frac{1}{1 + \exp(- \beta_t^T X)}, 1\right), \\
    Y(t) &:= (1 - t) Y(0) + t Y(1), \\
    Y &= Y(T)
    \label{eq:experimental_data_continuous}
\end{align}
so that reward $Y$ linearly depends on treatment $T$ given $Y(0)$ and $Y(1)$. We can think of this model as the case where partial treatment (i.e., $0 < T < 1$) is possible, and its outcome becomes an interpolation of no treatment $T=0$ and full treatment $T=1$.
In the policy evaluation, we use Gaussian policy $\pi(t|X) = \cN\left(t; \beta_\pi^T X, 0.25 \right)$.
For the continuous synthetic data, we consider box constraints $\Gamma^{-1} \leq \frac{\pi_\base(t|x, u)}{p_\obs(t|x)} \leq \Gamma$.

We also included the same real-world data as \citet{dorn2022sharp}, which is 668 subsamples of data from the 1966-1981 National Longitudinal Survey (NLS) of Older and Young Men. 
These subsamples consist of the 1978 cross-section of Young Men who are craftsmen or laborers and are not enrolled in school. 
For this data, we consider box constraints $\Gamma^{-1} \leq \frac{p_\obs(t|x) (1 - \pi_\base(t|x, u)}{(1 - p_\obs(t|x)) \pi_\base(t|x, u)} \leq \Gamma$.

Conditional probability $p_\obs(t|x)$ needed to construct the estimators was calculated from the true data generating process in the case of synthetic data, and it was estimated from the data using the logistic regression with covariate $X$ in the case of the real-world example.

To solve the convex programming involved in the above estimators, we used MOSEK \citep{aps2019mosek} and ECOS \citep{bib:Domahidi2013ecos} through the API of CVXPY \citep{diamond2016cvxpy}. Lastly, the experimental code necessary to reproduce the following results will be made fully available at \url{https://github.com/kstoneriv3/confounding-robust-inference}.

\subsection{Comparison of KCMC estimator to baselines}\label{sec:lower_upper_bounds}

\begin{figure}[htbp]
    \centering
    \includegraphics[width=0.7\linewidth, trim={40 10 40 10mm}, clip]{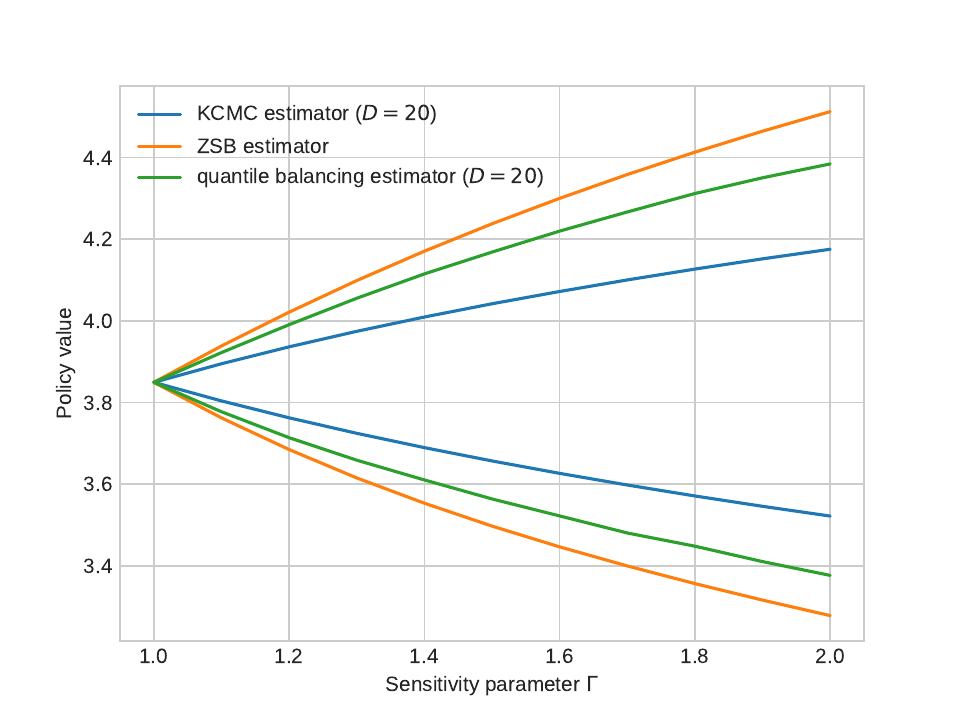}
    \caption{Estimated upper and lower bounds of policy value using different sensitivity parameter $\Gamma$ for the synthetic data of sample size 1000 with binary action space.}
    \label{fig:binary}
\end{figure}

\begin{figure}[htbp]
    \centering
    \begin{subfigure}[b]{0.49\textwidth}
    \centering
        \centering
        \includegraphics[width=\linewidth, trim={20 5 20 10mm}, clip]{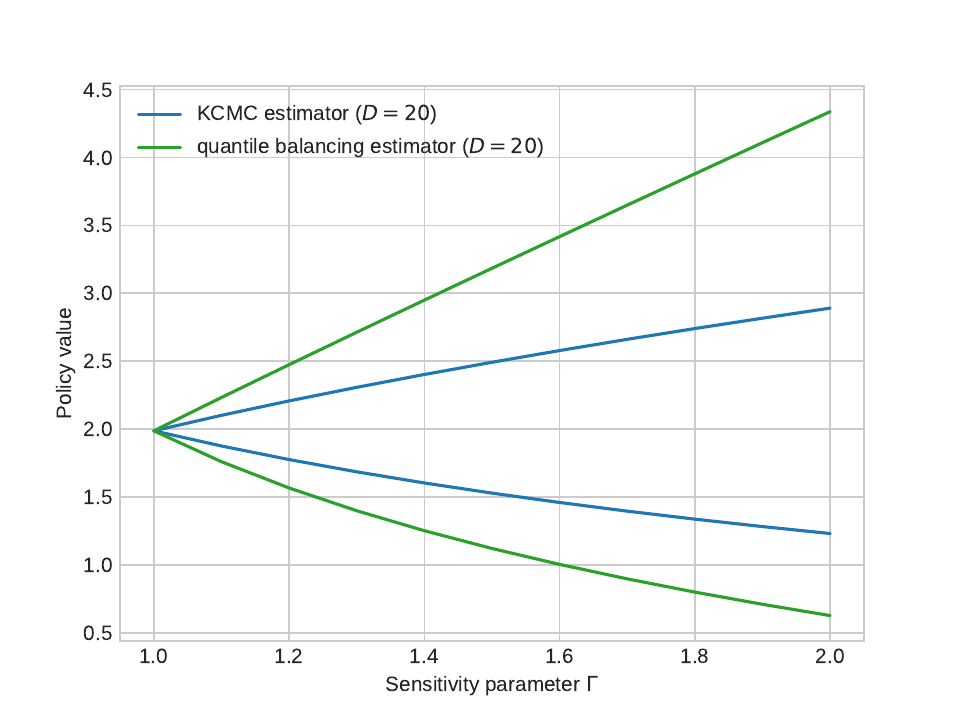}
    \end{subfigure}
    \begin{subfigure}[b]{0.49\textwidth}
        \centering
        \includegraphics[width=\linewidth, trim={20 5 20 10mm}, clip]{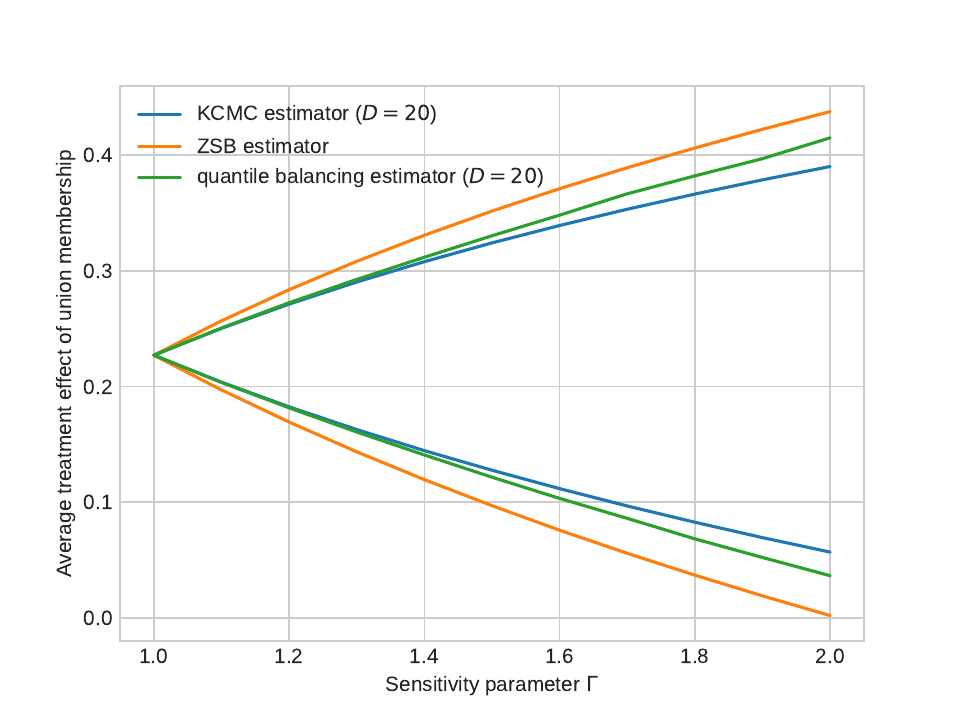}
    \end{subfigure}
    \caption{
        Estimated upper and lower bounds of policy value using different sensitivity parameters.
        Synthetic data of sample size 1000 with continuous action space (left) and the NLS data of sample size 668 with a binary action space (right) are used.}
         \label{fig:continuous_and_nls}
\end{figure}

Figure \ref{fig:binary} shows the upper and lower bounds of the policy value estimate for binary synthetic data for KCMC estimator \eqref{eq:empirical_v_inf_kernel_conditional_moment_constraints}, ZSB estimator \citep{zhao2019sensitivity}, and the quantile balancing (QB) estimator \citep{dorn2022sharp} \footnote{For ZSB estimator, use the empirical version of \eqref{eq:v_inf_zsb}. For both ZSB and the QB estimator, we do not use the linear fractional programming technique \citep{zhao2019sensitivity} to match the estimator values under no confounding for the purpose of comparison.} under different confounding levels $\Gamma$.
Here, fractional programming is not used in the ZSB estimator and the quantile balancing estimator in order to match the policy value with the KCMC estimator under no confounding (i.e., $\Gamma=1.0$).
The feature vectors for the QB estimator and the KCMC estimator were chosen so that they represent the Example \ref{ex:quantile_balancing_as_a_special_case}, where the first stage of the QB estimator solves the modified version of the dual problem for the KCMC estimator.
As discussed in Example \ref{ex:quantile_balancing_as_a_special_case}, the QB estimator is always no tighter than the KCMC estimator under such a choice of the feature vector, and it can clearly be observed in this example.
Indeed, we can see that the tightness of the bounds loosens as sensitivity parameter $\Gamma$ gets larger.
When compared to the sharp estimators (KCMC and QB), the ZSB estimator has looser bounds.

We also present the bounds of the policy values for the synthetic data with continuous action space and real-world data in Figure \ref{fig:continuous_and_nls}.
Please note that in the case of synthetic data with continuous action space, the ZSB estimator cannot be calculated as it requires that the action space be discrete and finite.
In both cases, the tightness of the bounds depends on the sensitivity parameter and the type of estimators in the same way as the case of the binary synthetic data.

\subsection{Confidence interval}\label{subsec:confidence_interval}
\begin{figure}[t]
    \centering
    \begin{subfigure}[b]{0.49\textwidth}
    \centering
        \includegraphics[width=\linewidth, trim={20 5 20 10mm}, clip]{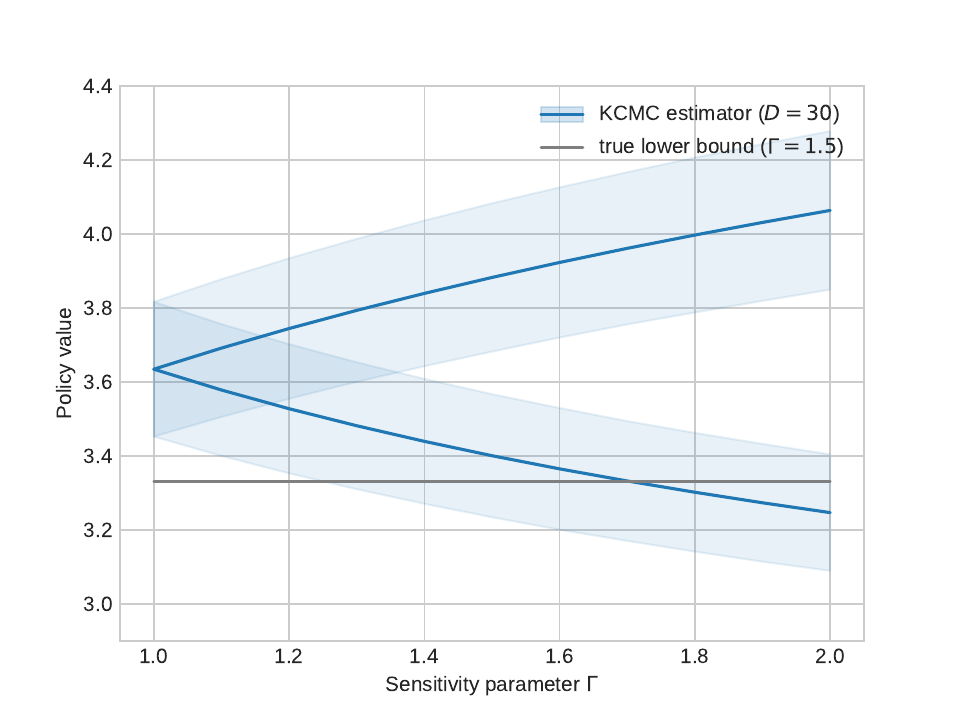}
    \end{subfigure}
    \begin{subfigure}[b]{0.49\textwidth}
         \centering
         \includegraphics[width=\linewidth, trim={20 5 20 10mm}, clip]{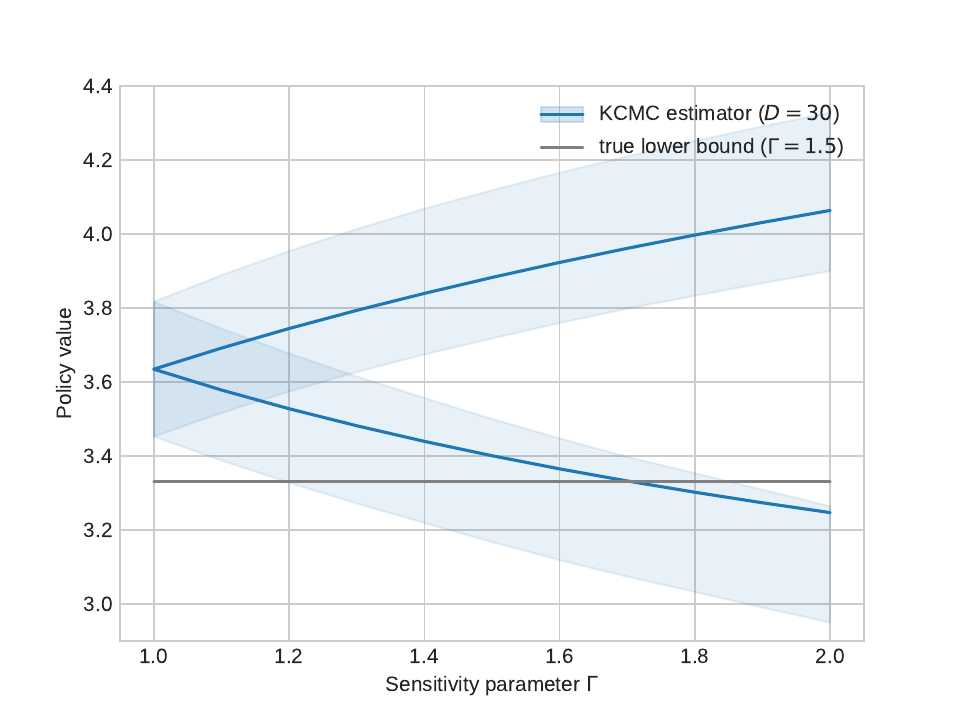}
    \end{subfigure}
    \caption{Estimated upper and lower bounds of policy value with 95\% confidence interval with (right) and without (left) second-order correction. Solid lines and the bands surrounding them indicate the point estimate (both without second order correction) and its confidence interval.}
     \label{fig:confidence_interval}
\end{figure}

Figure \ref{fig:confidence_interval} shows the confidence intervals of upper and lower bounds of policy values discussed in Example \ref{ex:confidence_interval} and \ref{ex:confidence_interval_second_order_correction}.
Since the second order debiasing term corrects the overly optimistic bounds due to the overfitting, the confidence intervals with the second order correction provide more pessimistic bounds.

\begin{figure}[t]
    \centering
    \begin{subfigure}[b]{0.49\textwidth}
    \centering
        \includegraphics[width=\linewidth, trim={20 5 20 10mm}, clip]{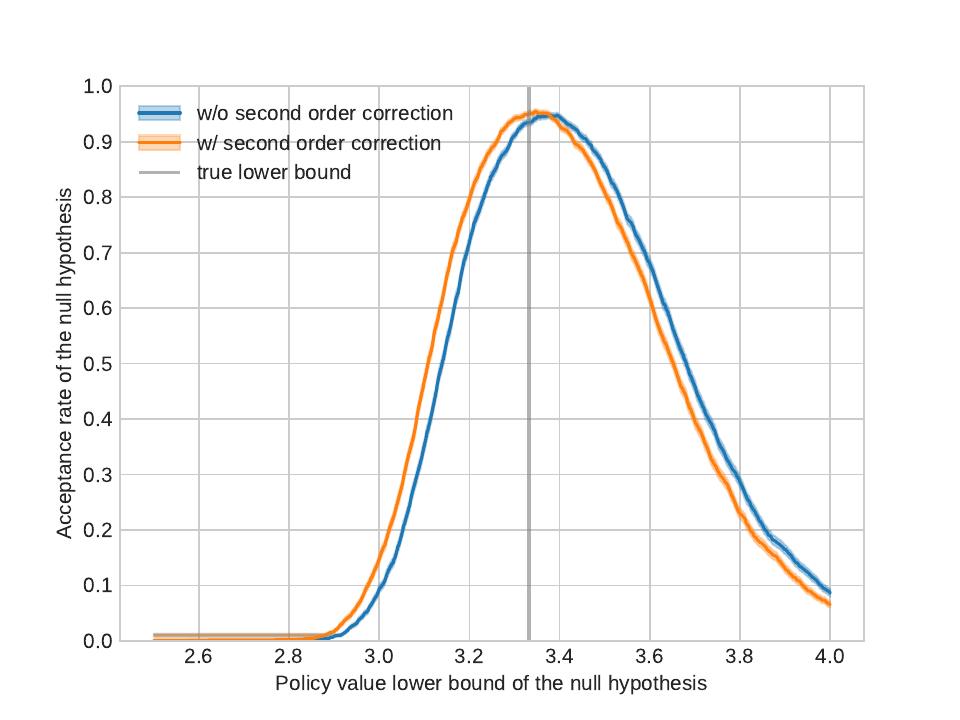}
        \label{fig:continuous}
    \end{subfigure}
    \begin{subfigure}[b]{0.49\textwidth}
         \centering
         \includegraphics[width=\linewidth, trim={20 5 20 10mm}, clip]{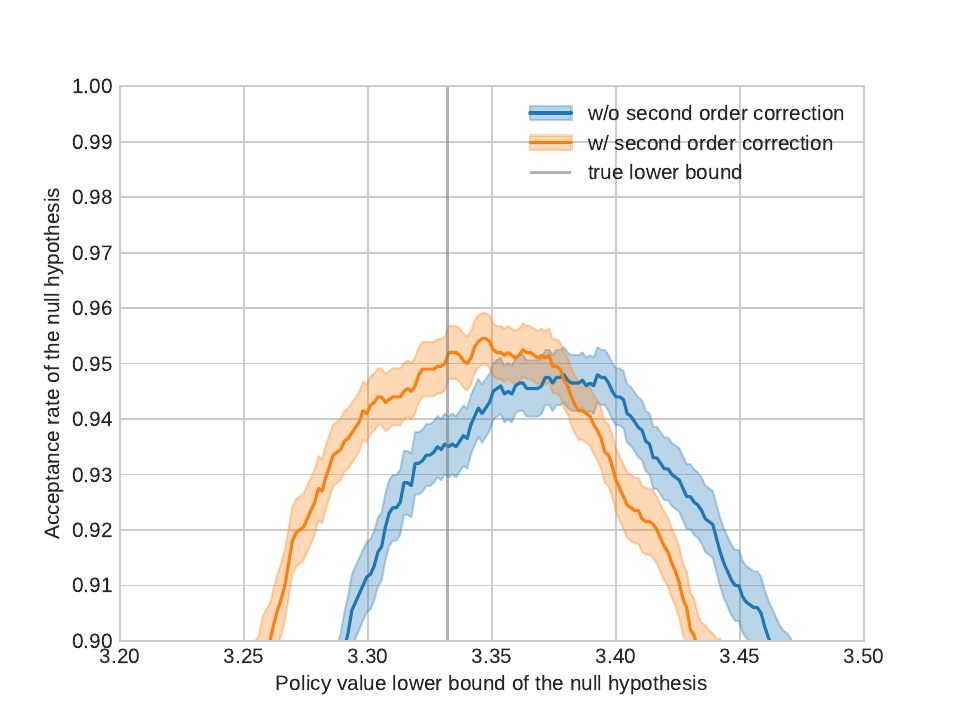}
         \label{fig:nls}
    \end{subfigure}
    \caption{Acceptance rate of the null hypothesis under different null hypothesises with Tan's box constraints ($\Gamma=1.5$) with significance level $\alpha=0.05$. The plot on the left is in the original size and the one on the right is its zoomed version.}
    \label{fig:acceptance_rate}
\end{figure}

In Figure \ref{fig:acceptance_rate}, we show the acceptance rate of the null hypotheses with different values of true lower bound.
The acceptance rate is calculated as the coverage rate of the confidence intervals for the synthetic data of sample size 2000, simulated 2000 times.
As noted earlier, this synthetic data has an analytically tractable true lower bound and it is indicated by the grey vertical line.
As can be expected, acceptance rates of both confidence intervals (with and without second order bias correction) reach their peaks near the true lower bound.
However, the confidence interval without second order bias correction has its peaks shifted towards the optimistic direction due to the overfitting in the dual problem and exhibits some level of overrejection at the true null hypothesis compared to its counterpart with the bias correction.

\subsection{Model Selection}

\begin{figure}[htbp]
    \centering
    \includegraphics[width=0.7\linewidth, trim={10 10 40 10mm}, clip]{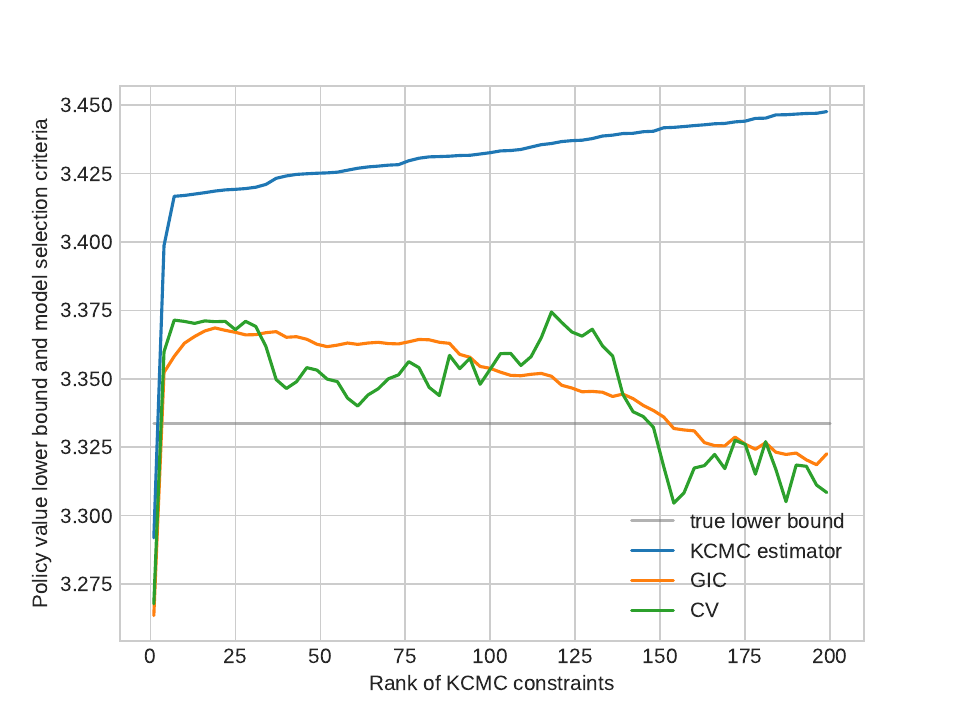}
    \caption{Estimated policy value lower bound for box constraints ($\Gamma=1.5$), its generalized information criterion, and its 10-fold cross validation (CV) for the synthetic data ($n=1000$) with binary action space.}
    \label{fig:model_selection}
\end{figure}

Here, we show an example use of the GIC discussed in Example \ref{ex:gic} for rank selection of the KCMC.
As the KCMC is a finite dimensional approximation of the infinite dimensional constraints, we have the motivation to increase the rank of the KCMC estimator.
On the other hand, the original KCMC lower bound increases monotonously as we increase the rank of KCMC constraints.
We can even recover the inverse probability weighting estimator\footnote{More precisely, the inverse probability weighting estimator that uses the $p_\obs(t|x)$ in place of the true propensity score.} by making the KCMC full-rank, but this nullifies the confounding robustness, so we would like to avoid using the excessively high rank.
This necessitates systematic procedures for selecting the appropriate rank of the KCMC.
Fortunately, our dual problem can be interpreted as a standard empirical risk minimization, where the rank selection can be interpreted as the model selection where the cross-validation and the GIC can be used as the selection criteria.

In Figure \ref{fig:model_selection}, we plotted the KCMC lower bound estimator, its GIC, and its cross validation. 
While the original estimator increases monotonically as the rank of the KCMC increases, the GIC and the cross validation start to decline when the rank of the constraints becomes excessively high.
In this example, both the GIC and the cross validation seem to prefer ranks around 20, though the cross validation tends to be more noisy and also prefers ranks of around 120.

\subsection{f-sensitivity model}

\begin{figure}[htbp]
     \centering
     \begin{subfigure}[b]{0.49\textwidth}
     \centering
        \includegraphics[width=\linewidth, trim={20 5 20 5mm}, clip]{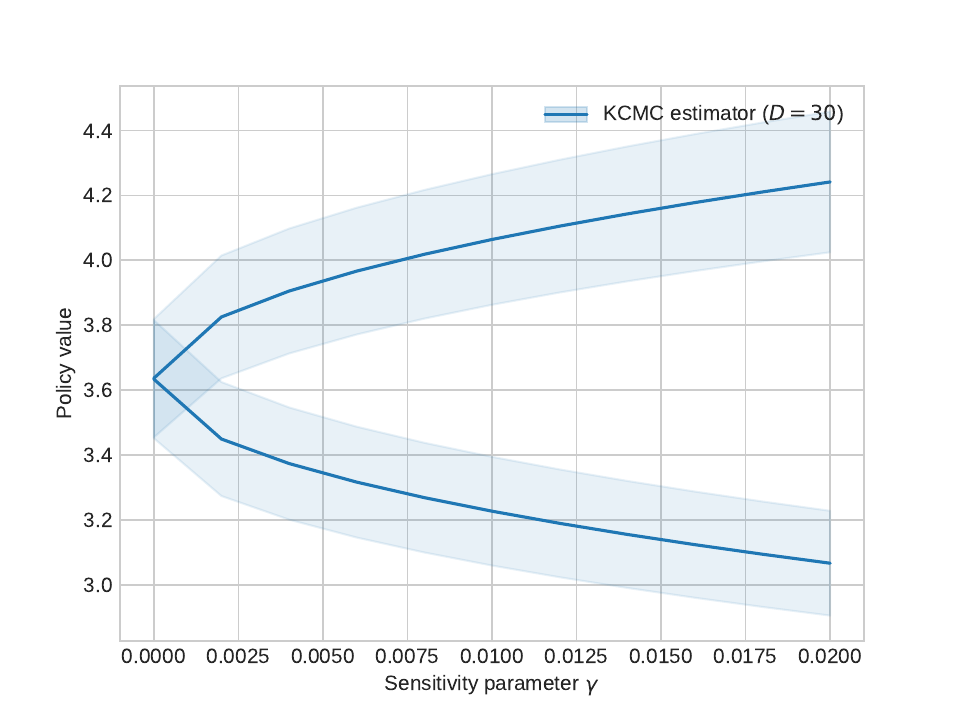}
         \caption{KL sensitivity model}
     \end{subfigure}
     \begin{subfigure}[b]{0.49\textwidth}
     \centering
        \includegraphics[width=\linewidth, trim={20 5 20 5mm}, clip]{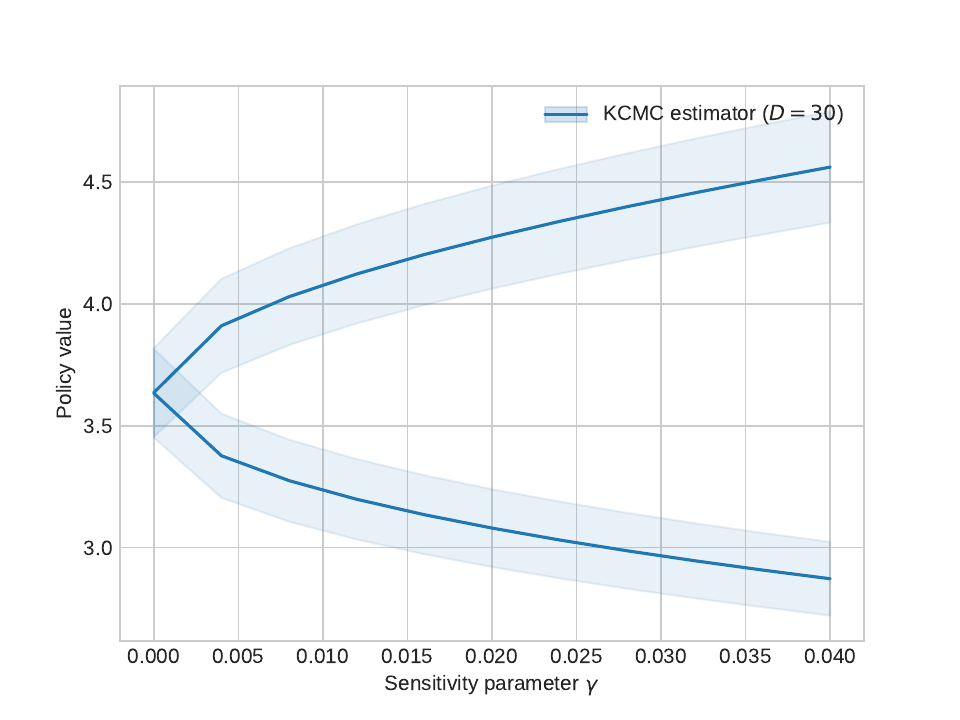}
         \caption{Reverse KL sensitivity model}
     \end{subfigure}
     \begin{subfigure}[b]{0.49\textwidth}
     \centering
        \includegraphics[width=\linewidth, trim={20 5 20 5mm}, clip]{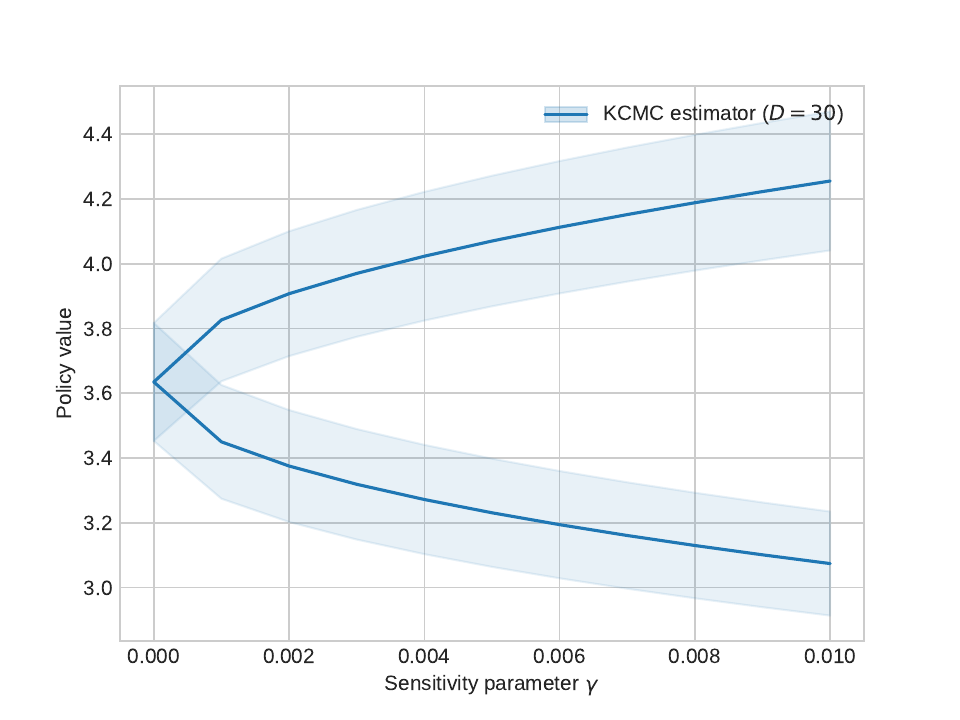}
         \caption{Squared Hellinger sensitivity model}
     \end{subfigure}
     \begin{subfigure}[b]{0.49\textwidth}
     \centering
        \includegraphics[width=\linewidth, trim={20 5 20 5mm}, clip]{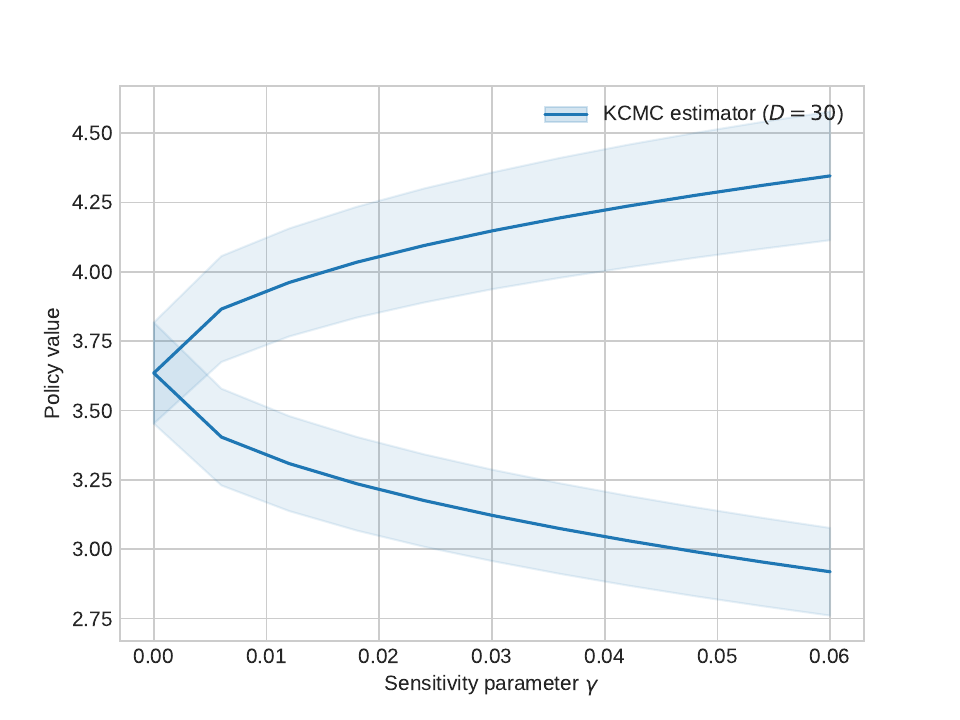}
         \caption{Pearson $\chi^2$ sensitivity model}
     \end{subfigure}
     \begin{subfigure}[b]{0.49\textwidth}
     \centering
        \includegraphics[width=\linewidth, trim={20 5 20 5mm}, clip]{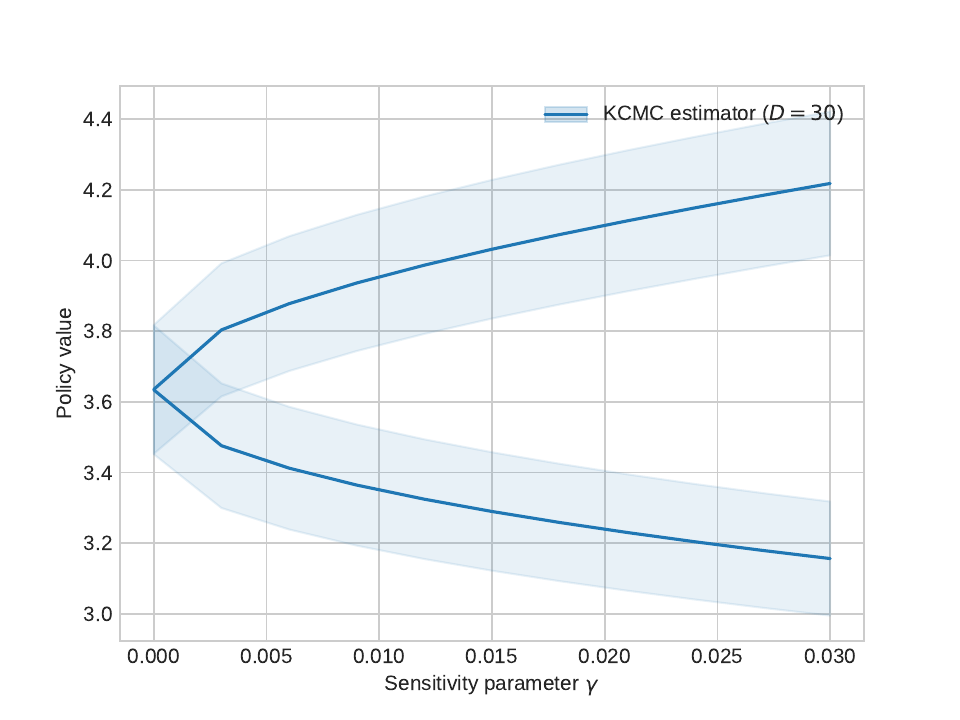}
         \caption{Neyman $\chi^2$ sensitivity model}
     \end{subfigure}
     \begin{subfigure}[b]{0.49\textwidth}
     \centering
        \includegraphics[width=\linewidth, trim={20 5 20 5mm}, clip]{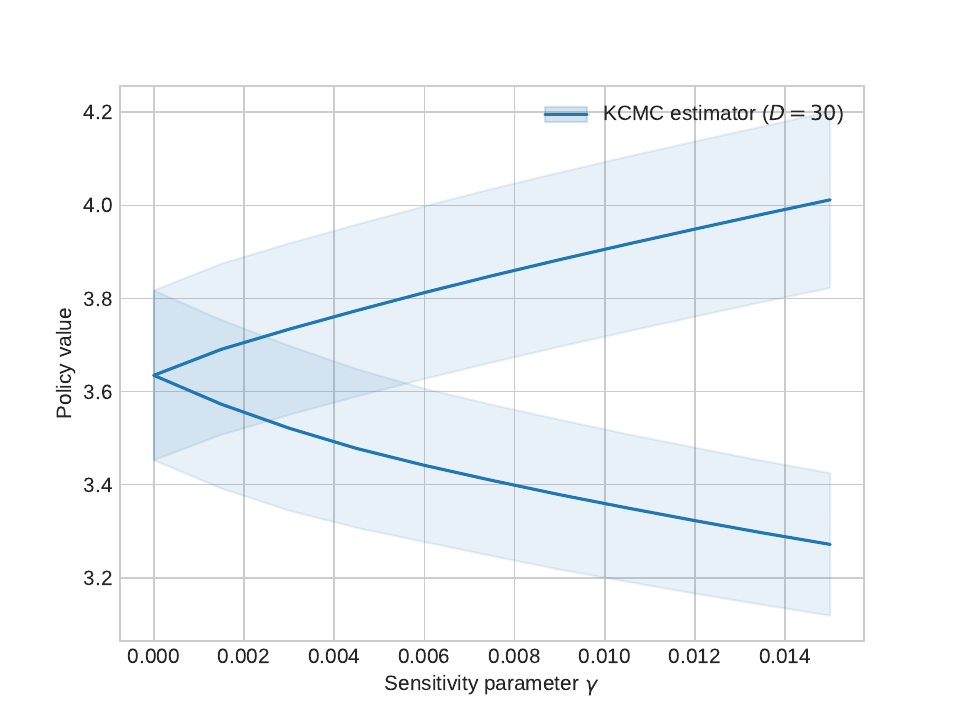}
         \caption{Total variation sensitivity model}
     \end{subfigure}
    \caption{Estimated upper and lower bounds of policy value and their confidence intervals using f-sensitivity models. The synthetic data of sample size 4000 with binary action space is used.}
    \label{fig:f_sensitivity_models}
\end{figure}
In Figure \ref{fig:f_sensitivity_models}, we present the upper/lower bounds and confidence intervals of the policy value given by the KCMC estimators of f-sensitivity models.
Similarly to the example of box constraints, we can see that the f-sensitivity models can produce a continuous control of the level of confounding by the sensitivity parameter.
We can also confirm that the confidence intervals behave quite similarly to the case of box constraints.

\subsection{Policy learning}

Finally, we consider the max-min policy learning with the KCMC estimator.
In the policy learning, we use the gradient ascent on the lower bound estimator similarly to \citet{kallus2018confounding, kallus2021minimax}.\footnote{By Danskin's theorem \citep{danskin1966theory}, we can calculate the gradient for outer maximization as the gradient at the solution of the inner maximization problem.}
As a baseline, we also implemented policy learning with the ZSB estimator using the same dataset and initial policy.
In Figure \ref{fig:differentiable_policy_learning}, we present the learning curves of both estimators.
As can be expected, after the max-min policy learning, the KCMC estimator achieves a higher lower bound than the ZSB estimator, both on the training data and the test data.

\begin{figure}[htbp]
    \centering
    \begin{subfigure}[b]{0.49\textwidth}
    \centering
        \includegraphics[width=\linewidth, trim={20 5 20 10mm}, clip]{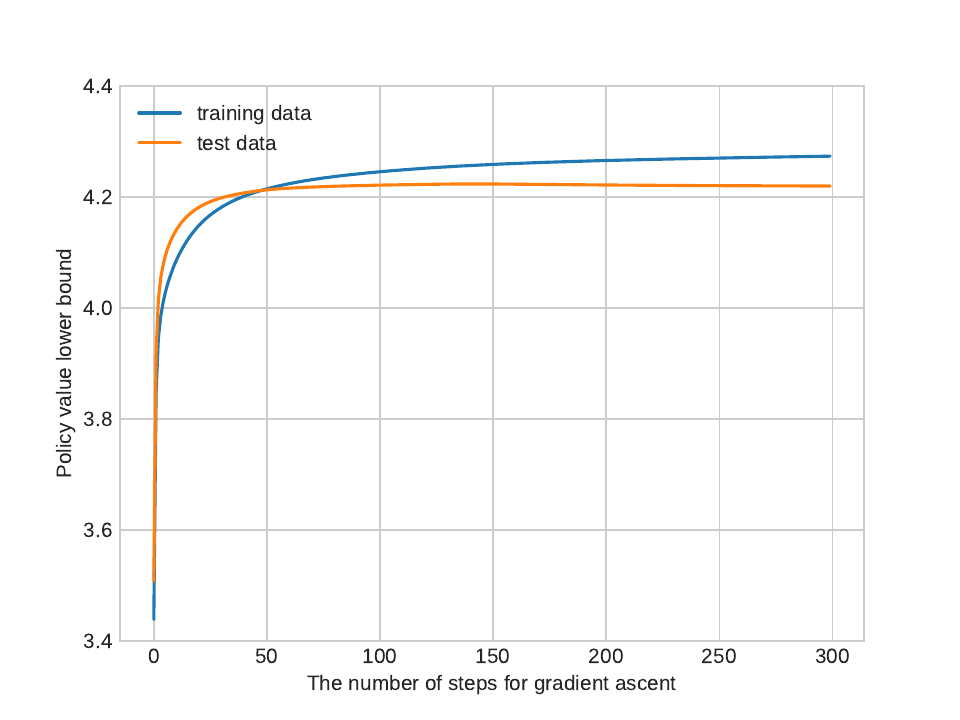}
    \end{subfigure}
    \begin{subfigure}[b]{0.49\textwidth}
         \centering
         \includegraphics[width=\linewidth, trim={20 5 20 10mm}, clip]{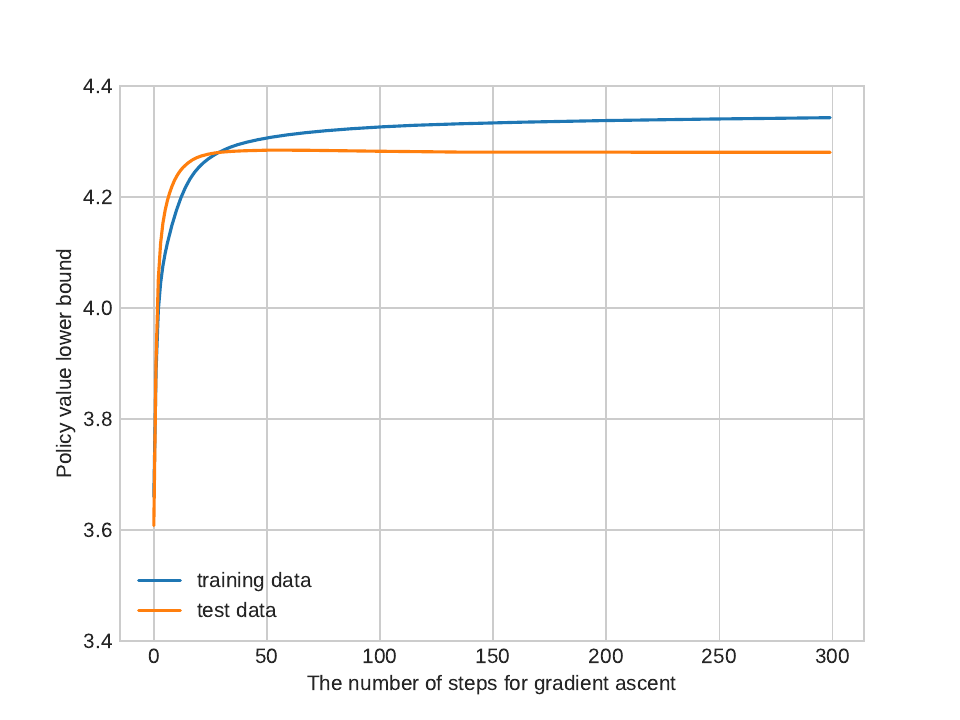}
    \end{subfigure}
    \caption{Learning curves of the max-min policy learning of logistic policy by the ZSB estimator (left) and the KCMC estimator (right). Both training and test data are binary synthetic data of sample size $1000$. The training curve and the test curve represent the values of each estimator for the data.}
     \label{fig:differentiable_policy_learning}
\end{figure}

\section{Conclusion}\label{chap:conclusion}


In this paper, we proposed a novel framework of the sensitivity analysis using kernel conditional moment constraints.
Together with the newly proposed conditional f-constraints, this approach provides a unified formulation of sensitivity analysis that encompasses a broad class of existing models and extends its applicability to previously intractable settings.
From a theoretical perspective, we exploited convexity properties to recast our estimation procedure as an empirical risk minimization problem and established consistency guarantees for both policy evaluation and policy learning using standard tools from M-estimation theory.
Finally, we demonstrated the practical effectiveness of the proposed kernel conditional moment constraints through numerical experiments, including several novel extensions of sensitivity analysis.

\clearpage
\appendix
{\LARGE{\textbf{Appendix}}}

\section{Notations}\label{app:notations}

\begin{table}[h]
    \centering
    \begin{tabular}{ll}
        \hline
        notation & meaning \\
        \hline
        $Y\in\cY$ & reward (outcome) \\
        $T\in\cT$ & action (treatment) \\
        $X\in\cX$ & context \\
        $U\in\cU$ & unobserved confounder \\
        $\pi_\base(t| x, u)$ & unknown (and confounded) base policy from which the data is generated \\
        $\pi(t|x)$ & target policy for policy evaluation (or learning) \\
        $p_\obs(y, t, x)$ & marginal distribution of $(Y, T, X)$ under data generating process \eqref{eq:data_gen_base_confounded} \\
        $\EE, \EE_\obs$ & expectation under data generating process \eqref{eq:data_gen_base_confounded} \\ 
        $\hat\EE_n$ & empirical average over $n$ samples generated from \eqref{eq:data_gen_base_confounded} \\ 
        $\cE$ & uncertainty set of $\pi_\base(t| x, u)$\\
        $f_{t,x}(v)$ & a convex function w.r.t. $v$ for fixed $t, x$ satisfying $f_{t, x}(1) = 1$ \\
        $f^*_{t,x}(u)$ & convex conjugate of $f_{t, x}(v)$ \\
        $w(y, t, x)$ & reparametrization $w(y, t, x)= \EE_{T\sim\pi_\base(\cdot|X, U)}\left[\frac{1}{\pi_\base(T|X, U)}|Y=y, T=t, X=x\right]$ \\
        $\tilde w(y, t, x)$ & reparametrization $\tilde w(y, t, x) = p_\obs(t|x)w(y, t, x)$ \\
        $\cW$ & uncertainty set of $w(y, t, x)$, such as $\cW^\KCMC_{f_{t, x}}$ \\
        $\hat\cW$ & uncertainty set of $w(y, t, x)$ defined using empirical samples, such as $\hat\cW^\KCMC_{f_{t, x}}$\\
        $V(\pi)$ & value of policy $\pi(t|x)$ under confounded data generating process \eqref{eq:data_gen_eval_confounded} \\
        $V_\textinf(\pi)$ & lower bound of the value of policy $\pi(t|x)$ under confounded data generating process \eqref{eq:data_gen_eval_confounded} \\
        $\hat V_\textinf(\pi)$ & estimator of $V_\textinf(\pi)$ using empirical samples \\
        $\psi_d(t, x)$ & $d$-th orthogonal function of kernel conditional moment constraints\\
        $\phi^\KPCA_d(t, x)$ & $d$-th principal component of kernel PCA fitted to empirical data \\
        $\phi^*_d(t, x)$ & $d$-th principal component of kernel PCA fitted to population (i.e. infinite data) \\
        $\lambda^*_d$ & the eigenvalue of $\phi^*_d(t, x)$\\
        $n$ & the number of samples \\
        $D$ & the number of kernel conditional moment constraints \\
        $e(y, t, x)$  & reparametrization $e(y, t, x) := p_\obs(t|x)w(y, t, x) - 1$ \\
        $r(y, t, x)$ & reparametrization $r(y, t, x) = \left(\frac{\pi(t|x)}{p_\obs(t|x)}\right)\cdot y$ \\
        $\ell_\theta(z)$ & loss function of dual problems \eqref{eq:v_inf_kernel_conditional_moment_constraints_dual} and \eqref{eq:empirical_v_inf_kernel_conditional_moment_constraints_dual}. \\
        $\partial_\theta \ell_\theta(z)$ & subgradient of the loss function w.r.t. $\theta$ \\ 
        $Z$ & shorthand notation for observable variables $(Y, T, X)$ \\
        $\theta\in\Theta$ & parameters of dual loss $(\eta_f, \eta_\KCMC)$ or $(\eta_f, \eta_\KCMC, \beta)$ \\
        $\theta^*$ & minimizer of population risk $\EE\ell_\theta(Z)$\\
        $\hat\theta_n$ & minimizer of empirical risk $\hat\EE_n\ell_\theta(Z)$\\
        $\Theta_\varepsilon$ & 
        $\varepsilon$-neighborhood of $(\eta^*(\beta), \beta)$ of the partially optimal parameters in policy learning \\
        \hline
    \end{tabular}
    \caption{List of common notations used throughout the paper.}
    \label{tab:notations}
\end{table}

\section{An Alternative f-Sensitivity Model by Jin et al. (2022)}\label{app:comp_jin_f_sensitivity}
\citet{jin2022sensitivity} proposed a similar uncertainty set that approximates the condition
\[
    D_f\left[p\left(Y(1)|X, T=1\right) || p\left(Y(1)|X, T=0\right)\right] \leq \gamma  \text{ almost surely},
\]
where variable $Y(1)$ is the potential outcome variable for treatment $T=1$ in Rubin's potential outcome framework \citep{rubin2005causal} with binary treatment.
Under the assumption of unconfoundedness, the potential outcome variable $Y(1)$ must satisfy $Y(1) \indep T|X$, and thus it must satisfy 
$D_f\left[p\left(Y(1)|X, T=1\right) || p\left(Y(1)|X, T=0\right)\right]=0$
almost surely with respect to $p_\obs$.
Their sensitivity model can be interpreted as a relaxation of this assumption by allowing the violation of it up to $\gamma$.
To highlight the difference in modeling paradigm, our f-sensitivity model follows the same modeling framework as Tan (2006), which takes into account the difference between observational policy $p_\obs(t|x)$ and underlying confounded policy $\pi_\base(t|x, u)$.
On the other hand, the model by \citet{jin2022sensitivity} considers the distributional shift between observation $Y(1)|X = x, T = 1$ and counterfactual $Y(1)|X = x, T = 0$, and therefore, their modeling approach is different from Tan (2006) and its extension.

\section{Omitted Proofs}\label{app:omitted_proofs}
\subsection{Details of Example \ref{ex:box_constraint_analytical_solution}}\label{app:proof_box_constraint_analytical_solution}
Here, we consider the box constraints corresponding to \eqref{eq:f_box_constraints} of the conditional f-constraint.
For function $f_{t, x}(\tilde w) = I_{[a(t, x), b(t, x)]}(\tilde w)$, we can derive its conjugate and its subgradient as
\begin{equation}
    f_{t, x}^*(v) = \begin{cases}
        a(t, x)v & \text{ if \ \ }v < 0, \\
        0 & \text{ if \ \ }v = 0, \\
        b(t, x)v & \text{ if \ \ }v > 0,
    \end{cases}
\end{equation}
and 
\begin{equation}
    \partial f_{t, x}^*(v) = \begin{cases}
        a(t, x) & \text{ if \ \ }v < 0, \\
        [a(t, x), b(t, x)] & \text{ if \ \ }v = 0, \\
        b(t, x) & \text{ if \ \ }v > 0.
    \end{cases}
\end{equation}
Substituting the above expression of $\partial f_{t, x}^*$ in the first order condition of \eqref{eq:v_inf_conditional_moment_constraints_dual_zero_eta_f}, we can derive more explicit expression
\begin{align}
1
&=\EE\left[\partial f_{T, X}^*\left(\eta_\CMC - r\right)|T=t, X=x\right] \\
&= \PP(r < \eta_\CMC(t, x)) \cdot b(t, x)
        + \PP(r = \eta_\CMC(t, x)) \cdot [a(t, x), b(t, x)]
        + \PP(r > \eta_\CMC(t, x)) \cdot a(t, x).
\end{align}
Here, we used the box constraints' property $a(t, x) \leq 1 \leq b(t, x)$, which follows from the requirement that $f_{t, x}$ must satisfy $f_{t, x}(1)=0$ for any $t\in\cT$ and $x\in\cX$.
Assuming that the conditional distribution of $r(Y, T, X)$ given $T=t$ and $X=x$ yields continuous distribution for any $t\in\cT$ and $x\in\cX$, the second subgradient condition becomes
\begin{equation}
1 = \PP(r \leq \eta_\CMC(t, x)) \cdot b(t, x)
        + \PP(r > \eta_\CMC(t, x)) \cdot a(t, x).
\end{equation}
This implies that $\PP(r \leq \eta_\CMC(t, x)) = \frac{1 - a(t, x)}{b(t, x) - a(t, x)}=: \tau(t, x)$, and therefore,
\begin{equation}
    \eta^*_\CMC(t, x) = \left( \frac{\pi(t|x)}{p_\obs(t|x)} \right)Q(t, x),
\end{equation}
where $Q(t, x)$ was defined as the $\tau(t, x)$-quantile of the conditional distribution of $Y$ given $T=t$ and $X=x$.
From this dual solution, the primal solution can also be recovered using \eqref{eq:cmc_solution_characterization} as
\begin{equation}
    w^*_\CMC(y, t, x) =
    \begin{cases}
        b(t, x) & \text{ if \ \ }y \leq Q(t, x), \\
        a(t, x) & \text{ otherwise}.
    \end{cases}
\end{equation}

\subsection{Proof of Theorem \ref{th:specification_error_liptchitz_bound}}\label{app:proof_specification_error_lipschitz_bound}
\begin{proof}
As $\Pi_{\bspsi} \eta_\CMC^*$ is on the subspace spanned by $\{\psi_1, \ldots, \psi_D\}$, we can take $\eta^*_\KCMC\in\RR^D$ such that ${\eta^*_\KCMC}^T\bspsi = \Pi_{\bspsi} \eta_\CMC^*$.
Then, due to the fundamental theorem of calculus, we have 
\begin{align}
    J[{\eta^*_\KCMC}^T\bspsi]
    &= J[\eta_\CMC] + \int_{\eta^*_\CMC}^{{\eta^*_\KCMC}^T\bspsi} \sup_{\delta J \in \partial J[h]} \langle \delta J, \rd h \rangle \\
    &= J[\eta_\CMC] + \int_0^1 \sup_{\delta J \in \partial J[\eta_\CMC + s \Delta \eta]} \langle \delta J, \Delta\eta \rangle \rd s \\
    &\leq J[\eta_\CMC] + \int_0^1 \sup_{\delta J \in \partial J[\eta_\CMC + s \Delta \eta]} \| \delta J \| \cdot \|\Delta\eta \|\rd s \\
    &\leq J[\eta_\CMC] + L\|\Delta \eta\|
\end{align}
for $\Delta \eta:= {\eta_\KCMC^*}^T\bspsi - \eta_\CMC = \Pi_{\bspsi} \eta_\CMC^* - \eta_\CMC^*$.
As  $V_\textinf^\CMC = - J[\eta^*_\CMC]$
and $V_\textinf^\KCMC
    \geq \sup_{\eta_\KCMC\in\RR^D}\left\{-J[{\eta_\KCMC}^T\bspsi]\right\}
    \geq - J[{\eta^*_\KCMC}^T\bspsi]$,
we get
\begin{equation}
    V_\textinf^\KCMC \leq V_\textinf^\CMC \leq V_\textinf^\KCMC + L \|\Pi_{\bspsi} \eta_\CMC^* - \eta_\CMC^*\|.
\end{equation}

To prove the latter statement on no specification error, we take multiplier $\eta_\KCMC^*$ that satisfies $\eta_\CMC^* = {\eta_\KCMC^*}^T \bspsi$.
Then, we can see that dual problem \eqref{eq:v_inf_kernel_conditional_moment_constraints_dual} for $V_\textinf^\KCMC$ can be considered as the restricted version of dual problem \eqref{eq:v_inf_conditional_moment_constraints_dual} for $V_\textinf^\CMC$
where $\eta_\CMC$ is constrained to the subspace spanned by $\{\psi_1, \ldots, \psi_D\}$.
Therefore, as restricted solution $\left(\eta^*_\KCMC, \ \eta^*_f\right)$ achieves the same value as the solution of the non-restricted problem, it is clearly a solution of restricted problem \eqref{eq:v_inf_kernel_conditional_moment_constraints_dual}.
Finally, owing to the strong duality, we can calculate the values of $V_\textinf^\CMC$ and $V_\textinf^\KCMC$ by the values of the dual problems, which implies $V_\textinf^\CMC = V_\textinf^\KCMC$.
\end{proof}

\subsection{Details of Example \ref{ex:lipschitz_box}}\label{app:proof_lipschitz_box}
As discussed previously, in the case of box constraints, $\partial f^*$ is bounded so that $\partial f_{t, x}^*\subseteq [a, b]$.
Here, we omitted the subscripts and the arguments as $a = a_{\tilde w}(t, x)$ and $b = b_{\tilde w}(t, x)$ for notational simplicity.
As $a$ and $b$ satisfy $a \leq 1 \leq b$, due to requirement $f_{t, x}(1)=0$, we get the following explicit formula for the Lipschitz constant:
\begin{align*}
L_{a, b} 
&= \sup\left\{
\| \delta J \|:\ \delta J\in \bigcup_{h:\ \| h - \eta^*_\CMC \| \leq \varepsilon_\text{spec}}\partial J[h]
\right\} \\
&= \sup\left\{
\| \delta J \|:\ \delta J\in \bigcup_{h:\ \| h - \eta^*_\CMC \| \leq \varepsilon_\text{spec}} \EE\left[\partial f_{T, X}^*\left(\frac{h - r}{\eta_f}\right)|T=t, X=x\right] - 1
\right\} \\
&\leq \sup\left\{
\| \delta J \|:\ \delta J\in L^2(\cT\times\cX, p_\obs), a(t, x) - 1\leq \delta J(t, x) \leq b(t, x) - 1
\right\} \\
&= \left(\EE\left|\max\left\{1 - a(T, X), b(T, X) - 1\right\}\right|^2\right)^{1/2}.
\end{align*}
As we know $0 \leq a \leq b$, we can see that $b \in L^2(\cT\times\cX, p_\obs)$ is the sufficient condition to have finite Lipschitz constant $L_{a, b}$.

\subsection{Details of Example \ref{ex:lipschitz_bounded_f}}\label{app:proof_lipschitz_bounded_f}
To show that we can still apply the same argument as Example \ref{ex:lipschitz_box}, we can calculate the conjugate and the subgradient of the conjugate as
\[
  f^*_{t, x}(v) 
  = \sup_{u\in[a, b]}\{vu - f^0_{t, x}(u)\}
  = \begin{cases}
      av - f^0_{t, x}(a) & \text{ if \ } v < \partial f^0_{t, x}(a), \\
      {f^0_{t,x}}^*(v) & \text{ if \ } v \in \partial f^0_{t, x}([a, b]), \\
      bv - f^0_{t, x}(b) & \text{ if \ } v > \partial f^0_{t, x}(b),
  \end{cases}
\]
and
\[
  \partial f^*_{t, x}(v) 
  = \begin{cases}
      a & \text{ if \ } v < \partial f^0_{t, x}(a), \\
      \partial {f^0_{t,x}}^*(v) & \text{ if \ } v \in \partial f^0_{t, x}([a, b]), \\
      b & \text{ if \ } v > \partial f^0_{t, x}(b).
  \end{cases}
\]
for $\partial f^0_{t, x}([a, b]):=\bigcup_{u\in[a, b]} \partial f^0_{t, x}(u)$.
As we know that $f^*$ is convex, $\partial f^*$ is non-decreasing, and therefore, we have $\partial f^* \subseteq [a, b]$.
As the subgradient of $f^*_{t, x}$ is similarly bounded, we can obtain the same Lipschitz constant as Example \ref{ex:lipschitz_box}.

\subsection{Formal Statement and Proof of Lemma \ref{lemma:kpca_bias}}\label{app:proof_kpca_bias}
The informal statement in Lemma \ref{lemma:kpca_bias} can be formally stated as follows:

\begin{lemma}
Let us consider solution $\eta^*_\CMC \in L^2(\cT\times\cX, p_\obs):\cT\times\cX\to\RR$ for some fixed policy.
Define the orthogonal function class as the principal components obtained by applying the kernel PCA to the data so that $\bspsi := \bsphi^\KPCA = \left(\phi^\KPCA_1, \ldots, \phi^\KPCA_D\right)$.
Suppose that kernel $k(\cdot, \cdot)$ used by the kernel PCA is universal in $L^2(\cT\times\cX, p_\obs)$ and has spectral decomposition
$k = \sum_{d=1}^D \lambda^*_d \phi^*_d \otimes \bar\phi^*_d$,
for $\lambda_1^* \geq \lambda_2^* \geq \ldots \geq 0$.
Let us introduce feature map $\Phi:(t, x)\mapsto\sum_{d=1}^\infty \lambda^*_d \phi^*_d(t, x)\phi^*_d$ so that 
$k\left((t, x), (t', x')\right) = \langle \Phi(t, x), \Phi(t', x')\rangle_\cH$.
Now, assume that there exist constants $M_1$ and $M_2$ such that 
$\|\Phi(T, X)\|_\cH^2 < M_1$ and $\|\Phi(T, X) \otimes \bar\Phi(T, X)\| < M_2$ almost surely,
where the norm of product $\Phi\otimes\bar\Phi$ is defined as the Hilbert-Schmidt norm.
Additionally, define $c_d := \langle \eta^*_\CMC, \phi^*_d\rangle$ so that 
$\Pi_{\bsphi^*} \eta^*_\CMC = \sum_{d=1}^D c_d \phi_d$ and further assume that 
$\lim_{D\to\infty}\left\{D\left(\max_{1 \leq d \leq D}\frac{c_d^2}{\lambda_d^*}\right)\left(\sum_{d=D+1}^\infty \lambda_d^*\right)\right\} = 0$.
Then, we have 
\[
    \lim_{D\to\infty}\plim_{n\to\infty} \|\Pi_{\bspsi} \eta^*_\CMC - \eta^*_\CMC\| = 0,
\]
where $\plim$ denotes convergence in probability.
\end{lemma}

\begin{proof}
Due to the universal property of the kernel, we can approximate $\eta^*_\KCMC$ with the element of $\cH$ arbitrarily well.
Therefore, for any $\varepsilon>0$, we can take $D$ such that 
$\|\Pi_{\bsphi^*} \eta^*_\CMC - \eta^*_\CMC \|  \leq \frac{\varepsilon}{3}$
for $\bsphi^*:=\left(\phi_1^*, \ldots, \phi^*_D\right)^T$
and
$D\left(\max_{1 \leq d \leq D}\frac{c_d^2}{\lambda^*_d}\right)\left(\sum_{d=D+1}^\infty \lambda_d^*\right) \leq \frac{\varepsilon}{3}$.
For such a choice of $D$, we can prove the statement of the theorem by showing 
$\plim_{n\to\infty} \|\Pi_{\bsphi^\KPCA}\Pi_{\bsphi^*} \eta^*_\CMC - \Pi_{\bsphi^*} \eta^*_\CMC \|  \leq \frac{\varepsilon}{3}$.
This is because 
\begin{align}
    &\|\Pi_{\bspsi} \eta^*_\CMC - \eta^*_\CMC\| \\
    &= \|\Pi_{\bsphi^\KPCA} \eta^*_\CMC - \eta^*_\CMC\| \\
    &\leq | \Pi_{\bsphi^\KCMC} (\eta^*_\CMC - \Pi_{\bsphi^*} \eta^*_\CMC)\| +  \|\Pi_{\bsphi^\KPCA} \Pi_{\bsphi^*} \eta^*_\CMC - \eta^*_\CMC\| \\
    &\leq \| \Pi_{\bsphi^*} \eta^*_\CMC - \eta^*_\CMC\| + \|\Pi_{\bsphi^\KPCA} \Pi_{\bsphi^*} \eta^*_\CMC - \eta^*_\CMC\|  \\
    &\leq \| \Pi_{\bsphi^*} \eta^*_\CMC - \eta^*_\CMC\| 
    +  \|\Pi_{\bsphi^\KPCA} \Pi_{\bsphi^*} \eta^*_\CMC - \Pi_{\bsphi^*} \eta^*_\CMC\| 
    + \| \Pi_{\bsphi^*} \eta^*_\CMC - \eta^*_\CMC\| \\
    &= \|\Pi_{\bsphi^\KPCA} \Pi_{\bsphi^*} \eta^*_\CMC - \Pi_{\bsphi^*} \eta^*_\CMC\| 
    + 2 \| \Pi_{\bsphi^*} \eta^*_\CMC - \eta^*_\CMC\|,
\end{align}
where we used the non-expansive property of the projection operator in the second line.

Now, as we know $\Pi_{\bsphi^*} \eta^*_\CMC = \sum_{d=1}^D c_d \phi^*_d$, we have
\begin{align}
\|\Pi_{\bsphi^\KPCA}\Pi_{\bsphi^*} \eta^*_\CMC -  \Pi_{\bsphi^*}\eta^*_\CMC\|^2
& = \left\|\Pi_{\bsphi^\KPCA} \sum_{d=1}^D c_d \phi^*_d - \sum_{d=1}^D c_d \phi^*_d \right\|^2 \\
& = \left\|\sum_{d=1}^D \frac{c_d}{\sqrt{\lambda^*_d}} \cdot \sqrt{\lambda^*_d} (\Pi_{\bsphi^\KPCA} - 1) \phi^*_d \right\|^2 \\
& \leq \max_{1\leq d\leq D} \left(\frac{c_d^2}{\lambda^*_d}\right)
    \cdot \left\| \sum_{d=1}^D \sqrt{\lambda^*_d} (\Pi_{\bsphi^\KPCA} - 1) \phi^*_d \right\|^2 \\
& \leq \max_{1\leq d\leq D} \left(\frac{c_d^2}{\lambda^*_d}\right)
    \cdot D \sum_{d=1}^D \lambda^*_d \left\| (\Pi_{\bsphi^\KPCA} - 1) \phi^*_d \right\|^2 \\
& \leq \max_{1\leq d\leq D} \left(\frac{c_d^2}{\lambda^*_d}\right)
    \cdot D \sum_{d=1}^\infty \lambda^*_d \left\| (\Pi_{\bsphi^\KPCA} - 1) \phi^*_d \right\|^2 \\
& = D\cdot \max_{1\leq d\leq D} \left(\frac{c_d^2}{\lambda^*_d}\right)
    \cdot \sum_{d=1}^\infty \lambda^*_d \left( 1 - \|\Pi_{\bsphi^\KPCA}\phi^*_d \|^2 \right).
\end{align}
To show that the summation in the last line converges to zero as $n\to \infty$, we use the convergence result of the subspace learned by the kernel PCA in \citet[Theorem 3.1.]{blanchard2007statistical}.
Their result states that under conditions $\|\Phi(T, X)\|_\cH^2 < M_1$ and $\|\Phi(T, X) \otimes \bar\Phi(T, X)\| < M_2$ almost surely, we have
\[
    0 \leq 
    \EE\left\| \Pi_{\bsphi^*} \Phi(T, X) \right\|^2 - 
    \EE\left\| \Pi_{\bsphi^\KPCA} \Phi(T, X) \right\|^2
    \pto 0
\]
as $n\to\infty$, where $\pto$ indicate the convergence in probability.

Indeed, we can re-write this condition as
\begin{align}
    \EE&\left\| \Pi_{\bsphi^*} \Phi(T, X) \right\|^2 - 
    \EE\left\| \Pi_{\bsphi^\KPCA} \Phi(T, X) \right\|^2
    \\
    &= 
    \EE\left\| \Pi_{\bsphi^*} \sum_{d=1}^\infty \sqrt{\lambda^*_d}\phi_d^*(T, X) \phi_d^* \right\|^2 - 
    \EE\left\| \Pi_{\bsphi^\KPCA} \sum_{d=1}^\infty \sqrt{\lambda^*_d}\phi_d^*(T, X) \phi_d^* \right\|^2 \\
    &= 
    \EE\left\| \sum_{d=1}^\infty \sqrt{\lambda^*_d} \phi_d^*(T, X) \Pi_{\bsphi^*} \phi_d^* \right\|^2 - 
    \EE\left\| \sum_{d=1}^\infty \sqrt{\lambda^*_d}\phi_d^*(T, X) \Pi_{\bsphi^\KPCA} \phi_d^* \right\|^2 \\
    &= 
    \EE\left\| \sum_{d=1}^D \sqrt{\lambda^*_d} \phi_d^*(T, X) \phi_d^* \right\|^2 - 
    \EE\left\| \sum_{d=1}^\infty \sqrt{\lambda^*_d}\phi_d^*(T, X) \Pi_{\bsphi^\KPCA} \phi_d^* \right\|^2 \\
    &= 
    \sum_{d=1}^D \lambda^*_d \EE|\phi_d^*(T, X)|^2 \|\phi_d^* \|^2 
    - \sum_{d=1}^\infty  \sum_{k=1}^\infty \sqrt{\lambda^*_d} \sqrt{\lambda^*_k}
    \left\langle
        \Pi_{\bsphi^\KPCA} \phi_d^*,
         \Pi_{\bsphi^\KPCA} \phi_k^* 
    \right\rangle_\cH
    \EE[\phi_d^*(T, X), \phi_d^*(T, X)] \\
    &= 
    \sum_{d=1}^D \lambda^*_d \|\phi_d^* \|^2 \cdot \|\phi_d^* \|^2 
    - \sum_{d=1}^\infty  \sum_{k=1}^\infty \sqrt{\lambda^*_d} \sqrt{\lambda^*_k}
    \left\langle \Pi_{\bsphi^\KPCA} \phi_d^*, \Pi_{\bsphi^\KPCA} \phi_k^*  \right\rangle_\cH
    \left\langle \phi_d^*, \phi_k^* \right\rangle_\cH \\
    &= 
    \sum_{d=1}^D \lambda^*_d
    - \sum_{d=1}^\infty  \sum_{k=1}^\infty \sqrt{\lambda^*_d} \sqrt{\lambda^*_k}
    \left\langle \Pi_{\bsphi^\KPCA} \phi_d^*, \Pi_{\bsphi^\KPCA} \phi_k^*  \right\rangle_\cH
    \bbmone_{d=k} \\
    &= 
    \sum_{d=1}^D \lambda^*_d
    -  \sum_{d=1}^\infty \lambda^*_d
    \left\| \Pi_{\bsphi^\KPCA} \phi_d^* \right\|^2 \\
    &=
    \sum_{d=1}^\infty \lambda^*_d - \sum_{d=D+1}^\infty \lambda^*_d 
    -  \sum_{d=1}^\infty \lambda^*_d
    \left\| \Pi_{\bsphi^\KPCA} \phi_d^* \right\|^2 \\
    &=
    \sum_{d=1}^\infty \lambda^*_d \left(1  -  \| \Pi_{\bsphi^\KPCA} \phi_d^* \|^2 \right)
    - \sum_{d=D+1}^\infty \lambda^*_d \\
    &\pto 0.
\end{align}

Therefore, we have
\begin{align}
\plim_{n\to\infty}\|&\Pi_{\bsphi^\KPCA}\Pi_{\bsphi^*} \eta^*_\CMC -  \Pi_{\bsphi^*}\eta^*_\CMC\|^2 \\
& \leq \plim_{n\to\infty} D \cdot \max_{1\leq d\leq D} \left(\frac{c_d^2}{\lambda^*_d}\right) \cdot
    \left(
        \EE\left\| \Pi_{\bsphi^*} \Phi(T, X) \right\|^2 - 
        \EE\left\| \Pi_{\bsphi^\KPCA} \Phi(T, X) \right\|^2
        + \sum_{d=D+1}^\infty \lambda^*_d
    \right)\\
& \leq D \cdot \max_{1\leq d\leq D} \left( \frac{c_d^2}{\lambda^*_d}\right) \cdot
    \plim_{n\to\infty} \left(
        \EE\left\| \Pi_{\bsphi^*} \Phi(T, X) \right\|^2 - 
        \EE\left\| \Pi_{\bsphi^\KPCA} \Phi(T, X) \right\|^2
    \right)
    + \frac{\varepsilon}{3} \\
&= \frac{\varepsilon}{3},
\end{align}
which concludes the proof.
\end{proof}

This lemma indicates that convergence of the specification bias to zero can be guaranteed by employing a universal kernel \citep{micchelli2006universal}. However, the current proof lacks the convergence rate, which is an open question for future work.

In the above proof, the most obscure assumption for readers would be the condition that
$\lim_{D\to\infty}\left\{D\left(\max_{1 \leq d \leq D}\frac{c_d^2}{\lambda_d^*}\right)\left(\sum_{d=D+1}^\infty \lambda_d^*\right)\right\} = 0$.
This assumption can be approximately decomposed into two slightly more interpretable assumptions.
First, it requires that $\max_{1 \leq d \leq D}\frac{c_d^2}{\lambda_d^*}$ is bounded so that basis function $\phi^*_d$ for smaller spectrum $\lambda^*_d$ constitutes proportionally smaller part of function $\eta^*_\CMC$.
For example, if $\eta_\CMC^*$ happens to coincide with a function generated from the Gaussian process with kernel $k(\cdot, \cdot)$, it will satisfy this property with high probability.
Second, it requires that $\sum_{d=D+1}^\infty \lambda_d^* = o_p(\frac{1}{D})$ so that the decay of spectrum $\lambda^*_d$ is fast.
In fact, this assumption is similar to another assumption $\|\Phi(T, X)\|^2_\cH \leq M_1$ we made, because this assumption leads to $M_1 \geq \EE\|\Phi(T, X)\|_\cH^2 = \EE\left\|\sum_{d=1}^\infty \sqrt{\lambda_d^*}\phi_d(T, X)\phi_d\right\|^2 = \sum_{d=1}^\infty \lambda_d^* \EE|\phi_d(T, X)|^2 \cdot \|\phi_d\|^2 = \sum_{d=1}^\infty \lambda_d^*$, which implies convergence of its tail part $\sum_{d=D+1}^\infty \lambda_d^*\to 0$.

\subsection{Details of Example \ref{ex:quantile_balancing_as_a_special_case}}\label{app:proof_quantile_balancing_as_a_special_case}
\begin{proof}
We know the analytical form of the dual solution, $\eta^*_\CMC(t, x) = \left(\frac{\pi(t|x)}{p_\obs(t| x)}\right)Q(t, x)$ as discussed earlier.
Therefore, we can take $D=1$ and set $\psi_1(t, y)= \left(\frac{\pi(t|x)}{p_\obs(t| x)}\right)Q(t, x)$ to meet condition \eqref{eq:eta_cmc_in_kernel_subspace} in Theorem \ref{th:specification_error_liptchitz_bound} to obtain the kernel conditional moment constraint with no specification error.
Therefore, we can solve \eqref{eq:v_inf_conditional_moment_constraints} in the case of the box constraints as
\[
    V_{\textinf}^\CMC(\pi) =\min_{a_w(t, x) \leq w(y, t, x) \leq b_w(t, x)}\EE_\obs[w(Y, T, X)\pi(T|X)Y]
\]
subject to
\begin{equation}
    \EE_\obs[w(Y, T, X) \pi(T|X) Q(T, X)] = \EE_\obs\left[\left(\frac{\pi(T|X)}{p_\obs(T|X)}\right)Q(T, X)\right],
    \label{eq:quantile_balancing_constraints}
\end{equation}
where $Q(t, x)$ denotes the $\tau(t, x)$-quantile of the conditional distribution of $Y$ given $T=t$ and $X=x$ for $\tau(t, x):=\frac{1 / p_\obs(t|x) - a_w(t, x)}{b_w(t, x) - a_w(t, x)}$.

In the case of marginal sensitivity model \eqref{eq:tan_box_constraints} by \citet{tan2006distributional}, the expression for $\tau(t, x)$ can be simplified as $\tau(x) = \frac{1}{1 + \Gamma}$, and therefore, $Q(t, x)$ can be estimated from finite samples by the standard quantile regression.
The quantile balancing (QB) estimator introduced by \citet{dorn2022sharp} relies on this property and estimates the sharp lower bound by solving the empirical version of the above problem with quantile estimate $\hat Q(t, x)$ obtained by the quantile regression.

In addition to showing that their two-stage estimator can be considered a special case of the KCMC estimator, we can also argue that the KCMC estimator can be sharper than the quantile balancing estimator.
For that purpose, we pick an equivalent choice of feature sets for the KCMC estimator and the QB estimator.
When function set $\{\psi_d(t, x)\}_{d=1}^D$ is used by the KCMC estimator, we can choose $\left\{\left(\frac{p_\obs(t|x)}{\pi(t|x)}\right)\psi_d(t,x)\right\}_{d=1}^D$ as the feature of the linear quantile regression.
Then, when the solution of the quantile regression is $\hat Q(t, x) = {\hat\eta_\KCMC}^T\bspsi(t, x)$, the QB estimator becomes as tight as the KCMC estimator.
However, in general, the solution of the quantile regression may not coincide $\hat\eta_\KCMC$ even at an infinite sample limit, since the objective of the quantile regression is not equivalent to the dual objective of the KCMC estimator, and in such cases, the KCMC estimator becomes tighter than the QB estimator.

As our estimator generalizes the previous work, our estimator overcomes some drawbacks of the quantile balancing estimators.
As discussed in Section \ref{chap:intro}, the quantile balancing estimator cannot handle policy learning and the f-constraint.
Policy learning is also difficult with the quantile balancing estimator because taking the derivative with respect to policy requires differentiability of the solution of the above linear programming with respect to parameter $\hat Q$.  
Moreover, the quantile balancing method is designed only for the box constraints and does not have a proper extension to the f-sensitivity model \eqref{eq:f_policy_uncertainty_set}.
In contrast, our estimator of sharper bound $V_\textinf^\CMC$ based on the kernel method can naturally handle the above-mentioned generalized cases of sensitivity analysis.

\end{proof}

\subsection{Proof and full version of Theorem \ref{th:policy_evaluation_consistency}}\label{app:proof_policy_evaluation_consistency}
Let us introduce loss function $\ell:\Theta\times\cZ\to\RR$, where $\Theta\subseteq\RR^K$ for some $K$ and $\cZ:=\cY\times\cT\times\cX$.
To prove the consistency of policy evaluation, we assume the following set of regularity conditions for this loss function:

\begin{condition}[Regularity of loss function I]\label{cond:regularity_1}
\
\begin{enumerate}
    \item  $\theta\mapsto\ell_{\theta}(z)$ is continuous for any $z\in\cZ$.
    \item  $\EE\left|\ell_{\theta}(Z)\right| < \infty$ for any $\theta\in\Theta$.
    \item  $\theta^*\in\arg\min_{\theta\in\Theta}\EE[\ell_{\theta}(Z)]$ is unique.
    \item  $\theta^*$ is well-separated, i.e., $\inf_{\theta: \|\theta - \theta^*\|>\varepsilon}\EE[\ell_\theta(Z)] > \EE[\ell_{\theta^*}(Z)]$ for any $\varepsilon > 0$.
    \item  $\EE[G_\varepsilon(Z)]<\infty$ for $G_\varepsilon(z):=\sup_{\theta\in\Theta:\ \|\theta - \theta^*\| \leq \varepsilon} \left| \ell_{\theta}(z) \right|$ for some $\varepsilon > 0$.
\end{enumerate}
\end{condition}

Additionally, we will need the following convergence lemmas:

\begin{lemma}[Uniform convergence on compact space {\citep[Lemma 7.2.1.]{van2020empirical}}]\label{lemma:uniform_convergence_compact_space}
Assume that parameter space $\left(\Theta, \|\cdot\|\right)$ is compact and satisfies Condition \ref{cond:regularity_1}.
Then, 
$$\sup_{\theta\in\Theta}\left| \hat\EE_n[\ell_\theta(Z)] - \EE[\ell_\theta(Z)] \right| \pto 0.$$
\end{lemma}

\begin{lemma}[Consistency of convex M-estimation {\citep[Lemma 7.2.2.]{van2020empirical}}]\label{lemma:m_estimation}
Suppose $\theta \mapsto \ell_\theta(z)$ is convex for any $z\in\cZ$ and that $\Theta\subseteq\RR^k$ is convex. Then, for M-estimator $\hat\theta_n \in\arg\min_{\theta \in \Theta}\hat\EE_n[\ell_\theta(Z)]$, we have $\hat \theta_n \pto \theta^*$.
\end{lemma}

With these lemmas, we can prove Theorem \ref{th:policy_evaluation_consistency}.
\begingroup
\renewcommand\thetheorem{3}
\begin{theorem}[Consistency of policy evaluation (full version)]
Define the parameter space of $\left(\eta_f, \eta_\KCMC\right)$ as $\Theta \subseteq \RR_+\times\RR^D$.
Further, define 
$\theta^* :=(\eta^*_f, \eta^*_\KCMC)$ as the solution of dual problem \eqref{eq:v_inf_kernel_conditional_moment_constraints_dual} for $V_\textinf^\KCMC(\pi)$ and 
$\hat\theta_n :=(\hat\eta_f, \hat\eta_\KCMC)$ as the solution to dual problem \eqref{eq:empirical_v_inf_kernel_conditional_moment_constraints_dual} for $\hat V_\textinf^\KCMC(\pi)$.
Define $\ell:\Theta\times\cZ\to\RR$ as
\begin{equation}
    \ell_\theta (t, y, x) := \eta_f \gamma - {\eta_\KCMC}^T \bspsi(t, x)
    + \eta_f
        f_{t, x}^*  \left( \frac{{\eta_\KCMC}^T \bspsi(t, x) - r(y, t, x)}{\eta_f} \right)
\end{equation}
so that it is the negative version of dual objectives \eqref{eq:v_inf_kernel_conditional_moment_constraints_dual} and  \eqref{eq:empirical_v_inf_kernel_conditional_moment_constraints_dual} before taking the expectations.
Now, assume Condition \ref{cond:regularity_1} holds for the above loss function.
Then, we have 
$\hat\theta_n\pto\theta^*$
and $\hat V_\textinf^\KCMC(\pi)  \pto V_\textinf^\KCMC(\pi)$.
\end{theorem}
\endgroup

\begin{proof}  

As our dual problem for policy evaluation \eqref{eq:v_inf_kernel_conditional_moment_constraints_dual} and \eqref{eq:empirical_v_inf_kernel_conditional_moment_constraints_dual} are concave maximization, we can immediately apply the above lemma as follows.
We can immediately apply Lemma \ref{lemma:m_estimation} and get $\hat\theta_n\pto\theta^*$.
Thus, $\hat\theta_n$ tend to the inside of compact set $\{\theta\in\Theta:\ \|\theta - \theta^*\| \leq \varepsilon\}$, in which we have the uniform convergence of $\hat\EE_n[\ell_\theta(Z)]$ to $\EE[\ell_\theta(Z)]$ by Lemma \ref{lemma:uniform_convergence_compact_space}.
Therefore, we have $\hat V_\textinf^\KCMC(\pi) = -\hat\EE_n[\ell_{\hat\theta_n}(Z)] \pto -\EE[\ell_{\hat\theta_n}(Z)] \pto  - \EE[\ell_{\theta^*}(Z)] = V_\textinf^\KCMC(\pi)$.
\end{proof}

In practice, it is difficult to check some of the conditions in \ref{cond:regularity_1}, such as the integrability assumption $\EE|\ell_\theta| < \infty$ for any $\theta\in\Theta$ as well as the uniqueness of the solution.
However, it is possible in some cases to verify $L^1$ envelope condition $\EE[G_\varepsilon] < \infty$, because local Lipschitzness of $f^*_{t,x}:\Theta\to\RR$ implies the existence of such $\varepsilon$.
For example, for the box constraints of Example \ref{ex:box_constraint_analytical_solution}, we know that $f^*_{t, x}$ is upper bounded by $b_{\tilde w}(t, x)$.
For f-constraints \eqref{eq:relaxed_f_constraints}, the conjugate function $f^*$ for many choices of f-divergence (such as Kullback-Leibler (KL), squared Hellinger, etc.) is locally Lipschitz.

\subsection{Proof and full version of Theorem \ref{th:concave_policy_learning}}\label{app:proof_concave_policy_learning}

Here, we simply utilize Lemma \ref{lemma:m_estimation} analogously to Theorem \ref{th:policy_evaluation_consistency} to prove the theorem.

\begingroup
\renewcommand\thetheorem{4}
\begin{theorem}[Consistency of concave policy learning (full version)]
Assume concave policy class $\{\pi_\beta(t|x):\ \beta\in\cB\}$ with convex parameter space $\cB$ satisfying that $\beta\mapsto\pi_\beta(t|x)y$ is concave for any $y\in\cY$, $t\in\cT$ and $x\in\cX$.
Assume Condition \ref{cond:regularity_1} holds for loss function \eqref{eq:dual_loss_policy_learning}.
Then, we have
$\hat\theta_n\pto\theta^*$
and 
$\hat V_\textinf^\KCMC(\pi_{\hat\beta}) \pto V_\textinf^\KCMC(\pi_{\beta^*})$.
\end{theorem}
\endgroup
\begin{proof}
Due to the concavity of policy class $\{\pi_\beta(t|x):\ \beta\in\cB\}$, 
we know that $\beta\mapsto\EE\left[\tilde w(Y, T, X)\left(\frac{\pi_\beta(T|X)}{p_\obs(T|X)}\right)Y\right]$ is concave for any $t\in\cT$ and $x\in\cX$.
Then, we can see that 
\begin{align}
    \max_{\beta\in\cB} V_\textinf^\KCMC 
    &= \max_{\beta\in\cB} \max_{\substack{\eta_\KCMC\in\RR^D\\\eta_f>0}} \min_{\tilde w}
    \EE\left[\tilde w(Y, T, X) \left( \frac{\pi_\beta(T|X)}{p_\obs(T|X} \right) Y \right]  \\
    &\quad\quad\quad\quad\quad\quad\quad\quad\quad\quad
    - \EE\left[(\tilde w - 1) {\eta_\KCMC}^T\bspsi\right] 
    + \eta_f\left(\EE[f_{T, X}(\tilde w)] - \gamma\right) \\
    &= \max_{\beta\in\cB}
        \max_{\substack{\eta_\KCMC\in\RR^D\\\eta_f>0}}
        \EE[ - \ell_\theta(Y, T, X)] \\
    &= \max_{\theta\in\Theta} \EE[ - \ell_\theta(Y, T, X)] 
\end{align}
is a concave maximization problem, because $\EE[\ell_\theta]$ is the pointwise infimum of concave functions.
Thus, we can apply Lemma \ref{lemma:m_estimation} and get $\hat\theta_n\pto \theta^*$,
which implies $\hat\theta_n$ tend to the inside of compact set $\{\theta\in\Theta:\ \|\theta - \theta^*\| \leq \varepsilon\}$, where we have uniform convergence guarantee of $\hat\EE_n[\ell_\theta(Z)]$ due to Lemma \ref{lemma:uniform_convergence_compact_space}.
Therefore, we have 
\begin{align*}
0 &\leq \hat V_\textinf^\KCMC(\pi_{\hat \beta}) - \hat V_\textinf^\KCMC(\pi_{\beta^*})  \\
&= -\hat\EE_n[\ell_{\hat\theta_n}(Z)] + \hat\EE_n[\ell_{\theta^*}(Z)]  \\
&= \left[-\hat\EE_n[\ell_{\hat\theta_n}(Z)] + \EE[\ell_{\hat\theta_n}(Z)] \right]
    + \left[\hat\EE_n[\ell_{\theta^*}(Z)] - \EE[\ell_{\theta^*}(Z)]\right]
    - \left[\EE[\ell_{\hat\theta_n}(Z)] - \EE[\ell_{\theta^*}(Z)] \right] \\
&\leq \left[-\hat\EE_n[\ell_{\hat\theta_n}(Z)] + \EE[\ell_{\hat\theta_n}(Z)] \right]
    + \left[\hat\EE_n[\ell_{\theta^*}(Z)] - \EE[\ell_{\theta^*}(Z)]\right] \\
&\pto \ 0,
\end{align*}
which implies $\hat V_\textinf^\KCMC(\pi_{\hat \beta}) \pto \hat V_\textinf^\KCMC(\pi_{\beta ^*})$.
Lastly, due to the law of large numbers, we have 
$\hat V_\textinf^\KCMC(\pi_{\beta ^*}) \pto V_\textinf^\KCMC(\pi_{\beta ^*})$, which concludes the proof.
\end{proof}

\subsection{Proof and full version of Theorem \ref{th:vc_policy_learning}}\label{app:proof_vc_policy_learning}

Let us define 
$\Theta_\varepsilon:= \left\{\left(\eta, \beta \right): \|\eta - \eta^*(\beta)\| \leq \varepsilon, \beta\in\cB\right\}$ 
for $\eta^*(\beta):= \arg\min_{\eta\in H} \EE[ \ell_{\beta, \eta}(Z)]$, which is the $\varepsilon$-neighborhood of the partially optimized parameters.
Here, we need the following regularity conditions on the loss function $\ell_\theta(y, t, x) = \ell_{\eta, \beta}(y, t, x)$.

\begin{condition}[Regularity of loss function II]\label{cond:regularity_2}
\
\begin{enumerate}
    \item  $\theta\mapsto\ell_{\theta}(z)$ is continuous for any $z\in\cZ$.
    \item  $\EE\left|\ell_{\theta}(Z)\right| < \infty$ for any $\theta\in\Theta$.
    \item  $\eta^*(\beta) \in\arg\min_{\eta\in H} \EE[ \ell_{\eta, \beta}(Z)]$ is unique for any $\beta\in\cB$.
    \item \label{cond:uniform_well_separatedness} $\eta^*(\beta)$ is uniformly well-separated, i.e., for any $\varepsilon>0$, 
    \begin{equation}
        \inf_{\theta\in\Theta \setminus \Theta_\varepsilon}
        \left| \EE[ \ell_{\eta, \beta}(Z)] - \EE[ \ell_{\eta^*(\beta), \beta}(Z)] \right| > 0.
    \end{equation}
    \item \label{cond:f_conj_lipschitz_continuous} $f^*_{t, x}$ is uniformly Lipschitz continuous, i.e., 
    there exists $L>0$ such that
    $$\sup_{x\in\cX, t\in\cT, v\in\RR}\sup_{\delta f^*\in \partial f^*_{t, x}}|\delta f^*(v)| \leq L.$$
    \item \label{cond:bounded_eta_f} $\eta^*(\beta)$ is uniformly bounded so that 1) there exist
    $\ubar\eta_f, \bar \eta_f \in \RR$ such that 
    $0 < \ubar\eta_f \leq \eta_f^*(\beta) \leq \bar\eta_f < \infty$ for any $\beta\in\cB$,
    and
    2) $\bar \eta_\KCMC := \sup_{\beta\in\cB}\|\eta^*_\KCMC(\beta)\| < \infty$.
    \item \label{cond:L1_envelope_policy_learning}
    $\EE[G^{\bspsi}(Z)]<\infty$
    and $\EE[G^{f^*}_\varepsilon(Z)]<\infty$
    for 
    \begin{align}
    G^{\bspsi}(z) &:= \left\| \bspsi(t, x) \right\|, \\
    G^{f^*}_\varepsilon(z) &:=
    \sup_{\theta\in\Theta_\varepsilon} \left| 
        f_{t, x}^*  \left( \frac{{\eta_\KCMC}^T \bspsi(t, x) - \left(\frac{\pi_\beta(t|x)}{p_\obs(t|x)}\right)y}{\eta_f} \right) \right|
    \end{align}
    for some $\varepsilon > 0$.
    \item \label{cond:L1_envelope_ipw} $\EE[R(Z)]<\infty$ for $R(z):=
    \sup_{\beta\in\cB} \left| \left(\frac{\pi_\beta(t|x)}{p_\obs(t|x)}\right)y \right|$.
\end{enumerate}
\end{condition}

\begin{remark}
In Condition \ref{cond:regularity_2}, we required uniform well-separatedness assumption and uniform Lipschitzness of $v\mapsto f^*_{t, x}(v)$, which may not seem obvious to the readers.
To put the former condition in slightly more practical terms, we can consider the Hessian of $\eta\mapsto \EE \ell_\theta(Z)$.
If the smallest eigenvalues of the Hessian of $\EE \ell_\theta(Z)$ is uniformly lower bounded so that there exists $\delta>0$ such that $\nabla_\eta \nabla_\eta \EE \ell_\theta(Z) - \delta I$ is positive semidefinite for any $\theta\in\Theta_\varepsilon$, we know that $\EE\ell_{\eta, \beta}(z) - \EE\ell_{\eta^*(\beta), \beta}(z) \geq \frac{1}{2}\delta \|\eta - \eta^*(\beta)\|^2$ for any $(\eta, \beta) \in \Theta_\varepsilon$.
With regard to the latter condition of uniform Lipschitzness, we can think of Example \ref{ex:lipschitz_bounded_f}.
Since the subgradient of $f^*$ can be uniformly bounded by $b_{\tilde w}(t, x)$, we require that $b_{\tilde w}(t, x)$ be uniformly bounded. 
In the case of Tan's box constraint \eqref{eq:tan_box_constraints}, we know that $b_{\tilde w}(t, x) = p_\obs(t|x) - \Gamma ( 1 - p_\obs(t|x))$ and $p_\obs(t|x) \leq 1$, which implies the uniform Lipschitz condition is always satisfied.
\end{remark}

We begin the proof by showing uniform convergence
$\sup_{\theta\in\Theta_\varepsilon} \left| \hat\EE_n\ell_\theta(Z) - \EE\ell_\theta(Z)\right| \pto 0$
for some $\varepsilon>0$.
For the proof of Theorem \ref{th:vc_policy_learning}, we need the following lemmas:

\begin{lemma}[Stability of VC functions {\citep[Lemma 2.6.18]{vaart1996weak}}]\label{lemma:vc_stability}
Let $\cF, \cG$ be VC classes of functions on $\cZ$ and $\RR^D$. Let $h_1: \cZ \to \RR$, $h_2: \cZ\to\RR^D$ be fixed functions. Then, product functions $h_1 \cdot \cF  := \{h_1 \cdot f: f\in\cF\}$ and composite functions $\cG \circ h_2 := \{g \circ h_2: g\in\cG\}$ are VC.
\end{lemma}

\begin{definition}[Covering number {\citep[Definition 6.1.2.]{van2020empirical}}]\label{def:covering_number}
A class of functions $\cG^{(\delta)}$ is said to be a $\delta$-covering of $\cG$ on $\cZ$ with respect to $L_1(\cZ, \EE)$ norm if
\begin{equation}
    \sup_{g\in\cG}\inf_{\tilde g\in\cG^{(\delta)}}\EE \left| g(Z) - \tilde g(Z) \right| \leq \delta.
\end{equation}
Additionally, the covering number of $\cG$ is defined as the minimal cardinality of its $\delta$-covering so that
\begin{equation}
    N_1(\delta, \cG, \EE) = \min\left\{
        \left|\cG^{(\delta)}\right|:\
        \cG^{(\delta)}\text{ is a $\delta$-covering of $\cG$ w.r.t. }L_1(\cZ, \EE)\text{ norm}
    \right\}.
\end{equation}    
\end{definition}

\begin{lemma}[Covering number of a VC class {\citep[Theorem 6.4.1]{van2020empirical}}]\label{lemma:vc_covering_number}
Let $Q$ be any probability measure on $\cZ$ and let $N_1(\cdot, \cG, \EE_Q)$ be the covering number of $\cG$. For a VC class $\cG$ with VC dimension $V$ and bounded envelope $G(z) := \sup_{g\in\cG}|g(z)|$ such that $\EE_Q[G(Z)]<\infty$, we have a constant $A$
depending only on $V$ (not on $Q$) satisfying $N_1(\delta \EE_Q[G], \cG, \EE_Q) \leq \max(A\delta^{-2V}, \exp{\delta / 4})$ for any $\delta > 0$.
\end{lemma}

\begin{lemma}[Uniform convergence {\citep[Theorem 6.1.2]{van2020empirical}}]\label{lemma:uniform_convergence_covering_number}
Suppose $\EE[G(Z)] < \infty$ for $G(z):=\sup_{g\in\cG} g(z)$ and $\frac{\log N_1(\delta, \cG, \hat\EE_n)}{n} \pto 0$ for any $\delta > 0$.
Then, $\sup_{g\in\cG} |\hat\EE_n g - \EE g | \pto 0$.
\end{lemma}

\begin{lemma}[$L_1(\cZ, \EE)$ envelope of $\eta\mapsto\partial_\eta\ell_{\eta, \beta}(z)$]\label{lemma:L1_envelope_subgradients}
Under Condition \ref{cond:regularity_2}, there exist $\varepsilon>0$ such that envelopes
$G^{\partial_{\eta_f} \ell}_\varepsilon(z)
:= \sup_{\substack{\delta\ell \in \partial_{\eta_f} \ell_\theta(z)\\\theta\in\Theta_\varepsilon}}
\left| \delta\ell \right|$
and
$G^{\partial_{\eta_\KCMC} \ell}_\varepsilon(z)
:= \sup_{\substack{\delta\ell \in \partial_{\eta_\KCMC} \ell_\theta(z)\\\theta\in\Theta_\varepsilon}}
\left\| \delta\ell \right\|$
satisfy
$\EE G^{\partial_{\eta_f} \ell}_\varepsilon(Z) < \infty$ and 
$\EE G^{\partial_{\eta_\KCMC} \ell}_\varepsilon(Z) < \infty$.
\end{lemma}
\begin{proof}
Let us take $\varepsilon>0$ sufficiently small so that envelopes $G^{\bspsi}(z)$, $G^{f^*}_\varepsilon(z)$, and $R(z)$ are in $L_1(\cZ, \EE)$.
We can show that
$\left\{z \mapsto \partial_{\eta_f}\ell_\theta(z): \theta\in\Theta_\varepsilon\right\}$
 has $L_1(\cZ, \EE)$ envelope 
$G^{\partial_{\eta_f} \ell}_\varepsilon(z)
:= \sup_{\substack{\delta_{\eta_f}\ell \in \partial_{\eta_f} \ell_\theta(z)\\\theta\in\Theta_\varepsilon}}
\left| \delta_{\eta_f}\ell \right|
\leq \gamma + G^{f^*}_\varepsilon(z) + L \left( \frac{\bar\eta_\KCMC G^{\bspsi}(z) + R(z)}{\ubar\eta_f - \varepsilon}\right)$
because 
\begin{align}
\partial_{\eta_f}\ell_\theta(z)
&= 
\gamma
+ f_{t, x}^*  \left( \frac{{\eta_\KCMC}^T \bspsi(t, x) - \left(\frac{\pi_\beta(t|x)}{p_\obs(t|x)}\right)y}{\eta_f} \right) \\
&\quad\quad+ \left(\frac{{\eta_\KCMC}^T \bspsi(t, x) - \left(\frac{\pi_\beta(t|x)}{p_\obs(t|x)}\right) y}{\eta_f}\right) \partial_{\eta_f} f_{t, x}^*  \left( \frac{{\eta_\KCMC}^T \bspsi(t, x) - \left(\frac{\pi_\beta(t|x)}{p_\obs(t|x)}\right)y}{\eta_f} \right).
\end{align}
Similarly, we can see $\left\{z \mapsto \partial_{\eta_\KCMC}\ell_\theta(z): \theta\in\Theta_\varepsilon\right\}$ 
also has $L_1(\cZ, \EE)$ envelope 
$G^{\partial_{\eta_\KCMC} \ell}_\varepsilon(z)
:= \sup_{\substack{\delta\ell \in \partial_{\eta_\KCMC} \ell_\theta(z)\\\theta\in\Theta_\varepsilon}} \left\| \delta\ell \right\|
\leq (L+1) G^{\bspsi}(z)$
as
\begin{align}
\partial_{\eta_\KCMC}\ell_\theta(z) 
&= \bspsi(t, x) \left\{
    1 - \partial f_{t, x}^*\left(
        \frac{{\eta_\KCMC}^T \bspsi(t, x) - \left(\frac{\pi_\beta(t|x)}{p_\obs(t|x)}\right)y}{\eta_f}
    \right)
\right\},
\end{align}
which concludes the proof.  
\end{proof}

\begin{lemma}[Uniform convergence over $\Theta_\varepsilon$]\label{lemma:uniform_convergence_near_partially_optima}
Assume policy class $\{\pi_\beta(t|x):\ \beta\in\cB\}$ is VC and that Condition \ref{cond:regularity_2} is met.
Then, there exits $\varepsilon > 0$ such that 
$$\sup_{\theta\in\Theta_\varepsilon}{ \left| \hat\EE_n \ell_\theta(Z) - \EE\ell_\theta(Z)\right|}\pto 0.$$
\end{lemma}
\begin{proof}
Remember that loss function $\ell:\Theta\times\cZ\to\RR$ is defined as
\begin{equation}
    \ell_\theta (t, y, x) := \eta_f \gamma - {\eta_\KCMC}^T \bspsi(t, x)
    + \eta_f
        f_{t, x}^*  \left( \frac{{\eta_\KCMC}^T \bspsi(t, x) - \left(\frac{\pi_\beta(t|x)}{p_\obs(t|x)}\right)y}{\eta_f} \right).
\end{equation}
We choose $\varepsilon$ such that envelopes $G^{\bspsi}(z)$, $G^{f^*}_\varepsilon(z)$, and $R(z)$ are in $L_1(\cZ, \EE)$.
Then, $\{z\mapsto \ell_\theta(z): \theta\in\Theta_\varepsilon\}$ has $L_1(\cZ, \EE)$ envelope 
$G^{\ell}_\varepsilon(z) := \sup_{\theta\in\Theta_\varepsilon}\left|\ell_\theta(z)\right| \leq \gamma \eta_f + \bar \eta_\KCMC G^{\bspsi}_\varepsilon(z) + \eta_f G^{f^*}_\varepsilon(z)$.

Now, we would like to show that the class of loss functions $\cG_\varepsilon:=\{\ell_\theta: \theta\in\Theta_\varepsilon\}$ has $\delta$-covering satisfying $\frac{N_1(\delta, \cG, \hat \EE_n)}{n} \pto 0$ for some $\varepsilon > 0$.
First, we know that 
$$\cG'_\varepsilon = \left\{(y, t, x) \mapsto {\eta_\KCMC}^T \bspsi(t, x): \left(\eta, \beta\right) \in \Theta_\varepsilon\right\}$$
and
$$\cG''_\varepsilon = \left\{(y, t, x) \mapsto  \left(\frac{\pi_\beta(t|x)}{p_\obs(t|x)}\right)y: \beta \in \cB \right\}$$
are VC classes by Lemma \ref{lemma:vc_stability}. Thus, by Lemma \ref{lemma:vc_covering_number}, there exist $\frac{\delta}{3(L+1)}$-covering of these satisfying
$\frac{\log N_1(\delta/3(L + 1), \cG'_\varepsilon, \hat\EE_n)}{n} \pto 0$  and
$\frac{\log N_1(\delta/3(L + 1), \cG''_\varepsilon, \hat\EE_n)}{n} \pto 0$.
We let
$H_\KCMC^{(\delta)} = \{\eta_{\KCMC, k}^{(\delta)}: k = 1, \ldots, N_1(\delta/3(L + 1), \cG'_\varepsilon, \hat\EE_n)\}$
and $\cB^{(\delta)} = \{\beta_k^{(\delta)}: k = 1, \ldots, N_1(\delta/3(L + 1), \cG''_\varepsilon, \hat\EE_n)\}$
denote the parameter sets corresponding to such coverings.
Also, we can take $\frac{\delta}{3M}$-mesh of $\{\eta_f: \eta_f\in[\ubar\eta_f - \varepsilon, \bar\eta_f + \varepsilon]\}$
as $H_f^{(\delta)} := \left\{\frac{\delta k}{3M} : k \in\mathbb{Z} \right\} \cap \left[\ubar\eta_f - \varepsilon, \bar\eta_f + \varepsilon \right]$ 
for $M:=2\EE G^{\partial_{\eta_f} \ell}_\varepsilon(Z)$, where $G^{\partial_{\eta_f} \ell}_\varepsilon(Z)$ is as defined in Lemma \ref{lemma:L1_envelope_subgradients}.
Clearly, the cardinality of $H_f^{(\delta)}$ does not depend on the sample size $n$ and $\frac{\log |H^{(\delta)}_f|}{n} \to 0$.

Now, we can show that $\Theta^{(\delta)} := H_f^{(\delta)} \times H_\KCMC^{(\delta)} \times \cB^{(\delta)}$
becomes $\delta$-covering of $\cG_\varepsilon$.
For any $\left(\eta_f, \eta_\KCMC, \beta\right) \in \Theta_\varepsilon$, we can take $\left(\eta_f^{(\delta)}, \eta_\KCMC^{(\delta)}, \beta^{(\delta)}\right) \in H_f^{(\delta)} \times H_\KCMC^{(\delta)} \times \cB^{(\delta)}$ such that
\begin{equation}
     |\eta_f - \eta_f^{(\delta)}| \leq \frac{\delta}{3M},
\end{equation}
\begin{equation}
    \hat\EE_n \left|
    {\eta_\KCMC}^T \bspsi(T, X) - {\eta_\KCMC^{(\delta)}}^T\bspsi(T, X)
    \right| \leq \frac{\delta}{3(L + 1)}, 
\end{equation}
and
\begin{equation}
    \hat\EE_n \left|
        \left(\frac{\pi_\beta(T|X)}{p_\obs(T|X)}\right)Y
        - \left(\frac{\pi_{\beta^{(\delta)}}(T|X)}{p_\obs(T|X)}\right)Y
    \right| \leq \frac{\delta}{3(L + 1)}.
\end{equation}
Thus, we have
\begin{align}
    \hat \EE_n \left| \ell_{\eta_f, \eta_\KCMC, \beta}(Z) - \ell_{\eta_f^{(\delta)}, \eta_\KCMC, \beta}(Z) \right|
    &\leq \hat \EE_n 
        \sup_{\substack{\delta_{\eta_f}\ell \in \partial_{\eta_f} \ell_\theta(Z)\\\theta\in\Theta_\varepsilon}}
            \left| \delta_{\eta_f}\ell \right|
        \left|\eta_f - \eta_f^{(\delta)} \right| \\
    &\leq \hat \EE_n G_\varepsilon^{\partial_{\eta_f}\ell}(Z)
        \left|\eta_f - \eta_f^{(\delta)} \right| \\
    &\pto \frac{M}{2} \cdot \frac{\delta}{3M} \\
    &< \frac{\delta}{3}.
    \label{eq:covering_wrt_eta_f}
\end{align}
and 
\begin{align}
    &\hat \EE_n \left| \ell_{\eta_f^{(\delta)}, \eta_\KCMC, \beta}(Z) - \ell_{\eta_f^{(\delta)}, \eta_\KCMC^{(\delta)}, \beta^{(\delta)}}(Z) \right| \\
    &\leq \hat \EE_n \left|
        J_{\eta_f^{(\delta)}}\left(
                    {\eta_\KCMC}^T \bspsi(T, X),
                    \left(\frac{\pi_\beta(T|X)}{p_\obs(T|X)}\right)Y,
        T, X\right)
        - J_{\eta_f^{(\delta)}}\left(
                    {\eta_\KCMC^{(\delta)}}^T\bspsi(T, X),
                    \left(\frac{\pi_{\beta^{(\delta)}}(T|X)}{p_\obs(T|X)}\right)Y,
        T, X\right)
    \right| \\
    &\quad\leq \hat \EE_n \left[
        \sup_{\substack{\delta J \in \partial_{(h, r)} J_{\eta_f^{(\delta)}}(h, r, t, x),\\h, r\in\RR, t\in\cT, x\in\cX}} {\left\|\delta J\right\|_\infty}
        \cdot \left\|
            \left(
                \begin{array}{c}
                    {\eta_\KCMC}^T \bspsi(T, X) - {\eta_\KCMC^{(\delta)}}^T\bspsi(T, X) \\
                    \left(\frac{\pi_\beta(T|X)}{p_\obs(T|X)}\right)Y - \left(\frac{\pi_{\beta^{(\delta)}}(T|X)}{p_\obs(T|X)}\right)Y
                \end{array}
            \right)
        \right\|_1
    \right] \\
    &\quad\leq \hat \EE_n \left[
        \left\|
            \left(
                \begin{array}{c}
                    1 + L \\
                    L
                \end{array}
            \right)
        \right\|_\infty
        \cdot \left\|
            \left(
                \begin{array}{c}
                    {\eta_\KCMC}^T \bspsi(T, X) - {\eta_\KCMC^{(\delta)}}^T\bspsi(T, X) \\
                    \left(\frac{\pi_\beta(T|X)}{p_\obs(T|X)}\right)Y - \left(\frac{\pi_{\beta^{(\delta)}}(T|X)}{p_\obs(T|X)}\right)Y
                \end{array}
            \right)
        \right\|_1
    \right] \\
    &\quad\leq (L + 1) \cdot \left(
    \hat\EE_n \left|
    {\eta_\KCMC}^T \bspsi(T, X) - {\eta_\KCMC^{(\delta)}}^T\bspsi(T, X)
    \right| 
    + \hat \EE_n \left|
        \left(\frac{\pi_\beta(T|X)}{p_\obs(T|X)}\right)Y
        - \left(\frac{\pi_{\beta^{(\delta)}}(T|X)}{p_\obs(T|X)}\right)Y
    \right| 
    \right) \\
    &\quad\leq (L + 1) \cdot \frac{2\delta}{3(L+1)} = \frac{2\delta}{3},
    \label{eq:covering_wrt_eta_kcmc_and_beta}
\end{align}
where $J_{\eta_f}(h, r, t, x):= - h + \eta_f f^*_{t, x}\left(\frac{h - r}{\eta_f}\right)$.
The subgradients of $J_{\eta_f}$ can be calculated as
$\partial_h J_{\eta_f}(h, r, t, x):= - 1 + \partial f^*_{t, x}\left(\frac{h - r}{\eta_f}\right)$
and
$\partial_r J_{\eta_f}(h, r, t, x):= - \partial f^*_{t, x}\left(\frac{h - r}{\eta_f}\right)$
whose absolute value can be bounded by $|1 + L|$ and $|L|$ respectively.
Finally, by combining \eqref{eq:covering_wrt_eta_f} and \eqref{eq:covering_wrt_eta_kcmc_and_beta}, we get
\begin{equation}
    \hat \EE_n \left| \ell_{\eta_f, \eta_\KCMC, \beta}(Z) - \ell_{\eta_f^{(\delta)}, \eta_\KCMC^{(\delta)}, \beta^{(\delta)}}(Z) \right| \leq \delta
\end{equation}
with high probability as $n\to \infty$.
Thus, $\Theta^{(\delta)}$ is a $\delta$-covering of $\Theta_\varepsilon$ satisfying $\frac{\log |\Theta^{(\delta)}|}{n} \pto 0$.
Thus, we  conclude that 
$$\sup_{\theta\in\Theta_\varepsilon}{ \left| \hat\EE_n \ell_\theta(Z) - \EE\ell_\theta(Z)\right|}\pto 0$$
by Lemma \ref{lemma:uniform_convergence_covering_number}.
\end{proof}

Finally, with the lemmas above, we can prove Theorem \ref{th:vc_policy_learning}.

\begingroup
\renewcommand\thetheorem{5}
\begin{theorem}[Consistency of VC policy learning (full version)]
Assume policy class $\{\pi_\beta(t|x):\ \beta\in\cB\}$ is VC and that Condition \ref{cond:regularity_2} is met.
Then, we have
$\hat V_\textinf^\KCMC(\pi_{\hat\beta}) \pto V_\textinf^\KCMC(\pi_{\beta^*})$.
\end{theorem}
\endgroup

\begin{proof}
Here, our aim is to prove the uniform convergence of the KCMC estimator over the policy class, i.e.
\begin{align}
    \sup_{\beta\in\cB}\left|
        \hat V_\textinf^\KCMC(\pi_\beta) - V_\textinf^\KCMC(\pi_\beta)
    \right| \pto 0.
\end{align}
Let $\hat \eta(\beta) := \arg\min_{\eta\in H} \hat\EE_n\left[ \ell_{\beta, \eta}(Z)\right]$.
Let us also use Lemma \ref{lemma:uniform_convergence_near_partially_optima} to take $\varepsilon > 0$ so that 
$$\sup_{\theta\in\Theta_\varepsilon}{ \left| \hat\EE_n \ell_\theta(Z) - \EE\ell_\theta(Z)\right|}\pto 0.$$
For this choice of $\varepsilon$ and any $\beta\in\cB$, let us define
\begin{equation}
\tilde \eta(\beta) :=
\frac{\varepsilon}{\varepsilon + \| \hat\eta(\beta) - \eta^*(\beta) \|} \cdot \hat \eta(\beta)
+ \frac{\| \hat \eta(\beta) - \eta^*(\beta) \|}{\varepsilon + \| \hat\eta(\beta) - \eta^*(\beta) \|} \cdot \eta^*(\beta)
\label{eq:eta_tilde}
\end{equation}
so that it becomes a linear interpolation between $\hat \eta(\beta)$ and $\eta^*(\beta)$ that is always included in the $\varepsilon$ neighborhood of $\eta^*(\beta)$ so that
$\|\tilde \eta(\beta) - \eta^*(\beta) \| \leq \varepsilon$.
Since $\eta\mapsto \ell_{\eta, \beta}(z)$ is convex for any $z\in\cZ$ and any $\beta\in\cB$, we have
\begin{equation}
    \hat \EE_n [\ell_{\tilde\eta(\beta), \beta}(Z)] \leq 
    \frac{\varepsilon}{\varepsilon + \| \hat\eta(\beta) - \eta^*(\beta) \|} \cdot \hat \EE_n [\ell_{\hat\eta(\beta), \beta}(Z)] 
    + \frac{\| \hat \eta(\beta) - \eta^*(\beta) \|}{\varepsilon + \| \hat\eta(\beta) - \eta^*(\beta) \|} \cdot \hat \EE_n [\ell_{\eta^*(\beta), \beta}(Z)].
\end{equation}
As the optimality of $\hat\eta(\beta)$ with respect to $\beta\mapsto\hat\EE_n \ell_{\eta, \beta}(Z)$ implies
$0 \leq \hat \EE_n [\ell_{\tilde\eta(\beta), \beta}(Z)] - \hat \EE_n [\ell_{\hat\eta(\beta), \beta}(Z)]$, we get
\begin{equation}
    0 \leq  \hat \EE_n [\ell_{\eta^*(\beta), \beta}(Z)] - \hat \EE_n [\ell_{\tilde\eta(\beta), \beta}(Z)].
\end{equation}
Since $(\eta^*(\beta), \beta)$ and $(\tilde\eta(\beta), \beta)$ are included in $\Theta_\varepsilon$, we can use Lemma \ref{lemma:uniform_convergence_near_partially_optima} to see
\begin{align}
    0 &\leq \sup_{\beta\in\cB}\left[
        \EE[\ell_{\tilde\eta(\beta), \beta}(Z)] - \EE[\ell_{\eta^*(\beta), \beta}(Z)]
    \right]\\
    &\leq \sup_{\beta\in\cB}\Bigl[
        \left(\EE[\ell_{\tilde\eta(\beta), \beta}(Z)] - \hat\EE_n[\ell_{\tilde\eta(\beta), \beta}(Z)]\right)
        - \left(\EE[\ell_{\eta^*(\beta), \beta}(Z)] - \hat \EE_n[\ell_{\eta^*(\beta), \beta}(Z)]\right) \\
        &\quad\quad\quad - \left(\hat \EE_n[\ell_{\eta^*(\beta), \beta}(Z)] - \hat\EE_n[\ell_{\tilde\eta(\beta), \beta}(Z)]\right)
    \Bigr]\\
    &\leq \sup_{\beta\in\cB}\Bigl[
        \left(\EE[\ell_{\tilde\eta(\beta), \beta}(Z)] - \hat\EE_n[\ell_{\tilde\eta(\beta), \beta}(Z)]\right)
        - \left(\EE[\ell_{\eta^*(\beta), \beta}(Z)] - \hat \EE_n[\ell_{\eta^*(\beta), \beta}(Z)]\right)
    \Bigr]\\
    &\leq 2 \sup_{(\eta, \beta)\in\Theta_\varepsilon} \left|
        \EE[\ell_{\eta, \beta}(Z)] - \hat\EE_n[\ell_{\eta, \beta}(Z)] 
    \right| \\
    &\pto 0,
\end{align}
which implies
\begin{equation}
    \sup_{\beta\in\cB}\left|
        \EE[\ell_{\tilde\eta(\beta), \beta}(Z)] - \EE[\ell_{\eta^*(\beta), \beta}(Z)]
    \right|
    \pto 0.
\end{equation}
Due to the uniform well-separatedness assumption in Condition \ref{cond:regularity_2}, this leads to
\begin{equation}
    \sup_{\beta\in\cB} \left\| \tilde\eta(\beta) - \eta^*(\beta) \right\| \pto 0,
\end{equation}
and by definition of $\tilde\eta$ \eqref{eq:eta_tilde}, we have uniform consistency of $\hat\eta(\beta)$,
\begin{equation}
    \sup_{\beta\in\cB} \left\| \hat\eta(\beta) - \eta^*(\beta) \right\| \pto 0.
\end{equation}
As $\eta\mapsto \EE\ell_{\eta, \beta}(z)$ has a Lipschitz constant
$\EE \sqrt{
\left| G_\varepsilon^{\partial_{\eta_f}\ell}(Z) \right|^2 + 
\left| G_\varepsilon^{\partial_{\eta_\KCMC}\ell}(Z) \right|^2}$
that does not depend on $\beta$ due to Lemma \ref{lemma:L1_envelope_subgradients}, we obtain
\begin{align}
    \sup_{\beta\in\cB}\left|
        \hat\EE_n\left[ \ell_{\beta, \hat\eta(\beta)}(Z)\right]
        - \EE\left[\ell_{\beta, \eta^*(\beta)}(Z) \right]
    \right| \pto 0
\end{align}
i.e.
\begin{align}
    \sup_{\beta\in\cB}\left|
        \hat V_\textinf^\KCMC(\pi_\beta) - V_\textinf^\KCMC(\pi_\beta)
    \right| \pto 0.
\end{align}

Then, we have
\begin{align}
    0 &\geq V_\textinf^\KCMC(\pi_{\hat\beta}) - V_\textinf^\KCMC(\pi_{\beta^*}) \\
    &\geq \left( V_\textinf^\KCMC(\pi_{\hat\beta}) - \hat V_\textinf^\KCMC(\pi_{\hat\beta}) \right)
    - \left( V_\textinf^\KCMC(\pi_{\beta^*}) - \hat V_\textinf^\KCMC(\pi_{\beta^*}) \right) \\
    &\quad\quad + \left( \hat V_\textinf^\KCMC(\pi_{\hat\beta}) - \hat V_\textinf^\KCMC(\pi_{\beta^*}) \right)\\
    &\geq \left( V_\textinf^\KCMC(\pi_{\hat\beta}) - \hat V_\textinf^\KCMC(\pi_{\hat\beta}) \right)
    - \left( V_\textinf^\KCMC(\pi_{\beta^*}) - \hat V_\textinf^\KCMC(\pi_{\beta^*}) \right) \\
    &\geq - 2\sup_{\beta\in\cB}\left|
        \hat V_\textinf^\KCMC(\pi_\beta) - V_\textinf^\KCMC(\pi_\beta)
    \right| \pto 0.
\end{align}
Thus, 
\begin{align}
    V_\textinf(\pi_{\hat \beta}) - V_\textinf(\pi_{\beta^*})
    &\pto 0.
\end{align}
and
\begin{align}
    \left| \hat V_\textinf(\pi_{\hat \beta}) - V_\textinf(\pi_{\beta^*}) \right|
    \leq \left| \hat V_\textinf(\pi_{\hat \beta}) - V_\textinf(\pi_{\hat\beta}) \right|
        + \left| V_\textinf(\pi_{\hat \beta}) - V_\textinf(\pi_{\beta^*}) \right|
   \pto 0,
\end{align}
which concludes the proof.  
\end{proof}

\subsection{Proof of Theorem \ref{thm:asymptotic_normality}}\label{app:proof_asymptotic_normality}

To derive the asymptotic distribution of $\hat\theta = \left(\hat\eta_f, \hat\eta_\KCMC\right)$, we need the following regularity conditions in addition to Condition \ref{cond:regularity_1}.

\begin{condition}[Regularity of loss function III]\label{cond:regularity_3}
\

\begin{enumerate}
    \item $\EE|\ell_{\theta^*}(Z)|^2 < \infty$.
    \item $\theta\mapsto \EE\ell_\theta(Z)$ is twice differentiable with positive definite Hessian at $\theta=\theta^*$ so that 
    $$V:=\nabla_\theta \nabla_\theta^T \EE\ell_{\theta^*}(Z) \succ 0.$$
    \item $\theta\mapsto\ell_\theta(Z)$ is differentiable at $\theta=\theta^*$ almost everywhere, i.e., 
    $$\EE\bsone_{\left\{z\in\cZ: \theta\mapsto\ell_\theta(z)\text{ is not differentiable at }\theta=\theta^*\right\}}(Z) = 0.$$
    \item $f^*_{t, x}$ is uniformly Lipschitz continuous, i.e., 
    there exists $L>0$ such that
    $$\sup_{x\in\cX, t\in\cT, v\in\RR}\sup_{\delta f^*\in \partial f^*_{t, x}}|\delta f^*(v)| \leq L.$$
    \item 
    $\EE\left|G^{\bspsi}(Z)\right|^2 < \infty$
    and $\EE\left|G^{f^*}_\varepsilon(Z)\right|^2<\infty$
    for 
    \begin{align}
    G^{\bspsi}(z) &:= \left\| \bspsi(t, x) \right\|, \\
    G^{f^*}_\varepsilon(z) &:=
    \sup_{\theta: \|\theta - \theta^*\| \leq \varepsilon} \left| 
        f_{t, x}^*  \left( \frac{{\eta_\KCMC}^T \bspsi(t, x) - \left(\frac{\pi_\beta(t|x)}{p_\obs(t|x)}\right)y}{\eta_f} \right) \right|
    \end{align}
    for some $\varepsilon > 0$.
    \item $\EE|R(Z)|^2 < \infty$ for $R(z):=
    \sup_{\beta\in\cB} \left| \left(\frac{\pi_\beta(t|x)}{p_\obs(t|x)}\right)y \right|$.
\end{enumerate}
\end{condition}

With regularity conditions \ref{cond:regularity_3}, we can apply the following lemma to immediately obtain the asymptotic normality of our estimator.

\begin{lemma}{\citep[Example 3.2.12, 3.2.22]{vaart1996weak}}
\label{lemma:asymptotic_normality_lipschitz_parameter}
Assume $\ell_\theta$ is Lipschitz in the neighborhood of $\theta=\theta^*$ and is differentiable in quadratic mean at $\theta=\theta^*$ with respect to parameter $\theta$, i.e. 
there exist $\varepsilon > 0$, $M:\cZ\to\RR$, and $\delta\ell_{\theta^*}: \cZ \to \RR^{D + 1}$ such that 
\begin{align}
    \left| \ell_{\theta_1}(z) - \ell_{\theta_2}(z) \right| &\leq M(z) \|\theta_1 - \theta_2\|
\end{align}
for any $\theta_1, \theta_2 \in \{\theta: \|\theta - \theta^*\| \leq \varepsilon\}$,
$\EE M(Z)^2 < \infty$,
and
\begin{equation}
    \lim_{\theta\to\theta^*} \EE \left|\frac{\ell_\theta(z) - \ell_{\theta^*}(z) - (\theta - \theta^*)^T \delta \ell_{\theta^*}(z))}{\|\theta - \theta^*\|}\right|^2 = 0.
    \label{eq:differentiability_in_quadratic_mean}
\end{equation}
Further assume that $\theta\mapsto \EE\ell_\theta(Z)$ is twice differentiable with positive definite Hessian at $\theta=\theta^*$ so that $V:=\nabla_\theta \nabla_\theta^T \EE\ell_{\theta^*}(Z) \succ 0$.
Then, for $\hat\theta \pto \theta^*$, we have
\begin{equation}
    \sqrt{n}\left[\hat\EE_n \ell_{\hat\theta}(Z) - \hat\EE_n \ell_{\theta^*}(Z) \right] \pto 0,
\end{equation}
and
\begin{equation}
    \sqrt{n}\left[
        V (\hat\theta  - \theta^*) + \hat\EE_n \delta \ell_{\theta^*}(Z)
    \right]
    \pto 0.
\end{equation}
\end{lemma}

With this lemma, we prove Theorem \ref{thm:asymptotic_normality} as follows:

\begingroup
\renewcommand\thetheorem{6}
\begin{theorem}[Asymptotic normality (full version)]
If loss function $\ell_\theta(z)$ satisfies Condition \ref{cond:regularity_1} and Condition \ref{cond:regularity_3}, then we have
\begin{equation}
    \sqrt{n}\left[\hat\EE_n \ell_{\hat\theta}(Z) - \hat\EE_n \ell_{\theta^*}(Z) \right] \pto 0,
\end{equation}
and
\begin{equation}
    \sqrt{n}\left[
        V (\hat\theta  - \theta^*) + \hat\EE_n \delta \ell_{\theta^*}(Z)
    \right]
    \pto 0.
\end{equation}
therefore,
\begin{equation}
    \sqrt{n}[ \hat\theta  - \theta^*]
    \pto  V^{-1} \sqrt{n} \hat\EE_n [\delta\ell_{\theta^*}(Z)]
    \dto N(0, V^{-1}J V^{-1}),
\end{equation}
where $V$ is defined as in Condition \ref{cond:regularity_3}, $J:=\EE\left[ \delta \ell_{\theta^*}(Z) \delta \ell_{\theta^*}(Z)^T \right]$, and $\dto$ implies the convergence in distribution.
\end{theorem}
\endgroup

\begin{proof}
Here, we just have to show that the assumptions in Lemma \ref{lemma:asymptotic_normality_lipschitz_parameter} are satisfied.
Due to Theorem \ref{th:policy_evaluation_consistency}, we know that $\hat\theta \pto \theta^*$.
We can take $\delta \ell_{\theta^*}\in\partial \ell_{\theta^*}$ such that $\EE[\delta \ell_{\theta^*}] = 0$ and it satisfies the differentiability in quadratic mean \eqref{eq:differentiability_in_quadratic_mean}.
Now, we can take $M(z) = G_\varepsilon^{\partial \ell} := \sup_{\theta: \|\theta - \theta^*\| \leq \varepsilon} \sup_{\delta \ell \in \partial \ell_\theta} \|\delta \ell(z)\|$.
Thus, it remains to show that
\begin{equation}
    \EE \left| G_\varepsilon^{\partial \ell}(Z) \right|^2 < \infty.
\end{equation}
Indeed, using the same argument as Lemma \ref{lemma:L1_envelope_subgradients}, we have 
\begin{equation}
    \left| G_\varepsilon^{\partial_\theta \ell}(z) \right|^2 \leq
    \left| G_\varepsilon^{\partial_{\eta_f} \ell}(z) \right|^2 + \left| G_\varepsilon^{\partial_{\eta_\KCMC} \ell}(z) \right|^2 
\end{equation}
for 
\begin{equation}
G^{\partial_{\eta_f} \ell}_\varepsilon(z)
    := \sup_{\substack{\delta_{\eta_f}\ell \in \partial_{\eta_f} \ell_\theta(z)\\ \theta: \|\theta - \theta^*\| \leq \varepsilon}}
    \left| \delta_{\eta_f}\ell \right|
    \leq \gamma + G^{f^*}_\varepsilon(z) + L \left( \frac{(\theta^*_\KCMC + \varepsilon) G^{\bspsi}(z) + R(z)}{\eta_f^* - \varepsilon}\right) 
\end{equation}
and
\begin{equation}
G^{\partial_{\eta_\KCMC} \ell}_\varepsilon(z)
    := \sup_{\substack{\delta\ell \in \partial_{\eta_\KCMC} \ell_\theta(z)\\ \theta: \|\theta - \theta^*\| \leq \varepsilon}} \left\| \delta\ell \right\|
    \leq (L+1) G^{\bspsi}(z).
\end{equation}
Since we can take $\varepsilon$ satisfying $\eta_f^* > \varepsilon > 0$ such that 
\begin{equation}
\EE\left| G^{\partial_{\eta_f} \ell}_\varepsilon(Z) \right|^2
\leq 4 \left\{
    \gamma ^ 2
    + \EE\left| G^{f^*}_\varepsilon(Z)\right|^2
    + \frac{L^2 {(\theta^*_\KCMC + \varepsilon) }^2 \EE |G^{\bspsi}(Z)|^2}{|\eta_f^* - \varepsilon|^2}
    + \frac{L^2 \EE|R(Z)|^2}{|\eta_f^* - \varepsilon|^2}
\right\}
< \infty
\end{equation}
and 
\begin{equation}
\EE \left| G^{\partial_{\eta_\KCMC} \ell}_\varepsilon(Z) \right|^2
\leq (L+1)^2 \EE \left|G^{\bspsi}(Z) \right|^2
\end{equation}
due to Condition \ref{cond:regularity_3}, we have $\EE \left| G_\varepsilon^{\partial_\theta \ell}(Z) \right|^2 < \infty.$ Therefore, we can apply Lemma \ref{lemma:asymptotic_normality_lipschitz_parameter} to our problem to conclude the proof.
\end{proof}

\subsection{Details of Example \ref{ex:gic}}\label{app:proof_gic}
\begin{proof}
Due to the second order expansion in Condition \ref{cond:regularity_3}, we have
\begin{align}
    &n\hat\EE_n[\ell_{\hat\theta}(Z) - \ell_{\theta^*}(Z)] \\
    &= \sqrt{n}\hat\EE_n[\delta\ell_{\theta^*}(Z)]^T \sqrt{n}(\hat\theta - \theta^*)
        + \int_0^1
                \sqrt{n}\hat\EE_n[\delta\ell_{t(\hat\theta - \theta^*) + \theta^*}(Z) - \delta\ell_{\theta^*}(Z)]^T 
                \sqrt{n}(\hat \theta - \theta^*) 
        \rd t \\
    &= \sqrt{n}\hat\EE_n[\delta\ell_{\theta^*}(Z)]^T \sqrt{n}(\hat\theta - \theta^*) \\
        &\quad+ \int_0^1
                    \sqrt{n}\left(\hat\EE_n - \EE\right)
                    \left[\delta\ell_{t(\hat\theta - \theta^*) + \theta^*}(Z) - \delta\ell_{\theta^*}(Z)\right]^T 
                \sqrt{n}(\hat \theta - \theta^*) 
        \rd t \\
        &\quad+ \int_0^1
                    \sqrt{n}\EE\left[\delta\ell_{t(\hat\theta - \theta^*) + \theta^*}(Z) - \delta\ell_{\theta^*}(Z)\right]^T 
                \sqrt{n}(\hat \theta - \theta^*) 
        \rd t \\
    &= \sqrt{n}\hat\EE_n[\delta\ell_{\theta^*}(Z)]^T \sqrt{n}(\hat\theta - \theta^*) \\
        &\quad+ \int_0^1
                o_\text{p}(1)
                \sqrt{n}(\hat \theta - \theta^*) 
        \rd t \\
        &\quad+ \int_0^1
                    \sqrt{n}\EE\left[\delta\ell_{t(\hat\theta - \theta^*) + \theta^*}(Z) - \delta\ell_{\theta^*}(Z)\right]^T 
                \sqrt{n}(\hat \theta - \theta^*) 
        \rd t \\
    &= \sqrt{n}\hat\EE_n[\delta\ell_{\theta^*}(Z)]^T \sqrt{n}(\hat\theta - \theta^*) 
        + o_\text{p}(1) \cdot O_\text{p}(\sqrt{n}\|\hat\theta - \theta^*\|) \\
        &\quad+ \int_0^1
                    \sqrt{n}\left\{
                        \nabla_\theta\EE[\ell_{t(\hat\theta - \theta^*) + \theta^*}(Z)]^T - \nabla_\theta\EE[\ell_{\theta^*}(Z)]^T 
                    \right\}
                \sqrt{n}(\hat \theta - \theta^*) 
        \rd t \\
    &= \sqrt{n}\hat\EE_n[\delta\ell_{\theta^*}(Z)]^T \sqrt{n}(\hat\theta - \theta^*) 
        + o_\text{p}(1)
        + \int_0^1 \sqrt{n} \left((\hat\theta - \theta^*)^T t V + o(\|\hat\theta - \theta^*\|)\right) \sqrt{n}(\hat \theta - \theta^*)  \rd t \\
    &\pto - \sqrt{n} \left(\hat \theta - \theta^*\right)^T V \sqrt{n}(\hat\theta - \theta^*) 
        + \frac{1}{2}\sqrt{n}\left(\hat \theta - \theta^*\right)^T V \sqrt{n} (\hat\theta - \theta^*) \\
    &= - \frac{1}{2}\sqrt{n}\left(\hat \theta - \theta^*\right)^T V \sqrt{n} (\hat\theta - \theta^*),
\end{align}
where we made an approximation $\sqrt{n}\left(\hat\EE_n - \EE\right) \left[\delta\ell_{t(\hat\theta - \theta^*) + \theta^*}(Z) - \delta\ell_{\theta^*}(Z)\right] \pto 0$.
\footnote{  
    Though this approximation is not justified by our previous asymptotic theorem, it is valid when $\{\delta \ell_\theta: \|\theta- \theta^*\| < \varepsilon\}$ is a Donsker class, for which the uniform central limit theorem holds.
    Though proving the Donsker property for general $f_{x, t}$ is difficult, we can still show that Example \ref{ex:lipschitz_bounded_f} with $f^0_{t, x} = f^0$ not depending on $t$ and $x$ satisfies this condition.
    To sketch the proof, we first show that $\partial \ell_\theta(z)$ consists of terms that are VC classes by using reparametrization $\nu = (1/\eta_f, \eta_\KCMC / \eta_f)$ and stability of VC classes \citep[2.6.18]{vaart1996weak} along with the fact that
    $f^*_{t, x}(v) = \max\{ {f^0}^*(v), \min\{a_{t, x}v - f^0(a_{t, x}), b_{t, x}v - f^0(b_{t, x})\} \}$,
    $\partial f^*_{t, x}(v) = \min\{ b_{t, x}, \max\{a_{t, x}, \partial {f^0}^*(v)\} \}$ and that $\partial f^*$ is monotone.
    Then, we use the fact that VC classes have bounded uniform entropy and are Donsker class \citep[2.5.2, 2.6.7]{vaart1996weak}, and their permanence \citep[2.10.6, 2.10.20, 2.10.23]{vaart1996weak} to show $\{\delta \ell_\theta: \|\theta- \theta^*\| < \varepsilon\}$ is a Donsker class.
}
\end{proof}

\section{Derivation of Hessian for the box constraints}\label{app:hessian_expected_loss}

Here, we derive the Hessian of $f^*$ in the case of box constraints. For conjugate function
\begin{equation}
    f_{t, x}^*(v) = \begin{cases}
        a(t, x)v & \text{ if \ \ }v < 0, \\
        0 & \text{ if \ \ }v = 0, \\
        b(t, x)v & \text{ if \ \ }v > 0,
    \end{cases}
\end{equation}
and its subgradient
\begin{equation}
    \partial f_{t, x}^*(v) = \begin{cases}
        a(t, x) & \text{ if \ \ }v < 0, \\
        [a(t, x), b(t, x)] & \text{ if \ \ }v = 0, \\
        b(t, x) & \text{ if \ \ }v > 0,
    \end{cases}
\end{equation}
the gradient of the dual objective
\footnote{For the box constraints, it can be shown that the dual objective becomes
$-\gamma \eta_f + {\eta_\KCMC}^T\EE[\bspsi] - \EE\left[f_{T, X}^*\left({\eta_\KCMC}^T\bspsi - r\right)\right]$
and that its supremum is reached when $\eta_f \to 0$.Furthermore, when we extend the domain of the dual objective from $\eta_f>0$ to $\eta_f\geq 0$, we know that $\lim_{n\to\infty}\PP(\hat\eta_f = 0) \to 1$, so $\hat \eta_f$ does not need any asymptotic analysis.}
is
\begin{align*}
&\nabla_{\eta_\KCMC} \left\{
    {\eta_\KCMC}^T\EE[\bspsi]
    - \EE\left[f_{T, X}^*\left({\eta_\KCMC}^T\bspsi - r\right)\right]
\right\} \\
&\quad= \EE\left[\left(
    1 - \partial f_{T, X}^*\left({\eta_\KCMC}^T\bspsi - r\right)
\right) \bspsi \right].
\end{align*}

Therefore, the second-order derivatives are:
\begin{align}
&\nabla_{\eta_\KCMC} \nabla_{\eta_\KCMC}^T
    \left\{
        {\eta_\KCMC}^T\EE[\bspsi]
        - \eta_f \EE\left[f_{T, X}^*\left({\eta_\KCMC}^T\bspsi - r\right)\right]
    \right\} \\
&\quad\quad = \nabla_{\eta_\KCMC}  \EE\left[\left(
    1 - \partial f_{T, X}^*\left({\eta_\KCMC}^T\bspsi - r\right)
\right) \bspsi^T \right] \\
&\quad\quad = - \EE\left[ \left\{ \nabla_{\eta_\KCMC} 
    \EE\left[ \partial f_{T, X}^*\left({\eta_\KCMC}^T\bspsi - r\right) | T, X \right] \right\}
    \bspsi^T \right] \\
&\quad\quad = - \EE\left[ \left\{ \nabla_{\eta_\KCMC}  \EE\left[
    a(T, X)\bsone_{\RR_-}\left({\eta_\KCMC}^T\bspsi - r\right)
    + b(T, X)\bsone_{\RR_+}\left({\eta_\KCMC}^T\bspsi - r\right) | T, X \right] \right\}
    \bspsi^T \right] \\
&\quad\quad = - \EE\left[ \left\{
    \nabla_{\eta_\KCMC} a(T, X)
    \PP\left[\left({\eta_\KCMC}^T\bspsi - r\right) < 0 | T, X \right] \right. \right. \\
&\quad\quad\quad
    \left. \left. + \nabla_{\eta_\KCMC} b(T, X)
    \PP\left[\left({\eta_\KCMC}^T\bspsi - r\right) > 0 | T, X \right]
    \right\}
    \bspsi^T \right] \\
&\quad\quad = - \EE\left[ \left\{
    \nabla_{\eta_\KCMC} a(T, X)
    \PP\left[\eta_\KCMC^T\bspsi < r | T, X \right]
    + \nabla_{\eta_\KCMC} b(T, X)
    \PP\left[\eta_\KCMC^T\bspsi > r | T, X \right] 
    \right\} \bspsi^T \right] \\
&\quad\quad = - \EE\left[ 
    \bspsi \bspsi^T \{b(T, X) - a(T, X)\} p_r(\eta_\KCMC^T \bspsi|T, X)
    \right],
\end{align}
where $p_r(\cdot|t, x)$ is the probability density function of $r(Y, T, X)$ conditioned on $T=t, X=x$.

\bibliography{references}
\end{document}